\documentclass{article}

\PassOptionsToPackage{numbers, compress}{natbib}

     \usepackage[final]{neurips_2022}
     
\usepackage[utf8]{inputenc} 
\usepackage[T1]{fontenc}    
\usepackage{color, colortbl}
\usepackage{xcolor}
\usepackage{hyperref}       
\hypersetup{
	colorlinks   = true, 
	urlcolor     = blue, 
	linkcolor    = blue, 
	citecolor   = blue 
}
\usepackage{url}            
\usepackage{booktabs}       
\usepackage{amsfonts}       
\usepackage{nicefrac}       
\usepackage{microtype}      

\usepackage{microtype}
\usepackage{graphicx}
\usepackage{algorithm}
\usepackage{algpseudocode}
\usepackage{booktabs} 
\usepackage{etex}
\usepackage{amsmath}
\usepackage{amsthm}
\usepackage{amssymb} 
\usepackage{hyperref}
\usepackage[capitalise]{cleveref}
\usepackage{footnote}

\usepackage{enumitem}
\usepackage{bm}
\usepackage{bbm}
\usepackage{graphicx}
\usepackage{thmtools}
\usepackage{thm-restate}
\usepackage{mathtools}
\usepackage[skip=0pt]{caption}
\usepackage[skip=0pt]{subcaption}
\usepackage{booktabs} 
\usepackage{fontawesome}
\usepackage{enumitem}
\usepackage{accents}

\usepackage{threeparttable, array, float} 

\newtheorem{theorem}{\bf Theorem}
\newtheorem{lemma}{\bf Lemma}

\newtheorem{condition}{\bf Condition}

\newtheorem{definition}{\bf Definition}

\definecolor{red}{HTML}{E51400} 
\definecolor{blue}{HTML}{0050EF} 
\definecolor{green}{HTML}{008A00} 
\definecolor{purple}{HTML}{AA00FF} 
\definecolor{orange}{HTML}{FF7F00}
\definecolor{gray}{HTML}{848482}
\definecolor{Gray}{gray}{0.85}
\definecolor{LightGray}{gray}{0.96}

\DeclareMathOperator*{\argmin}{argmin}
\DeclareMathOperator*{\argmax}{argmax}

\newcommand{\norm}[1]{\left\lVert#1\right\rVert}

\newcommand{\cS}{\mathcal{S}}
\newcommand{\abs}[1]{\left| #1 \right|}

\DeclarePairedDelimiter{\ceil}{\lceil}{\rceil}

\newcommand{\TPVMm}{{TPVM$_<$}}

\newcommand{\R}{\mathbb{R}}

\newcommand{\E}{\mathbb{E}}
\newcommand{\G}{\mathbb{G}}

\newcommand{\I}{\mathbb{I}}
\newcommand{\bT}{\mathbb{T}}

\renewcommand{\bT}{\boldsymbol{T}}

\newcommand{\bX}{\boldsymbol{X}}

\newcommand{\bx}{\boldsymbol{x}}
\newcommand{\by}{\boldsymbol{y}}

\newcommand{\bmu}{\boldsymbol{\mu}}

\newcommand{\boldeta}{\boldsymbol{\eta}}
\newcommand{\boldzeta}{\boldsymbol{\zeta}}

\newcommand{\bzero}{\boldsymbol{0}}

\newcommand{\cA}{\mathcal{A}}

\newcommand{\cD}{\mathcal{D}}
\newcommand{\cE}{\mathcal{E}}

\newcommand{\cN}{\mathcal{N}}

\newcommand{\ts}[1]{}

\newcommand{\compilefullversion}{full}
\ifthenelse{\equal{\compilefullversion}{false}}{%
	\newcommand{\OnlyInFull}[1]{}
	\newcommand{\OnlyInShort}[1]{#1}
}{%
	\newcommand{\OnlyInFull}[1]{#1}%
	\newcommand{\OnlyInShort}[1]{}%
}%

\newcommand{\compilehidecomments}{false}
\ifthenelse{ \equal{\compilehidecomments}{true} }{%
	\newcommand{\wei}[1]{}
	\newcommand{\xutong}[1]{}
	\newcommand{\jinhang}[1]{}
	\newcommand{\siwei}[1]{}
\newcommand{\carlee}[1]{}
}{
\newcommand{\wei}[1]{{\color{blue}{[Wei: #1]}}}
\newcommand{\xutong}[1]{{\color{green} [Xutong: #1]}}
\newcommand{\jinhang}[1]{{\color{orange} [\text{Jinhang:} #1]}}
\newcommand{\siwei}[1]{{\color{red} [\text{Siwei:} #1]}}
\newcommand{\carlee}[1]{{\color{cyan} [\text{Carlee:} #1]}}
}

\usepackage{titlesec}
\titlespacing{\section}{0pt}{5pt plus -1pt minus 1pt}{0pt plus 1pt minus 1pt}
\titlespacing{\subsection}{0pt}{3pt plus -1pt minus 1pt}{-2pt plus 1pt minus 2pt}
\titlespacing{\subsubsection}{0pt}{\parskip}{-\parskip}
\makeatletter
\def\thm@space@setup{
\thm@preskip=0.8\topsep
\thm@postskip=\thm@preskip 
}
\makeatother

\allowdisplaybreaks

\newenvironment{talign*}
 {\csname align*\endcsname}
 {\endalign}



\title{Batch-Size Independent Regret Bounds for Combinatorial
Semi-Bandits with Probabilistically Triggered Arms or Independent Arms}  

\author{%
  Xutong Liu \\
  The Chinese University of Hong Kong\\
  Hong Kong SAR, China \\
  \texttt{liuxt@cse.cuhk.edu.hk} \\
   \And
   Jinhang Zuo \\
   Carnegie Mellon University \\
  Pittsburgh, PA, USA \\
  \texttt{jzuo@andrew.cmu.edu} \\
   \And
   \hspace{15pt}Siwei Wang \\
    \hspace{15pt}Microsoft Research \\
    \hspace{15pt}Beijing, China \\
   \hspace{15pt}\texttt{siweiwang@microsoft.com} \\
   \And
   Carlee Joe-Wong \\
   Carnegie Mellon University \\
   Pittsburgh, PA, USA \\
  \texttt{cjoewong@andrew.cmu.edu} \\
   \And
   John C.S. Lui \\
   The Chinese University of Hong Kong \\
   Hong Kong SAR, China \\
  \texttt{cslui@cse.cuhk.edu.hk} \\
    \And
   Wei Chen \\
   Microsoft Research \\
    Beijing, China \\
  \texttt{weic@microsoft.com} \\
}

\begin{document}
\maketitle
\begin{abstract}


In this paper, we study the combinatorial semi-bandits (CMAB) and focus on reducing the dependency of the batch-size $K$ in the regret bound, where $K$ is the total number of arms that can be pulled or triggered in each round. First, for the setting of CMAB with probabilistically triggered arms (CMAB-T), we discover a novel (directional) triggering probability and variance modulated (TPVM) condition that can replace the previously-used smoothness condition for various applications, such as cascading bandits, online network exploration and online influence maximization. Under this new condition, we propose a BCUCB-T algorithm with variance-aware confidence intervals and conduct regret analysis which reduces the $O(K)$ factor to $O(\log K)$ or $O(\log^2 K)$ in the regret bound, significantly improving the regret bounds for the above applications. Second, for the setting of non-triggering CMAB with independent arms, we propose a SESCB algorithm which leverages on the non-triggering version of the TPVM condition and completely removes the dependency on $K$ in the leading regret. As a valuable by-product, the regret analysis used in this paper can improve several existing results by a factor of $O(\log K)$. Finally, experimental evaluations show our superior performance compared with benchmark algorithms in different applications.

\vspace{-3pt}

\end{abstract}

\section{Introduction}

Stochastic multi-armed bandit (MAB) \cite{robbins1952some,auer2002finite,bubeck2012regret} is a classical model that has been extensively studied in online decision making. As an extension of MAB, combinatorial multi-armed bandits (CMAB) have drawn much attention recently, owing to its wide applications in marketing, network optimization and online advertising \cite{gai2012combinatorial, kveton2015combinatorial, chen2013combinatorial, chen2016combinatorial, wang2017improving, merlis2019batch}.
In CMAB, the learning agent chooses a combinatorial action in each round, and this action would trigger a set of arms (or a super arm) to be pulled simultaneously, and the outcomes of these pulled arms are observed as feedback. Typically, such feedback is known as the semi-bandit feedback. The agent's goal is to minimize the expected \textit{regret}, which is the difference in expectation for the overall rewards between always playing the best action (i.e., the action with highest expected reward) and playing according to the agent's own policy. For CMAB, an agent not only need to deal with the exploration-exploitation tradeoff: whether the agent should explore arms in search for a better action, or should the agent stick to the best action observed so far to gain rewards; but also need to handle the exponential explosion of all possible actions.

To model a wider range of application scenarios where action may trigger arms probabilistically, \citet{chen2016combinatorial} first generalize CMAB to CMAB with probabilistically triggered arms (or CMAB-T for short), which successfully covers cascading bandit~\cite{combes2015combinatorial} (CB) and online influence maximization (OIM) bandit~\cite{wen2017online} problems. 
Later on, \citet{wang2017improving} improve the regret bound of \cite{chen2016combinatorial} by introducing a smoothness condition, 
	called the triggering probability modulated (TPM) condition, which removes a factor of $1/p^*$ compared to \cite{chen2016combinatorial}, where $p^*$ is the minimum positive probability that any arm can be triggered.
However, in both studies, the regret bounds still depend on a factor of \textit{batch-size} $K$, where $K$ is the maximum number of arms that can be triggered, and this factor could be quite large, e.g., for OIM $K$ can be as large as the number of edges in a large social network. 

\textbf{Our Contributions.} In this paper, we reduce or remove the dependency on $K$ in the regret bounds. 
For CMAB-T, we first discover a new {\em triggering probability and variance modulated (TPVM)} bounded smoothness condition, which is stronger than the TPM condition, yet still holds for several applications (such as CB and OIM) where only the TPM condition is known previously. We observe that for these applications, the previous TPM condition bounds the global speed of reward change regarding the parameter change, which will cause a large $K$ coefficient due to the rapid change at the boundary regions (i.e., when an arm's mean $\mu_i$ is close to $0$ or $1$). 
Our TPVM condition utilizes this observation by raising up the regret contribution of those boundary regions, leading to a significant reduction on the dependency of $K$.
Second, we propose a ``variance-aware" BCUCB-T algorithm that adaptively changes the width of the confidence interval according to the (empirical) variance, cancelling out the large regret contribution raised by the TPVM condition at the boundary regions (where the variances are also very small).
Combining these two techniques, we 
successfully reduce the batch-size dependence from $O(K)$ to $O(\log K)$ or $O(\log^2K)$ for all CMAB-T problems satisfying the TPVM condition, leading to significant improvements of the regret bounds for applications like CB or OIM. As a by-product, we also give refined proofs that shall improve the regret for several existing works by a factor of $O(\log K)$, e.g., \cite{degenne2016combinatorial, merlis2019batch}, which may be of independent interests.

In addition to the general CMAB-T setting, we show how a non-triggering version of the TPVM condition (i.e., VM condition) can help to completely remove the batch-size $K$, under the additional independent arm assumption for non-triggering CMAB problems. In particular, we propose a novel Sub-Exponential Efficient Sampling for Combintorial Bandits Policy (SESCB) that produces tighter sub-exponential concentrated confidence intervals. In our analysis, we show that the total regret only depends on the arm that is observed least instead of all $K$ arms, so that we can achieve a completely batch-size independent regret bound. 
Our empirical results demonstrate that our proposed algorithms can achieve around $20\%$ lower regrets than previous ones for several applications.
Due to the space limit, we will move the complete proofs and empirical results into the appendix.

\begin{table*}[t]
\centering
	\caption{Summary of the algorithms and results for CMAB with probabilistically triggered arms.}	\label{tab:triggering}
		\resizebox{0.97\columnwidth}{!}{
			\centering
	\begin{threeparttable}
	\begin{tabular}{|ccccc|}
		\hline
		\textbf{Algorithm}&\textbf{Smoothness}& \textbf{Independent Arms?} & \textbf{Computation}& \textbf{Regret}\\
		\hline
		CUCB~\cite{wang2017improving} & 1-norm TPM, $B_1$ & Not required & Efficient & $O(K  \sum_{i \in [m]}\frac{B_1^2\log T}{\Delta_i^{\min}})$ \\
		BOIM-CUCB \citep[Section 4]{perrault2020budgeted}$^*$ & 1-norm TPM, $B_1$ & Required & Hard & $O((\log K)^2  \sum_{i \in [m]}\frac{B_1^2\log T}{ \Delta_i^{\min}})$\\
		\hline
		\rowcolor{Gray}
		BCUCB-T (\cref{alg:BCUCB-T}) & \TPVMm, $B_v,^\dagger \lambda>1$ & Not required & Efficient & $O(\log K  \sum_{i \in [m]}\frac{B_v^2 \log T}{\Delta_i^{\min}})$ \\
        \rowcolor{Gray}
		BCUCB-T (\cref{alg:BCUCB-T}) & \TPVMm, $B_v,^\dagger
		\lambda=1$  & Not required & Efficient & $O((\log K  \log \frac{B_v K}{\Delta_{\min}}\sum_{i \in [m]}\frac{B_v^2 \log T}{\Delta_i^{\min}})$ \\
		\rowcolor{Gray}
		BOIM-CUCB (\OnlyInFull{\cref{apdx_sec:main_regret_analysis}}\OnlyInShort{Appendix C})$^\ddagger$ & 1-norm TPM, $B_1$ & Required & Hard & $O(\log K  \sum_{i \in [m]}\frac{B_1^2\log T}{ \Delta_i^{\min}})$\\
		\hline
	\end{tabular}
    \begin{tablenotes}[para, online,flushleft]
	\footnotesize
	\item[]\hspace*{-\fontdimen2\font}$^*$ This work is for a specific application, but we treat it as a general framework; 
	\item[]\hspace*{-\fontdimen2\font}$^\dagger$ Generally, $B_v=O(B_1\sqrt{K})$, and the existing regret bound is improved when $B_v=o(B_1\sqrt{K})$;
	\item[]\hspace*{-\fontdimen2\font}$^\ddagger$ Using our new analysis.
	\end{tablenotes}
	\end{threeparttable}
	}
\end{table*}

\begin{table*}[t]
	\caption{Summary of the algorithms and results for non-triggering CMAB problems.}	\label{tab:non-triggering}	
	\centering
	\resizebox{0.97\columnwidth}{!}{
	\centering
	\begin{threeparttable}
	\begin{tabular}{|ccccc|}
		\hline
		\textbf{Algorithm}&\textbf{Smoothness}& \textbf{Independent Arms?} & \textbf{Computation}& \textbf{Regret}\\
	\hline
		CUCB~\cite{wang2017improving} & 1-norm, $B_1$ & Not required & Efficient & $O(K  \sum_{i \in [m]}\frac{B_1^2\log T}{\Delta_i^{\min}})$ \\
		CTS \cite{wang2018thompson}$^*$ & 1-norm, $B_1$ & Required & Efficient & $O(K\sum_{i \in [m]}\frac{B_1^2\log T}{\Delta_i^{\min}})$\\
		ESCB \cite{combes2015combinatorial} & 1-norm, $B_1$$^{**}$ & Required & Hard & $O((\log K)^2  \sum_{i \in [m]}\frac{B_1^2\log T}{ \Delta_i^{\min}})$\\
			AESCB \cite{cuvelier2021statistically} & Linear & Required & Efficient & $O((\log K)^2  \sum_{i \in [m]}\frac{\log T}{ \Delta_i^{\min}})$\\

		BC-UCB~\cite{merlis2019batch}$^{\dagger}$ & VM, $B_v$ $^{\ddagger}$ & Not required & Efficient & $O((\log K)^2  \sum_{i \in [m]}\frac{B_v^2 \log T}{\Delta_i^{\min}})$.\\
		CTS \cite{wang2018thompson}$^*$ & Linear & Required & Efficient & $O(\log K  \sum_{i \in [m]}\frac{\log T}{ \Delta_i^{\min}})$\\
		\hline
        \rowcolor{Gray}
		SESCB (\cref{alg:SECUCB}) & VM, $B_v$ $^{\ddagger}$ & Required & {Efficient$^{***}$} & $O(\sum_{i \in [m]}\frac{B_v^2\log T}{\Delta_i^{\min}})$ \\
		\rowcolor{Gray}
		BC-UCB (\OnlyInFull{\cref{apdx_sec:main_regret_analysis}}\OnlyInShort{Appendix C})$^{\S}$  & VM, $B_v$ $^{\ddagger}$ & Not required & Efficient & $O(\log K  \sum_{i \in [m]}\frac{B_v^2 \log T}{\Delta_i^{\min}})$ \\
		\hline
	\end{tabular}
	  \begin{tablenotes}[para, online,flushleft]
	\footnotesize
	\item[]\hspace*{-\fontdimen2\font}$^*$ Requires exact offline oracle instead of $(\alpha, \beta)$-approximate oracle;
		\item[]\hspace*{-\fontdimen2\font}$^{\dagger}$ This work gives sufficient smoothness condition with factor $\gamma_g$ and translates to $B_v=3\sqrt{2}\gamma_g$ in our setting;
		\item[]\hspace*{-\fontdimen2\font}$^{\S}$ Using our new analysis.
	\item[]\hspace*{-\fontdimen2\font}$^{**}$ This work is for the linear case, but can easily generalize to 1-norm $B_1$ case; 
	\item[]\hspace*{-\fontdimen2\font}$^\ddagger$ Generally, $B_v=O(B_1\sqrt{K})$ and the existing regret bound is improved when $B_v=o(B_1\sqrt{K})$;
		\item[]\hspace*{-\fontdimen2\font}$^{***}$ Efficient when the reward function is submodular, otherwise the computation is hard;
	\end{tablenotes}
			\end{threeparttable}
	}
\end{table*}

\textbf{Related Work.} 
The stochastic CMAB has received much attention recently. From the modelling point of view, these CMAB works can be divided into two categories: 
CMAB with or without probabilistically triggered arms (i.e. CMAB-T setting or non-triggering CMAB). 
For CMAB-T, our work improves (a) the general framework in \cite{chen2016combinatorial, wang2017improving}, (b) the combinatorial cascading bandit \cite{kveton2015combinatorial}, (c) the online multi-layered network exploration~\cite{liu2021multi} problem, (d) the online influence maximization bandits \cite{wang2017improving, perrault2020budgeted}, by reducing or removing the batch-size dependent factor $K$ in the regret bounds with our new TPVM condition and/or our refined analysis. We defer the detailed technical comparison to \cref{sec:TPVM} and \cref{sec:app}. 
For the algorithm, most CMAB-T studies use Combinatorial Upper Confidence Bound (CUCB) based on Chernoff concentration bounds \cite{wang2017improving}, our BCUCB-T algorithm is different and uses the Bernstein concentration bound \cite{audibert2009exploration,merlis2019batch}
that considers variance of the arms. 

For non-triggering CMAB, \cite{gai2012combinatorial} is the first study on stochastic CMAB, and its regret has been improved by \citet{kveton2015tight}, \citet{combes2015combinatorial}, \citet{chen2016combinatorial}, but they still have $O(K)$ factor in their regrets. When arms are mutually independent, \citet{combes2015combinatorial} build a tighter ellipsoidal confidence region for exploration, and devise the Efficient Sampling for Combinatorial Bandit policy (ESCB), which reduces the dependence on $O(K)$ to $O(\log^2 K)$ at the cost of high computational complexity (since combinatorial optimization over the ellipsoidal region is NP-hard in general~\cite{atamturk2017maximizing}). Later on, the computational complexity is improved by AESCB~\cite{cuvelier2021statistically} in the linear CMAB problem. 
Recently, \citet{merlis2019batch} focus on the Probabilistic Maximum Coverage (PMC) bandit problem and propose the BC-UCB algorithm with the Gini-smoothness condition to achieve a similar improvement as ESCB/AESCB, but without the independent arm assumption. 
Our work is largely inspired by their work, however, our study generalizes theirs to the CMAB-T setting which can handle much broader application scenarios beyond the non-triggering CMAB
	(more detailed comparison is given in Section~\ref{sec:gini_cond}). 
In addition,  we provide a refined analysis that can save a $O(\log K)$ factor for BC-UCB (or ESCB/AESCB) algorithm. 
Compared with other ESCB-type algorithms for independent arms, as far as we know, our SESCB algorithm are the first to completely remove the dependence of $K$ in the leading regrets, owing to our non-triggering version of the TPVM condition. 
The detailed comparisons are summarized in \cref{tab:triggering} and \cref{tab:non-triggering}.

{ The usage of variance-aware algorithms to give improved regret bounds can be dated back to \cite{audibert2009exploration}. Recently, there is a surge of interest to apply the variance-aware principle in bandit~\cite{merlis2019batch,vial2022minimax} and reinforcement learning (RL) settings~\cite{zhou2021nearly,zhang2021improved}. It is notable that Vial et al.~\cite{vial2022minimax} share a similar variance-aware principle as ours but focus on the distribution-independent regret bounds for the cascading bandits~\cite{vial2022minimax}. Our work is more general and achieves the matching regret bound when translating to the distribution-independent regret bound. Compared with RL works, our paper studies a different setting as we do not consider the state transitions.}

From the application's point of view, this paper covers the applications of PMC
bandit~\cite{merlis2019batch}, combinatorial cascading bandits~\cite{kveton2015combinatorial,li2016contextual}, network exploration~\cite{liu2021multi}, and online influence maximization~\cite{wen2017online,wang2017improving, li2020online}. 
Our proposed algorithms can significantly reduce the regret bounds of them, e.g., from $O(K)$ to $O(\log^2 K)$ for OIM
where $K$ can be hundreds of thousands in large social networks.


\section{Problem Settings}\label{sec: problem setting}

\ts{The definition of the CMAB-T problem instance.}
We study the combinatorial multi-armed bandit problem with probabilistic triggering arms, which is denoted as CMAB-T for short.
Following the setting from~\cite{wang2017improving}, a CMAB-T {\em problem instance} can be described by a tuple $([m], \cS, \cD, D_{\text{trig}},R)$,
	where $[m]=\{1,2,...,m\}$ is the set of base arms; 
	$\cS$ is the set of eligible actions and $S \in \cS$ is an action;\footnote{In some cases $\cS$ is a collection of subsets of $[m]$, in which case we often 
	refer to $S\in \cS$ as a super arm. In this paper we treat $\cS$ as a general action space, same as in \cite{wang2017improving}.}
	$\cD$ is the set of possible distributions over the outcomes of base arms with bounded support $[0,1]^m$;
$ D_{\text{trig}}$ is the probabilistic triggering function and $R$ is the reward function, the definitions of which will be introduced shortly.

\ts{The environment and interaction protocol.}
In CMAB-T, the learning agent interacts with the unknown environment in a sequential manner as follows.
First, the environment chooses a distribution $D \in \cD$ unknown to the agent.
Then, at round $t=1,2,...,T$, the agent selects an action $S_t \in \cS$ and the environment draws from the unknown distribution $D$ a random outcome $\bX_t=(X_{t,1},...X_{t,m})\in [0,1]^m$. 
Note that the outcome $\bX_t$ is assumed to be independent from outcomes generated in previous rounds, but outcomes $X_{t,i}$ and $X_{t,j}$ in the same round could be correlated.
Let $D_{\text{trig}}(S,\bX)$ be a distribution over all possible subsets of $[m]$, i.e. its support is $2^{[m]}$.
When the action $S_t$ is played on the outcome $\bX_t$, base arms in a random set $\tau_t \sim D_{\text{trig}}(S_t, \bX_t)$ are triggered, 
	meaning that the outcomes of arms in $\tau_t$, i.e. $(X_t)_{t\in \tau_t}$ are revealed as the feedback to the agent, and are involved in determining the reward of action $S_t$.
Function $D_{\text{trig}}$ is referred as the \textit{probabilistic triggering function}.
At the end of the round $t$, the agent will receive a non-negative reward $R(S_t, \bX_t, \tau_t)$, determined by $S_t,
\bX_t$ and $\tau_t$.
CMAB-T significantly enhances the modeling power of CMAB~\cite{chen2013combinatorial, kveton2015tight} and can model many applications such as cascading bandits and online influence maximization~\cite{wang2017improving}, which we will discuss in later sections.

\ts{The goal and the definition of the approximate regret.}
The goal of CMAB-T is to accumulate as much reward as possible over $T$ rounds, by learning distribution $D$ or its parameters.
Let $\bmu=(\mu_1,...,\mu_m)$ denote the mean vector of base arms' outcomes.
Following~\cite{wang2017improving}, we assume that the expected reward $\E[R(S,\bX,\tau)]$ is a function of 
the unknown mean vector $\bmu$, where the expectation is taken over the randomness of $\bX\sim D$ and $\tau \sim D_{\text{trig}}(S,\bX)$.
In this context, we denote $r(S;\bmu)\triangleq \E[R(S,\bX,\tau)]$ and it suffices to learn the unknown mean vector instead of the joint distribution $D$, based on the past observation.
To allow the algorithm to estimate the mean $\mu_i$ directly from samples, we assume the outcome does not depend on whether the arm $i$ is triggered, i.e., $\E_{\bX \sim D, \tau \sim D_{\text{trig}}(S,\bX)}[X_i | i\in \tau]=\E_{\bX\sim D}[X_i]$.

The performance of an online learning algorithm $A$ is measured by its {\em regret}, defined as the difference of the expected cumulative reward between always playing the best action $S^* \triangleq \argmax_{S \in \cS}r(S;\bmu)$ and playing actions chosen by algorithm $A$.
For many reward functions, it is NP-hard to compute the exact $S^*$ even when $\bmu$ is known, so similar to~\cite{wang2017improving}, we assume that the algorithm $A$ has access to an offline $(\alpha, \beta)$-approximation oracle, which for mean vector $\bmu$ outputs an action $S$ such that $\Pr\left[r(S;\bmu)\ge \alpha \cdot r(S^*;\bmu)\right] \ge \beta$. 
Formally, the $T$-round $(\alpha, \beta)$-approximate regret is defined as
\resizebox{1.0\columnwidth}{!}{
\begin{minipage}{\columnwidth}
\begin{equation}
    Reg(T;\alpha, \beta, \bmu)= T \cdot \alpha\beta \cdot r(S^*;\bmu)-\E\left[\sum_{t=1}^Tr(S_t;\bmu)\right],
\end{equation}
\end{minipage}}
where the expectation is taken over the randomness of outcomes $\bX_1, ..., \bX_T$, the triggered sets $\tau_1, ..., \tau_T$, as well as the randomness of algorithm $A$ itself.

In the CMAB-T model, there are several quantities that are crucial to the subsequent study.
We define {\em triggering probability} $p_i^{D,D_{\text{trig}},S}$ as the probability that base arm $i$ is 
	triggered when the action is $S$, the outcome distribution is $D$, and the probabilistic triggering function is $D_{\text{trig}}$.
Since $D_{\text{trig}}$ is always fixed in a given application context, we ignore it in the notation for simplicity, and use $p_i^{D,S}$ henceforth.
Triggering probabilities  $p_i^{D,S}$'s are crucial for the triggering probability modulated bounded smoothness conditions to be defined below.
We define {\em batch size} $K$ as the maximum number of arms that can be triggered, i.e., $K=\max_{S \in \cS}|\{i \in [m]: p_i^{D,S} > 0\}|$.
Our main contribution of this paper is to remove or reduce the regret dependency on batch size $K$, where $K$ could be quite large, e.g., $K$ can be hundreds of thousands in a large social network.

%

\ts{The motivation of smoothness conditions}
Owing to the nonlinearity and the combinatorial structure of the reward, it is essential to give some conditions for the reward function in order to achieve any meaningful regret bounds~\cite{chen2013combinatorial, chen2016combinatorial, wang2017improving, degenne2016combinatorial,merlis2019batch}. 
The following are two standard conditions originally proposed by \citet{wang2017improving}.

\ts{Two standard conditions for the reward function.}

\begin{condition}[Monotonicity]\label{cond:mono}
We say that a CMAB-T problem instance satisfies monotonicity condition, if for any action $S \in \cS$, any two distributions $D,D' \in \cD$ with mean vectors $\bmu,\bmu' \in  [0,1]^m$ such that $\mu_i \le \mu'_i $ for all $i \in [m]$, we have $ r(S;\bmu) \le r(S;\bmu') $. 
\end{condition}

\begin{condition}[1-norm TPM Bounded Smoothness]\label{cond:TPM}
We say that a CMAB-T problem instance satisfies the triggering probability modulated (TPM) $B_1$-bounded smoothness condition, if for any action $S \in \cS$, any distribution $D,D'\in \cD$ with mean vectors $\bmu, \bmu' \in [0,1]^m$, we have $|r(S;\bmu')-r(S;\bmu)|\le B_1\sum_{i \in [m]}p_{i}^{D,S}|\mu_i-\mu'_i|$.
\end{condition}
The first monotonicity condition indicates the reward is larger if the parameter vector $\bmu$ is larger.
The second condition bounds the reward difference caused by the parameter change (from $\bmu$ to $\bmu'$).
One key feature is that the parameter change in each base arm $i \in [m]$ is modulated by the triggering probability $p_i^{D,S}$.
Intuitively, for base arm $i$ that is unlikely to be triggered/observed (small $p_i^{D,S}$), Condition~\ref{cond:TPM} ensures that a large change in $\mu_i$ 
only causes a small change (multiplied by $p_i^{D,S}$) in the reward, and thus one does not need to pay extra cost to observe such arms.
%
Many applications satisfy Condition~\ref{cond:mono} and Condition~\ref{cond:TPM}, including linear combinatorial bandits~\cite{kveton2015tight}, combinatorial cascading bandits~\cite{kveton2015combinatorial}, online influence maximization~\cite{wang2017improving}, etc.
With the above two conditions, \citet{wang2017improving} show that a CUCB algorithm achieves
	the distribution-dependent regret bound of $O(\sum_{i \in [m]}\frac{B_1^2K\log T}{\Delta_i^{\min}})$,
	where $\Delta_i^{\min}$ is the distribution-dependent reward gap, to be formally defined in Definition~\ref{def:gap}.
In the following sections, we will show how to remove or reduce the dependency on $K$ in the above bounds under our new conditions. 

\section{Algorithm and Regret Analysis for CMAB-T}\label{sec:gini_cond}
\ts{The definition of the TPM gini-smoothness}
In this section, for the CMAB-T framework with probabilistic triggering, we improve the regret dependency on the batch size from $O(K)$ in \cite{wang2017improving} to $O(\log K)$ or $O(\log^2 K)$.
Our main tool is a new condition called {\em triggering probability and variance modulated (TPVM) bounded smoothness condition}, replacing the TPM condition (Condition~\ref{cond:TPM}).
We will define the TPVM condition, comparing it with the TPM condition and the gini-smoothness condition of~\cite{merlis2019batch}, show our algorithm and regret analysis that utilize this condition. Later in \cref{sec:app}, we will
demonstrate how this condition is applied to applications such as cascading bandits and online influence maximization.

\subsection{Triggering Probability and Variance Modulated (TPVM) Bounded Smoothness Condition}
\label{sec:TPVM}

In this paper, we discover a new smoothness condition for many important applications as follows. 

\begin{condition}[Directional TPVM Bounded Smoothness]\label{cond:TPVMm}
We say that a CMAB-T problem instance satisfies the directional TPVM $(B_v, B_1,\lambda)$-bounded smoothness condition ($B_v,B_1\ge 0, \lambda\ge 1$), 
	if for any action $S \in \cS$, any distribution $D,D' \in \cD$ with mean vector $\bmu, \bmu' \in (0,1)^m$, for any non-negative $\boldzeta, \boldeta \in [0,1]^m$ s.t. $\bmu'=\bmu+\boldzeta+\boldeta$, we have
	\resizebox{1.0\columnwidth}{!}{
\begin{minipage}{\columnwidth}
\begin{align}\label{eq:gini}
   &|r(S;\bmu')-r(S;\bmu)|\le B_v\sqrt{\sum_{i\in [m]}(p_i^{D,S})^{\lambda}\frac{\zeta_i^2 }{(1-\mu_i)\mu_i}} + B_1 \sum_{i\in[m]}p_i^{D,S}\eta_i.
\end{align}
\end{minipage}}
\end{condition}

\textbf{Remark 1 (Intuition for Condition \ref{cond:TPVMm}).} 
Looking at \cref{eq:gini}, if we ignore the $(1-\mu_i)\mu_i$ term in the denominator and set $\lambda=2$, the RHS of \cref{eq:gini} becomes $B_v\sqrt{\sum_{i\in [m]}(p_i^{D,S})^{2}\zeta_i^2} + B_1 \sum_{i\in[m]}p_i^{D,S}\eta_i$, which holds with $B_v=B_1\sqrt{K}$ by applying the Cauchy-Schwarz inequality to Condition~\ref{cond:TPM}.
However, the regret upper bound following this modified \cref{eq:gini} would not directly lead to the improvement in the regret due to the $\sqrt{K}$ factor in $B_v$.
To deal with this issue, an important observation here is that for many applications, the reason $B_v$ is large is because that the reward changes abruptly when parameters $\mu_i$ approaches $0$ or $1$. This motivates us to plug in the $1/(1-\mu_i)\mu_i$ term in \cref{eq:gini} to enlarge the square root term when $\mu_i$ is close to $0$ or $1$, so that $B_v$ can be as small as possible. 
On the other hand, notice that when $\mu_i$ approaches $0$ or $1$, the variance $V_i\le(1-\mu_i)\mu_i$  is also very small,
\footnote{For bounded random variable $X\in [0,1]$ with mean $\mu_i$, variance $V_i=\E[X^2]-\E[X]^2\le \E[X]-(\E[X])^2\le (1-\mu_i)\mu_i$, where the equality is achieved when $X$ is a Bernoulli random variable. }
so the estimation of $\mu_i$ should be quite accurate. 
Therefore, the gap $\zeta_i$ between our estimation and true value produces a variance-related term which cancels the $(1-\mu_i)\mu_i$ in the denominator. Since $\zeta_i$ in \cref{eq:gini} is modulated by both triggering probability $p_i^{D,S}$ and inverse upper bound of the variance $1/(1-\mu_i)\mu_i$, we call Condition~\ref{cond:TPVMm} the directional triggering probability and variance modulated (TPVM) condition for short, where the term ``directional'' is explained in the next remark.
The exponent $\lambda \ge 1$ on the triggering probability gives flexibility to trade-off between the strength of the condition and the quantity of the regret bound:
With a larger $\lambda$, we can obtain a smaller regret bound, while
	with a smaller $\lambda$, the condition is easier to satisfy and allows us to include more applications.

\textbf{Remark 2 (On directional TPVM vs. undirectional TPVM).} 
In the above definition, ``directional'' means that we have $\boldzeta,\boldeta \ge \bzero$ such that $\bmu' \ge \bmu$ in every dimension.
This is weaker than the version of the undirectional TPVM condition, where $\boldzeta, \boldeta \in [-1,1]^m$, and the $\eta_i$ in the right hand side of Eq.\eqref{eq:gini} is replaced with $|\eta_i|$.
The reason we use the weaker version is that some of our applications considered in this paper only satisfy the weaker version.
To differentiate, we use {\TPVMm} when we refer to the directional TPVM condition.

\textbf{Remark 3 (Relation between Conditions \ref{cond:TPM} and \ref{cond:TPVMm}).}
First, when setting $\boldzeta$ to $\bzero$, the directional TPVM condition degenerates to the directional TPM condition.
However, Condition \ref{cond:TPM} is the undirectional TPM condition, which is typically stronger than its directional counterpart.
Thus, in general Condition \ref{cond:TPVMm} does not imply Condition \ref{cond:TPM}.
Nevertheless, with some additional assumptions Condition \ref{cond:TPVMm} does imply Condition \ref{cond:TPM} with the same
	coefficient $B_1$ (See Appendix~\OnlyInFull{\ref{app:TPVMvsTPM}}\OnlyInShort{A} for an example of such assumptions).
Conversely, by applying the Cauchy-Schwartz inequality, one can verify that if a reward function is TPM $B_1$-bounded smooth, then it is 
	(directional) TPVM $(B_1\sqrt{K}/2, B_1, \lambda)$-bounded smooth for any $\lambda\le 2$.
For applications considered in this paper, we are able to reduce their $B_v$ coefficient from $B_1\sqrt{K}/2$ to a coefficient independent of $K$, leading to significant savings in the regret bound. 

\textbf{Remark 4 (Comparing with \cite{merlis2019batch}).}
\citet{merlis2019batch} introduce a Gini-smoothness condition to reduce the batch-size dependency for CMAB problems, which largely inspires our {\TPVMm} condition.
Their condition is specified in a differential form of the reward function, with parameters $\gamma_\infty$ and $\gamma_g$ (See Appendix~\OnlyInFull{\ref{app:gini}}\OnlyInShort{B} for the exact definition).
We emphasize that their original condition cannot handle the probabilistic triggering setting in CMAB-T.
One natural extension is to incorporate triggering probability modulation into their differential form of Gini-smoothness.
However, we found that the resulting TPM Gini-smoothness condition is not strong enough to guarantee desirable regret bounds (See Appendix~\OnlyInFull{\ref{apdx_sec:direct_tpvm}}\OnlyInShort{B.1}).
This motivates us to provide a new condition directly on the difference form $\abs{r(S;\bmu')-r(S;\bmu)}$, similar to the TPM condition in~\cite{wang2017improving}.
Our {\TPVMm} condition (Condition~\ref{cond:TPVMm}) can be viewed as extending Lemma 6 of \cite{merlis2019batch} to incorporate triggering probabilities and
	bound the difference form $\abs{r(S;\bmu')-r(S;\bmu)}$.
Intuitively, $B_1$ and $B_v$ correspond to $\gamma_\infty$ and $\gamma_g$, respectively, but since they are for different forms of definitions, their numerical
	values may not exactly match one another.



\begin{algorithm}[t]
	\caption{BCUCB-T: Bernstein Combinatorial Upper Confidence Bound Algorithm for CMAB-T}\label{alg:BCUCB-T}
			\resizebox{.94\columnwidth}{!}{
\begin{minipage}{\columnwidth}
	\begin{algorithmic}[1]
	    \State {\textbf{Input:}} Base arms $[m]$, computation oracle ORACLE.
	   \State \textbf{Initialize:} For each arm $i$, $T_{0,i} \leftarrow 0, \hat{\mu}_{0,i}=0, \hat{V}_{0,i}=0$. 
	   \For{$t=1, ...,T$ }
	   \State For arm $i$, compute $\rho_{t,i}$ according to \cref{eq:confidence_interval} and set UCB value $\bar{\mu}_{t,i}=\min \{\hat{\mu}_{t-1, i}+\rho_{t, i},1\}$.
	   \State $S_t=\text{ORACLE}(\bar{\mu}_{t,1}, ..., \bar{\mu}_{t,m})$.
	   \State Play $S_t$, which triggers arms $\tau_t \subseteq [m]$ with outcome $X_{t,i}$'s, for $i \in \tau_t$.\label{line:trigger_obs}
	    \State For every $i \in \tau_t$, update $T_{t,i}= T_{t-1, i}+1$, $\hat{\mu}_{t,i}= \hat{\mu}_{t-1,i}+(X_{t,i}-\hat{\mu}_{t-1,i})/T_{t,i}$,
	   $\hat{V}_{t, i}= \frac{T_{t-1, i}}{T_{t, i}}\left(\hat{V}_{t-1,i} + \frac{1}{T_{t, i}} \left(\hat{\mu}_{t-1,i} - X_{t,i}\right)^2\right) $.
	    \label{line:trigger_upd}
	    	   \EndFor
		\end{algorithmic}
		\end{minipage}}
\end{algorithm}

\subsection{BCUCB-T Algorithm and Regret Analysis}\label{sec:BCUCB_reg}
\ts{Describe the algorithm}
Our proposed algorithm BCUCB-T is a generalization of the BC-UCB algorithm~\citep[Algorithm 1]{merlis2019batch} which originally solves the non-triggering CMAB problem. 
\Cref{alg:BCUCB-T} maintains the empirical estimate $\hat{\mu}_{t,i}$ and $\hat{V}_{t,i}$ for the true mean and the true variance of the base arm outcomes.
To select the action $S_t$, it feeds the upper confidence bound $\bar{\mu}_i$ into the offline oracle, where $\bar{\mu}_i$ optimistically estimates the $\mu_i$ by a confidence interval $\rho_{t,i}$.
Compared with the CUCB algorithm~\citep[Algorithm 1]{wang2017improving} which uses confidence interval $\rho_{t,i}=\sqrt{\frac{3\log t}{2T_{t-1,i}}}$ for the CMAB-T problem, the novel part is the usage of empirical variance $\hat{V}_{t-1, i}$ to construct the following ``variance-aware" confidence interval:
\noindent
\resizebox{1.0\columnwidth}{!}{
\begin{minipage}{\columnwidth}
\begin{equation}\label{eq:confidence_interval}
    \rho_{t,i}=\sqrt{\frac{6\hat{V}_{t-1, i} \log t}{T_{t-1, i}}} + \frac{9\log t}{T_{t-1, i}}
\end{equation}
\end{minipage}} 

This confidence interval leverages on the empirical Bernstein inequality instead of the Chernoff-Hoeffding inequality.
As we will show in Appendix \OnlyInFull{\ref{apdx_sec:tpvm_useful}}\OnlyInShort{C.1}, for the first term in \cref{eq:confidence_interval}, $\hat{V}_{t-1, i}$ is approximately equal to the true variance $V_{i}\le(1-\mu_i)\mu_i$ and 
	this indicates the estimation of $\mu_i$ is more accurate when $\mu_i$ is close to $0$ or $1$, which will cancel out the $(1-\mu_i)\mu_i$ coefficient of the $B_v$ term in  Condition \ref{cond:TPVMm} as we discussed before.
The second term of \cref{eq:confidence_interval} is to compensate the usage of the empirical variance $\hat{V}_{t-1, i}$, rather than the true variance $V_{i}$ which is unknown to the learner.


To state the regret bound, we first give some definitions followed by our main result. 

\begin{definition}[(Approximation) Gap]\label{def:gap}
Fix a distribution $D \in \cD$ and its mean vector $\bmu$, for each action $S \in \cS$, we define the (approximation) gap as $\Delta_S=\max\{0, \alpha r(S^*;\bmu)-r(S;\bmu)\}$. 
For each arm $i$, we define $\Delta_i^{\min}=\inf_{S \in \cS: p_i^{D,S} > 0,\text{ } \Delta_S > 0}\Delta_S$, $\Delta_i^{\max}=\sup_{S \in \cS: p_i^{D,S} > 0, \Delta_S > 0}\Delta_S$.
As a convention, if there is no action $S \in \cS$ such that $p_i^{D,S}>0$ and $\Delta_S>0$, 
	then $\Delta_i^{\min}=+\infty, \Delta_i^{\max}=0$. We define $\Delta_{\min}=\min_{i \in [m]} \Delta_i^{\min}$ and $\Delta_{\max}=\max_{i \in [m]} \Delta_i^{\min}$.
\end{definition}

\begin{restatable}{theorem}{lambda1}\label{thm:reg_lambda1}
For a CMAB-T problem instance $([m], \cS, \cD, D_{\text{trig}},R)$ that satisfies monotonicity (Condition \ref{cond:mono}), 
and {\TPVMm} bounded smoothness (Condition \ref{cond:TPVMm}) with coefficient $(B_v, B_1, \lambda)$, 

(1) if $\lambda > 1$, { BCUCB-T (\cref{alg:BCUCB-T}) with an $(\alpha,\beta)$-approximation oracle achieves an $(\alpha,\beta)$-approximate regret} bounded by
\begin{equation}\label{eq:reg_lambda1}
    O\left(\sum_{i \in [m]}\frac{B_v^2 \log K \log T}{\Delta_i^{\min}} + \sum_{i \in [m]}B_1\log^2\left(\frac{B_1K}{\Delta_i^{\min}}\right)\log T\right);
\end{equation}

(2) if $\lambda=1$,  { BCUCB-T (\cref{alg:BCUCB-T}) with an $(\alpha,\beta)$-approximation oracle achieves an $(\alpha,\beta)$-approximate regret} bounded by
\begin{equation}\label{eq:reg_lambda2}
    O\left(\sum_{i \in [m]} \log\left(\frac{B_vK}{\Delta_i^{\min}}\right) \frac{B_v^2\log K \log T}{\Delta_i^{\min}} + \sum_{i \in [m]}B_1\log^2\left(\frac{B_1K}{\Delta_i^{\min}}\right)\log T\right).
\end{equation}
\end{restatable}

\textbf{Remark 5 (Discussion for Regret Bounds).} 
Looking at the above regret bounds, for $\lambda > 1$ and $\lambda=1$, the leading terms are $O( \sum_{i=1}^m\frac{ B_v^2\log K \log T}{\Delta_{i}^{\min}} )$ and $O( \sum_{i=1}^m(\log \frac{B_vK}{\Delta_i^{\min}})\frac{  B_v^2\log K \log T}{\Delta_{i}^{\min}})$.
When $B_v \ge B_1$ (which typically holds, see \cref{sec:app}) and gaps are small (i.e., $\Delta^i_{\min} \le 1/\log^2K$),  
	the dependencies over $K$ are $O(\log K)$ and $O(\log^2 K)$, respectively. 
For the setting of CMAB-T, \cite{wang2017improving} is the closest work to our paper, where the reward function satisfies Condition \ref{cond:mono} and Condition \ref{cond:TPM} with coefficient $B_1$. As mentioned in Remark 3 in \cref{sec:TPVM}, their reward function trivially satisfies our Condition~\ref{cond:TPVMm} with coefficient $(B_1\sqrt{K}/2, B_1, 2)$ so our work reproduces a bound of $O(\sum_{i \in [m]}\frac{B_1^2K\log K\log T}{\Delta_i^{\min}})$, matching \cite{wang2017improving} up to a factor of $O(\log K)$. As will be shown in \cref{sec:app}, for applications that satisfy TPVM (or {\TPVMm}) condition with non-trivial $B_v$, i.e., $B_v=o(B_1 \sqrt{K})$, our work improves their regret bounds up to a factor of $O(K/\log K)$. As for the lower bound, according to the lower bound results by~\citet{merlis2020tight}, our regret bound is tight up to a factor of $O(\log^2 K)$ on the (degenerate) non-triggering CMAB case. We defer the details about the lower bound results and the distribution-independent regret bounds in the Appendix \OnlyInFull{\ref{apdx_sec:summary_reg}}\OnlyInShort{C.5}.

\begin{proof}[Proof ideas.]
Our proof uses a few events to filter the total regret and then bound these event-filtered regrets separately.
As will be shown in the supplementary material, the event that contributes to the leading regret is $ E_t=\{ \Delta_{S_t} \le e_t(S_t)\}$, where the error term $e_t(S_t)= O(B_v\sqrt{\sum_{i\in \tilde{S}_t}(\frac{\log t}{T_{t-1,i}})(p_{i}^{D,S_t})^\lambda} + B_1\sum_{i \in \tilde{S}_t}(\frac{\log t}{T_{t-1,i}})(p_{i}^{D,S_t}))$.
To handle the probabilistic triggering, our key ingredient is to use the triggering probability group technique proposed by~\citet{wang2017improving} in the definition of above events.
For the $\lambda=1$ case, one new issue arises since the triggering probability group divides sub-optimal actions $S$ into \textit{infinite} geometrically separated bins $(1/2,1], (1/4,1/2], ..., (2^{-j}, 2^{-j+1}),...,$ over $p_i^{D,S}$, and the regret should be proportional to the number of bins (which are infinitely large).
To handle this, we show that it suffices to consider the first $j \le j_i^{\max}=O(\log \frac{B_vK}{\Delta_i^{\min}})$ bins (which is why \cref{eq:reg_lambda2} has this additional factor in the leading term) and the regret of other bins (with very small $p_i^{D,S}$) can be safely neglected.
To bound the leading regret filtered by $E_t$ as mentioned earlier, we use the reverse amortization trick from~\citet{wang2017improving,wang2018thompson} and adaptively allocates each arm's regret contribution (according to thresholds on the number of times arm $i$ is triggered).  
Note that these thresholds are carefully chosen for the error term $e_t(S_t)$, since trivially following the thresholds in~\citet{wang2017improving} would either yield no meaningful bound or 
	suffer from additional $O(\log T)$ or $O(\log K)$ factors in the regret. 
	As a by-product, one can also use our analysis to replace that of~\citet{merlis2019batch} and~\citet{perrault2020budgeted} (where similar error term $e_t(S_t)$ appears) to improve their bound by a factor of $O(\log K)$. For the detailed proofs, we defer them in the Appendix~\OnlyInFull{\ref{apdx_sec:main_regret_analysis}}\OnlyInShort{C}.
\end{proof}

\section{Algorithm and Analysis For CMAB with Independent Arms}\label{sec:independent}
\ts{Two additional conditions for the independent case.}
In this section, we aim to show that for the non-triggering CMAB, the assumption that all arms are independent, compounded with 
	a non-triggering version of the above TPVM condition (named as VM condition below), together allow us to completely remove the $O(\log^2 K)$ or $O(\log K)$ dependence in the existing regret bounds.
In particular, we focus on the a non-triggering CMAB problem instance $([m],\cS, \cD, R)$. Its setting is similar to CMAB-T, but here 
	we assume that $\cS$ are collections of subsets of $[m]$ and only arms pulled by action $S_t\in\cS$ are revealed as feedback (i.e., $\tau_t=S_t$). 


\begin{condition}[VM Bounded Smoothness]\label{cond:VM}
We say that a non-triggering CMAB problem instance $([m],\cS, \cD, R)$ satisfies the Variance Modulated (VM) $(B_v, B_1\ge 0 )$-bounded smoothness condition, if for any action $S \in \cS$, 
	any distribution $D,D' \in \cD$ with mean vector $\bmu, \bmu' \in (0,1)^m$, for any $\boldzeta, \boldeta\in [-1,1]^m$ 
	s.t. $\bmu'=\bmu+\boldzeta+\boldeta$, we have $|r(S;\bmu')-r(S;\bmu)|\le B_v\sqrt{\sum_{i\in S}\frac{\zeta_i^2 }{(1-\mu_i)\mu_i}} + B_1 \sum_{i\in[m]}|\eta_i|$.
\end{condition}

\begin{condition}[Independent base arms]\label{cond:independent}
We say that the base arms are independent, if for any $D \in \cD$, the outcome vectors $\bX\sim D$ are independent (across base arms), i.e., $D=\otimes_{i \in [m]} D_i$. 
\end{condition}

\begin{condition}[$C_1 \mu_i(1-\mu_i)$ sub-Gaussian]\label{cond:subg} The outcome distribution $D_i$ with mean $\mu_i$ is $C_1 \mu_i(1-\mu_i)$ sub-Gaussian, where $C_1$ is a known coefficient.
\end{condition}

\textbf{Remark 6 (Comparison with TPVM Condition and \cite{merlis2019batch}).} Condition~\ref{cond:VM} is the non-triggering version of TPVM, by setting $p_i^{D,S}=1$ if $i \in S$ and $0$ otherwise. 
As shown in Appendix~\OnlyInFull{\ref{apdx_sec:derive_vm}}\OnlyInShort{B.2}, Condition~\ref{cond:VM} can be implied by the original Gini-smoothness condition~\cite{merlis2019batch} with $(B_v,B_1)=(3\sqrt{2}\gamma_g,\gamma_{\infty})$, so PMC application satisfies the VM condition (the fifth row in~\cref{tab:app}). But different from~\citep[Lemma 6]{merlis2019batch} and {\TPVMm}, the VM condition is the undirectional version (i.e., we allow $\boldzeta, \boldeta$ to be negative). This is important for using empirical means in the algorithm (as we did in our SESCB policy), since they are not necessarily larger than the true means. 
%

\textbf{Remark 7 (Motivation and Feasibility for Condition~\ref{cond:subg}).}
Condition~\ref{cond:subg} helps to cancel out the $(1-\mu_i)\mu_i$ effect in the VM condition 
without explicitly using the empirical variance that will bring in additional batch-size dependent errors.  For Bernoulli arms with mean $\mu_i$, we can compute the explicit value of $C_1$, i.e., $C_1=\max_{i \in [m]}\frac{1-2\mu_i}{2\ln(\frac{1-\mu_i}{\mu_i})(1-\mu_i)(\mu_i)}$ by \cite{marchal2017sub}. Notice that $C_1$ could be large when $\mu_i$ is approaching $0$ or $1$, but it is safe to consider $\mu_i$ over bounded supports that are not too close to $0$ or $1$, e.g., when $\mu_i \in [0.01, 0.99]$, $C_1\approx 10.78$. 



\textbf{SESCB Algorithm.} Our proposed algorithm is shown in \cref{alg:SECUCB}. Instead of maintaining one upper confidence bound for each base arm $i$, we maintain an upper confidence bound for each super arm $S$, based on the estimated reward of the empirical means and a confidence interval.
In line~\ref{line:SE_interval}, we compute the confidence interval $\rho_{t}(S)$ by taking the max of two tentative segments within the square root, which corresponds to two different segments of the concentration bound for the sub-exponential random variable~\cite{vershynin2018high}. Such a sub-exponential concentrated confidence interval comes from the VM condition by treating $(\zeta_i)_{i \in S}$ as $|S|$ \textit{independent} sub-Gaussian random variables, whose summation produces a more concentrated sub-exponential random variable compared with considering them as $|S|$ possibly dependent variables. 
It is notable that for the second tentative interval, SESCB uses the min-counter $T^{\min}_{t-1,S}$ instead of all counters $T_{t-1,i}$ in $S$, which is the key ingredient that removes the $O(\log K)$ factor 
	as to be shown in the analysis.
After getting $\rho_t(S)$, {the optimistic reward is defined in~\cref{line:SE_opt_reward} and the learner selects $S_t$ via the $(\alpha,\beta)$-approximation oracle $\bar{O}$ and updates the corresponding statistics}. 

\begin{algorithm}[t]
	\caption{SESCB: Sub-Exponential Sampling for Combinatorial Bandits with Independent Arms}\label{alg:SECUCB}
		\resizebox{.93\columnwidth}{!}{
\begin{minipage}{\columnwidth}
	\begin{algorithmic}[1]
	    \State \textbf{Input:} Base arms $[m]$, sub-Gaussian parameter $C_1$, VM smoothness coefficient $B_v$, {$(\alpha,\beta)$-approximation ORACLE $\bar{O}$}.
	   \State \textbf{Initialize:} For each arm $i$, $T_{0,i} \leftarrow 0, \hat{\mu}_{0,i}=0$.
	   \For{$t=1, ...,T$ }
	  \State  {For all $S \in \cS$, define min-count $T^{\min}_{t-1, S}=\min_{i \in S}T_{t-1, i}$, let interval $\rho_{t}(S)=B_v\sqrt{\sum_{i \in S} \frac{C_1}{T_{t-1,i}}+ \max\left\{8C_1\sqrt{\sum_{i \in S}\frac{\log(2|\cS|T)}{T_{t-1,i}^2}}, \frac{8C_1\log(2|\cS|T)}{T^{\min}_{t-1, S}}\right\}}$.}\label{line:SE_interval}
	   \State {For all $S \in \cS$, define optimistic reward $\bar{r}_t(S)=r(S; \hat{\bmu}_{t-1})+\rho_{t}(S)$.}\label{line:SE_opt_reward}
	   \State Play {$S_t=\bar{O}(\hat{\bmu}_t, \bT_t)$ s.t. $\Pr\left[\bar{r}_t(S)\ge \alpha \cdot \bar{r}_t(\bar{S}^*_t)\right] \ge \beta$,where $\bar{S}^*_t=\arg\max_{S\in \mathcal{S}} \bar{r}_t(S)$}, and observe outcome $X_{t,i}$'s, for $i \in S_t$.\label{line:SE_oracle}
	    \State For every $i \in S_t$, update $T_{t,i}= T_{t-1, i}+1$, $\hat{\mu}_{t,i}= \hat{\mu}_{t-1,i}+(X_{t,i}-\hat{\mu}_{t-1,i})/T_{t,i}$.
	    	   \EndFor
		\end{algorithmic}   
				\end{minipage}}
\end{algorithm}
\textbf{Regret Bound and Analysis.} The following theorem summarizes the regret bound for \cref{alg:SECUCB}.
\begin{restatable}{theorem}{IndependentGini}\label{thm:independent}
For a non-triggering CMAB problem instance $([m],\cS, \cD, R)$ that satisfies VM bounded smoothness (Condition \ref{cond:VM}) with coefficient $(B_v, B_1)$, Condition~\ref{cond:independent} and Condition~\ref{cond:subg} with coefficient $C_1$, {SESCB (\cref{alg:SECUCB}) with an $(\alpha,\beta)$-approximation oralce achieves $(\alpha,\beta)$-approximate regret} that is bounded by $O\left(\sum_{i \in [m]}\frac{B_v^2\log T}{\Delta_i^{\min}} + \frac{B_v^2mK}{\Delta_{\min}}+m\Delta_{\max}\right)$.
\end{restatable}
Looking at the above regret bound, the leading term totally removes the $O(\log K)$ dependency compared with \cref{thm:reg_lambda1}. Compared with~\cite{merlis2019batch}, our regret bounds improves theirs by $O(\log ^2K)$.
\begin{proof}[Proof Ideas.]
Similar to the proof of \cref{thm:reg_lambda1}, we first identify an error term $e_t(S_t)=2\rho_t(S_t)$ as~\cref{line:SE_interval} and consider the regret filtered by the event $\{\Delta_{S_t} \le e_t(S_t)\}$. The key ingredient is by following Condition \ref{cond:VM} and Condition \ref{cond:subg}, and bound $|r(S;\hat{\bmu})-r(S;\bmu)|\le B_v \sqrt{\sum_{i \in S}u_{t,i}^2}$, where $u_{t,i}$ is a ($\frac{C_1}{T_{t-1,i}}$)-sub-Gaussian random variable. Let $Y_{t,S}=\sum_{i \in S}u^2_{t,i}$. One can show $Y_{t,S}$ is a $(32C_1^2\sum_{i \in S}\frac{1}{T_{t-1,i}^2}, 4C_1\frac{1}{T^{\min}_{t-1,S}})$-sub-Exponential random variable, so applying the concentration bounds on $Y_{t,S}$ \cite{vershynin2018high} and one can obtain the above $e_t(S_t)$.
    Then we consider two cases based on the value of $\sum_{i \in S_t}{\frac{1}{T_{t-1,i}}}$. For both cases, we use the reverse amortization trick from~\cite{wang2017improving} but different from~\cref{sec:BCUCB_reg}, $e_t(S_t)$ ensures that we only need to consider regret contributions from the min-arm (which is least played in $S_t$) according to certain batch-size independent thresholds. This in turn gives batch-size independent regret bounds that totally removes $O(\log K)$ in the leading term. See Appendix \OnlyInFull{\ref{apdx_sec:independent}}\OnlyInShort{D} for more details.
\end{proof}

\textbf{Computational Efficiency.}
Notice that like other ESCB-type algorithms \cite{combes2015combinatorial}, for the general reward function $r(S;\bmu)$, { there may not exist efficient $\bar{O}$, so one needs to enumerate over all possible actions $S \in \cS$ each round, where the time complexity could be
as high as $O(|\cS|T)$. 
However, when $r(S;\bmu)$ is a monotone submodular function (e.g, the reward function of the PMC problem~\cite{chen2016combinatorial}), we can modify $\rho_t(S)$ so that the optimistic reward $\bar{r}_t(S)$ is also monotone submodular, which can be efficiently optimized with a greedy $(1-1/e, 1)$-approximation oracle. Observe that the current $\rho_t(S)$ is not submodular since the maximum of two submodular functions are not necessarily submodular, but we know the summation of two submodular functions are submodular.
Based on this observation, we change $\rho_t(S)$ to $\rho'_t(S)=B_v\sqrt{\sum_{i \in S} \frac{C_1}{T_{t-1,i}}+ 8C_1\sqrt{\sum_{i \in S}\frac{\log(2|\mathcal{S}|T)}{T_{t-1,i}^2}}+ \frac{8C_1\log(2|\mathcal{S}|T)}{T^{\min}_{t-1, S}} }$, where $\max$ is replaced with a sum ($+$), and we prove in Appendix \OnlyInFull{\ref{apdx_sec:eff_escb}}\OnlyInShort{D.3} that $\rho'_t(S)$ is a monotone submodular function. Now we can use the greedy oracle to maximize a new optimistic reward ${r}'_t(S)=r(S;\hat{\boldsymbol{\mu}}_{t-1})+\rho'_t(S)$ in our SESCB algorithm. As for the final regret, using $\rho'_t(S)$ instead of $\rho_t(S)$ only worsens the final regret by a constant factor of two. 

Compared with \cite{merlis2019batch} for the PMC problem, our SESCB achieves the same $(1-1/e, 1)$-approximate regret bound but completely removes the $O(\log K)$ dependency. Moreover, our greedy oracle is efficient with computational complexity $O(TK^2LV)$, where $T$ is the total number of rounds, $K$ is the number of source nodes to be selected in each round, and $L, V$ are the total number of source and target nodes, which is much faster than the enumeration method with exponential time complexity $O(T\binom{L}{K}KV)$. For the regret analysis when using $r'_t(S)$, see Appendix \OnlyInFull{\ref{apdx_sec:eff_escb}}\OnlyInShort{D.3} for more details.}

\section{Applications}\label{sec:app}
In this section, we show how various applications satisfy our new TPVM, {\TPVMm} or VM smoothness condition and their corresponding $(B_v, B_1, \lambda)$ coefficients with non-trivial $B_v$, i.e., $B_v=o(B_1 \sqrt{K})$, which in turn improves the regret bounds over the batch-size dependence of $K$.

\begin{theorem}\label{lem:app_tab}
The combinatorial cascading bandits~\cite{kveton2015combinatorial}, the multi-layered network exploration~\cite{liu2021multi}, the influence maximization problems~\cite{wang2017improving} and the probabilistic maximum coverage problem~\cite{merlis2019batch} satisfy the TPVM (TPVM$_<$ or VM) conditions with coefficients $(B_v, B_1, \lambda)$, resulting regret bounds and improvements shown in \cref{tab:app}.
\end{theorem}

Note that the first four applications in \cref{tab:app} applies Theorem~\ref{thm:reg_lambda1}, while the last application applies Theorem~\ref{thm:independent}.
More specifically, the first two applications we consider are disjunctive and conjunctive cascading bandits~\cite{kveton2015combinatorial}, where $m$ base arms represent web pages and routing edges in online advertising and network routing, respectively. Batch-size $K$ is the maximum size of the ordered sequence $S \in \cS$ to be selected in each round, which will trigger web pages/routing edge one by one until certain stopping condition is satisfied, i.e., a click or a routing edge being broken. The reward is $1$ if \textit{any} web page is clicked (or if \textit{all} routing edges are live) and $0$ otherwise. Compared with~\cite{wang2017improving}, we achieve an improvement $O(K/(\log K\log \frac{K}{\Delta_{\min}}))$ for the conjunctive case and an improvement $O(K/\log K)$ for the disjunctive case, due to the same $B_v,B_1=1$ but different orders $\lambda$. 

The third application is the mutli-layered network exploration (MuLaNE) problem~\cite{liu2021multi}, 
and the MuLaNE task is to allocate $B$ budgets into $n$ layers to explore target nodes $V$. In MuLaNE, the base arms form a set $\cA=\{(i,u,b): i \in [n], u \in [V], b \in [B]\}$, the batch-size $K=(n+1)|V|$ and the reward is defined as the total reward give by the first visit of any target nodes. MuLaNE fits into our study, and compared with~\cite{liu2021multi}, the regret bound is improved by a factor of $O(n/\log K)$.

Our fourth application is the online influence maximization (OIM) problems direct acyclic graphs (DAG). 
For this application, the goal is to select at most $k$ seed nodes to influence as many target nodes $V$ as possible, where the influence process follows the independent cascade (IC) model~\cite{wang2017improving} (see Appendix \OnlyInFull{\ref{apdx_sec:proof_app}}\OnlyInShort{E} for more details). The base arms are the edges with unknown edge probabilities and the batch-size $K$ is the total number of edges that could be triggered by any set of $k$ seed nodes, denoted as $|E|$. 
The improvements here are significant, improving the existing results~\cite{merlis2019batch} by a factor of  $O(|E|/(L\log |E|\log \frac{|E|}{\Delta_{\min}}))$. 


For the PMC problem~\cite{merlis2019batch}, we consider a complete bipartite graph with $L$ source nodes on the left and $|V|$ target nodes $V$ on the right. The goal is to select $k$ seed nodes from $L$ nodes trying to influence as many as target nodes, so the edges $E$ are independent base arms and the batch-size is $k|V|$. 
By using the computational efficient version of \cref{alg:SECUCB} and applying Theorem~\ref{thm:independent}, we achieve $O(\log^2 k)$ improvement compared with~\cite{merlis2019batch} while maintaining good computational efficiency.

        \begin{table*}[t]
	\caption{Summary of the coefficients, regret  bounds and improvements for various applications.}	\label{tab:app}	
		\centering
	\resizebox{0.98\columnwidth}{!}{
	\begin{threeparttable}
	\begin{tabular}{|ccccc|}
		\hline
		\textbf{Application}&\textbf{Condition}& \textbf{$(B_v, B_1, \lambda)$} & \textbf{Regret}& \textbf{Improvement}\\
	\hline

		Disjunctive Cascading Bandits~\citep{kveton2015combinatorial} & {\TPVMm} & $(1,1,2)$ & $O(\sum_{i \in [m]}\frac{\log K\log T}{\Delta_i^{\min}})$ & $O(K/\log K)$ \\
			\hline
		Conjunctive Cascading Bandits~\citep{kveton2015combinatorial} & TPVM & $(1,1,1)$ & $O( \sum_{i \in [m]}\log \frac{K}{\Delta_{\min}}\frac{ \log K\log T}{\Delta_{i,\min}})$ & $O(K/(\log K\log \frac{K}{\Delta_{\min}}))$\\
			\hline
		Multi-layered Network Exploration~\citep{liu2021multi} & TPVM & $(\sqrt{1.25|V|},1,2)$ $^\dagger$ & $O(\sum_{i \in \cA}\frac{|V|\log (n|V|)\log T}{\Delta_i^{\min}})$ & $O(n/\log (n|V|))$\\
			\hline
	    Influence Maximization on DAG \cite{wang2017improving} & {\TPVMm} &$(\sqrt{L}|V|,|V|,1)$ $^{\dagger}$  & $O( \sum_{i \in [m]}\log \frac{|E|}{\Delta_{\min}}\frac{L|V|^2\log |E| \log T}{\Delta_i^{\min}})$ & $O(|E|/(L\log |E|\log \frac{|E|}{\Delta_{\min}}))$\\
	    \hline
	    Probabilistic Maximum Coverage \cite{merlis2019batch}$^*$ & VM & $(3\sqrt{2|V|},1,-)$ & $O(\sum_{i \in [m]}\frac{|V|\log T}{\Delta_i^{\min}})$ & $O(\log^2 k)$.\\
	    \hline
	\end{tabular}
	  \begin{tablenotes}[para, online,flushleft]
	\footnotesize
	\item[]\hspace*{-\fontdimen2\font}$^*$ This row is for the application in \cref{sec:independent} and the rest of rows are for \cref{sec:TPVM}; 
	\item[]\hspace*{-\fontdimen2\font}$^{\dagger}$ $|V|, |E|, n, k, L$ denotes the number of target nodes, the number of edges that can be triggered by the set of seed nodes, the number of layers, the number of seed nodes and the length of the longest directed path, respectively.
	\end{tablenotes}
			\end{threeparttable}
	}
\end{table*}

\begin{proof}[Proof Ideas.] 
For all above applications (except for the OIM on DAG), our proof involves the use of telescoping series to decompose the reward difference, together with a smart use of the Cauchy–Schwarz inequality aided by the variance terms. For disjunctive cascading bandits, for example, the reward difference $\abs{r(S;\bar{\bmu})-r(S;\bmu)}=\prod_{i=1}^K(1-\mu_i)-\prod_{i=1}^K(1-\bar{\mu}_i)$ can be telescoped as $\sum_{i \in [K]} (\bar{\mu}_i-\mu_i)\left(\prod_{j=1}^{i-1}(1-\mu_j) \cdot \prod_{j=i+1}^{K}(1-\bar{\mu}_j)\right)$. After this decomposition, we replace certain terms with $p_i^{D,S}$ and bound above by $\sum_{i \in [K]} \zeta_i p_i^{D,S} \sqrt{\prod_{j=i+1}^{K}(1-\mu_j) } + \sum_{i \in [K]} \eta_ip_i^{D,S}$. Then we simultaneously multiply and divide the variance term $\sqrt{(1-\mu_i)\mu_i}$ on the first term and apply the Cauchy–Schwarz inequality to move the summation over $K$ into the square root, concluding the satisfaction of Condition~\ref{cond:TPVMm} with $B_v=\sqrt{\sum_{i\in [K]}(1-\mu_i)\mu_i\prod_{j=i+1}^{K}(1-\mu_j) }\le 1$. As for the OIM on DAG, since reward function have no closed-form solutions~\cite{chen2009efficient}, the analysis is more involved with the need of advanced techniques such as the coupling technique~\cite{li2020online}, see Appendix \OnlyInFull{\ref{apdx_sec:proof_app}}\OnlyInShort{E} for details.
\end{proof}

\vspace{-7pt}
\section{Conclusion and Future Direction}\label{sec:conclusion}
This paper studies the CMAB problem with probabilistically triggered arms or independent arms. We discover new TPVM and VM conditions, and propose BCUCB-T and SESCB algorithms to reduce and remove the batch-size $K$ in the regret bounds, respectively.
We also show that several important applications all satisfy our conditions to achieve improved regrets, both theoretically and empirically.
There are many compelling directions for future study.
For example, it would be interesting to study the setting of CMAB-T together with independent arms. One could also explore how to extend our application and consider general graphs in online influence maximization bandits. 
\begin{ack}
The work of John C.S. Lui was supported in part by the HK RGC SRF2122-4202. The work of Siwei Wang was supported in part by the National Natural Science Foundation of China Grant 62106122.
\end{ack}

\newpage
\bibliographystyle{plainnat}
\bibliography{main}

\begin{thebibliography}{33}
\providecommand{\natexlab}[1]{#1}
\providecommand{\url}[1]{\texttt{#1}}
\expandafter\ifx\csname urlstyle\endcsname\relax
  \providecommand{\doi}[1]{doi: #1}\else
  \providecommand{\doi}{doi: \begingroup \urlstyle{rm}\Url}\fi

\bibitem[Atamt{\"u}rk and G{\'o}mez(2017)]{atamturk2017maximizing}
Alper Atamt{\"u}rk and Andr{\'e}s G{\'o}mez.
\newblock Maximizing a class of utility functions over the vertices of a
  polytope.
\newblock \emph{Operations Research}, 65\penalty0 (2):\penalty0 433--445, 2017.

\bibitem[Audibert et~al.(2009)Audibert, Munos, and
  Szepesv{\'a}ri]{audibert2009exploration}
Jean-Yves Audibert, R{\'e}mi Munos, and Csaba Szepesv{\'a}ri.
\newblock Exploration--exploitation tradeoff using variance estimates in
  multi-armed bandits.
\newblock \emph{Theoretical Computer Science}, 410\penalty0 (19):\penalty0
  1876--1902, 2009.

\bibitem[Auer et~al.(2002)Auer, Cesa-Bianchi, and Fischer]{auer2002finite}
Peter Auer, Nicolo Cesa-Bianchi, and Paul Fischer.
\newblock Finite-time analysis of the multiarmed bandit problem.
\newblock \emph{Machine learning}, 47\penalty0 (2-3):\penalty0 235--256, 2002.

\bibitem[Bubeck et~al.(2012)Bubeck, Cesa-Bianchi, et~al.]{bubeck2012regret}
S{\'e}bastien Bubeck, Nicolo Cesa-Bianchi, et~al.
\newblock Regret analysis of stochastic and nonstochastic multi-armed bandit
  problems.
\newblock \emph{Foundations and Trends{\textregistered} in Machine Learning},
  5\penalty0 (1):\penalty0 1--122, 2012.

\bibitem[Chen et~al.(2009)Chen, Wang, and Yang]{chen2009efficient}
Wei Chen, Yajun Wang, and Siyu Yang.
\newblock Efficient influence maximization in social networks.
\newblock In \emph{Proceedings of the 15th ACM SIGKDD international conference
  on Knowledge discovery and data mining}, pages 199--208, 2009.

\bibitem[Chen et~al.(2013{\natexlab{a}})Chen, Lakshmanan, and
  Castillo]{chen2013information}
Wei Chen, Laks~VS Lakshmanan, and Carlos Castillo.
\newblock Information and influence propagation in social networks.
\newblock \emph{Synthesis Lectures on Data Management}, 5\penalty0
  (4):\penalty0 1--177, 2013{\natexlab{a}}.

\bibitem[Chen et~al.(2013{\natexlab{b}})Chen, Wang, and
  Yuan]{chen2013combinatorial}
Wei Chen, Yajun Wang, and Yang Yuan.
\newblock Combinatorial multi-armed bandit: General framework and applications.
\newblock In \emph{International Conference on Machine Learning}, pages
  151--159. PMLR, 2013{\natexlab{b}}.

\bibitem[Chen et~al.(2016)Chen, Wang, Yuan, and Wang]{chen2016combinatorial}
Wei Chen, Yajun Wang, Yang Yuan, and Qinshi Wang.
\newblock Combinatorial multi-armed bandit and its extension to
  probabilistically triggered arms.
\newblock \emph{The Journal of Machine Learning Research}, 17\penalty0
  (1):\penalty0 1746--1778, 2016.

\bibitem[Combes et~al.(2015)Combes, Talebi Mazraeh~Shahi, Proutiere,
  et~al.]{combes2015combinatorial}
Richard Combes, Mohammad~Sadegh Talebi Mazraeh~Shahi, Alexandre Proutiere,
  et~al.
\newblock Combinatorial bandits revisited.
\newblock \emph{Advances in neural information processing systems}, 28, 2015.

\bibitem[Cuvelier et~al.(2021)Cuvelier, Combes, and
  Gourdin]{cuvelier2021statistically}
Thibaut Cuvelier, Richard Combes, and Eric Gourdin.
\newblock Statistically efficient, polynomial-time algorithms for combinatorial
  semi-bandits.
\newblock \emph{Proceedings of the ACM on Measurement and Analysis of Computing
  Systems}, 5\penalty0 (1):\penalty0 1--31, 2021.

\bibitem[Degenne and Perchet(2016)]{degenne2016combinatorial}
R{\'e}my Degenne and Vianney Perchet.
\newblock Combinatorial semi-bandit with known covariance.
\newblock In \emph{Advances in Neural Information Processing Systems}, pages
  2972--2980, 2016.

\bibitem[Dubhashi and Panconesi(2009)]{dubhashi2009concentration}
Devdatt~P Dubhashi and Alessandro Panconesi.
\newblock \emph{Concentration of measure for the analysis of randomized
  algorithms}.
\newblock Cambridge University Press, 2009.

\bibitem[Gai et~al.(2012)Gai, Krishnamachari, and Jain]{gai2012combinatorial}
Yi~Gai, Bhaskar Krishnamachari, and Rahul Jain.
\newblock Combinatorial network optimization with unknown variables:
  Multi-armed bandits with linear rewards and individual observations.
\newblock \emph{IEEE/ACM Transactions on Networking (TON)}, 20\penalty0
  (5):\penalty0 1466--1478, 2012.

\bibitem[Honorio and Jaakkola(2014)]{honorio2014tight}
Jean Honorio and Tommi Jaakkola.
\newblock Tight bounds for the expected risk of linear classifiers and
  pac-bayes finite-sample guarantees.
\newblock In \emph{Artificial Intelligence and Statistics}, pages 384--392.
  PMLR, 2014.

\bibitem[Kempe et~al.(2003)Kempe, Kleinberg, and Tardos]{kempe2003maximizing}
David Kempe, Jon Kleinberg, and {\'E}va Tardos.
\newblock Maximizing the spread of influence through a social network.
\newblock In \emph{Proceedings of the ninth ACM SIGKDD international conference
  on Knowledge discovery and data mining}, pages 137--146, 2003.

\bibitem[Kveton et~al.(2015{\natexlab{a}})Kveton, Szepesvari, Wen, and
  Ashkan]{kveton2015cascading}
Branislav Kveton, Csaba Szepesvari, Zheng Wen, and Azin Ashkan.
\newblock Cascading bandits: Learning to rank in the cascade model.
\newblock In \emph{International Conference on Machine Learning}, pages
  767--776. PMLR, 2015{\natexlab{a}}.

\bibitem[Kveton et~al.(2015{\natexlab{b}})Kveton, Wen, Ashkan, and
  Szepesv{\'a}ri]{kveton2015combinatorial}
Branislav Kveton, Zheng Wen, Azin Ashkan, and Csaba Szepesv{\'a}ri.
\newblock Combinatorial cascading bandits.
\newblock In \emph{Proceedings of the 28th International Conference on Neural
  Information Processing Systems-Volume 1}, pages 1450--1458,
  2015{\natexlab{b}}.

\bibitem[Kveton et~al.(2015{\natexlab{c}})Kveton, Wen, Ashkan, and
  Szepesvari]{kveton2015tight}
Branislav Kveton, Zheng Wen, Azin Ashkan, and Csaba Szepesvari.
\newblock Tight regret bounds for stochastic combinatorial semi-bandits.
\newblock In \emph{AISTATS}, 2015{\natexlab{c}}.

\bibitem[Li et~al.(2016)Li, Wang, Zhang, and Chen]{li2016contextual}
Shuai Li, Baoxiang Wang, Shengyu Zhang, and Wei Chen.
\newblock Contextual combinatorial cascading bandits.
\newblock In \emph{International conference on machine learning}, pages
  1245--1253. PMLR, 2016.

\bibitem[Li et~al.(2020)Li, Kong, Tang, Li, and Chen]{li2020online}
Shuai Li, Fang Kong, Kejie Tang, Qizhi Li, and Wei Chen.
\newblock Online influence maximization under linear threshold model.
\newblock \emph{Advances in Neural Information Processing Systems},
  33:\penalty0 1192--1204, 2020.

\bibitem[Liu et~al.(2021)Liu, Zuo, Chen, Chen, and Lui]{liu2021multi}
Xutong Liu, Jinhang Zuo, Xiaowei Chen, Wei Chen, and John~CS Lui.
\newblock Multi-layered network exploration via random walks: From offline
  optimization to online learning.
\newblock In \emph{International Conference on Machine Learning}, pages
  7057--7066. PMLR, 2021.

\bibitem[Marchal and Arbel(2017)]{marchal2017sub}
Olivier Marchal and Julyan Arbel.
\newblock On the sub-gaussianity of the beta and dirichlet distributions.
\newblock \emph{Electronic Communications in Probability}, 22:\penalty0 1--14,
  2017.

\bibitem[Merlis and Mannor(2019)]{merlis2019batch}
Nadav Merlis and Shie Mannor.
\newblock Batch-size independent regret bounds for the combinatorial
  multi-armed bandit problem.
\newblock In \emph{Conference on Learning Theory}, pages 2465--2489. PMLR,
  2019.

\bibitem[Merlis and Mannor(2020)]{merlis2020tight}
Nadav Merlis and Shie Mannor.
\newblock Tight lower bounds for combinatorial multi-armed bandits.
\newblock In \emph{Conference on Learning Theory}, pages 2830--2857. PMLR,
  2020.

\bibitem[Perrault et~al.(2020)Perrault, Healey, Wen, and
  Valko]{perrault2020budgeted}
Pierre Perrault, Jennifer Healey, Zheng Wen, and Michal Valko.
\newblock Budgeted online influence maximization.
\newblock In \emph{International Conference on Machine Learning}, pages
  7620--7631. PMLR, 2020.

\bibitem[Robbins(1952)]{robbins1952some}
Herbert Robbins.
\newblock Some aspects of the sequential design of experiments.
\newblock \emph{Bulletin of the American Mathematical Society}, 58\penalty0
  (5):\penalty0 527--535, 1952.

\bibitem[Vershynin(2018)]{vershynin2018high}
Roman Vershynin.
\newblock \emph{High-dimensional probability: An introduction with applications
  in data science}, volume~47.
\newblock Cambridge university press, 2018.

\bibitem[Vial et~al.(2022)Vial, Sanghavi, Shakkottai, and
  Srikant]{vial2022minimax}
Daniel Vial, Sujay Sanghavi, Sanjay Shakkottai, and R~Srikant.
\newblock Minimax regret for cascading bandits.
\newblock \emph{arXiv preprint arXiv:2203.12577}, 2022.

\bibitem[Wang and Chen(2017)]{wang2017improving}
Qinshi Wang and Wei Chen.
\newblock Improving regret bounds for combinatorial semi-bandits with
  probabilistically triggered arms and its applications.
\newblock In \emph{Advances in Neural Information Processing Systems}, pages
  1161--1171, 2017.

\bibitem[Wang and Chen(2018)]{wang2018thompson}
Siwei Wang and Wei Chen.
\newblock Thompson sampling for combinatorial semi-bandits.
\newblock In \emph{International Conference on Machine Learning}, pages
  5114--5122, 2018.

\bibitem[Wen et~al.(2017)Wen, Kveton, Valko, and Vaswani]{wen2017online}
Zheng Wen, Branislav Kveton, Michal Valko, and Sharan Vaswani.
\newblock Online influence maximization under independent cascade model with
  semi-bandit feedback.
\newblock \emph{Advances in neural information processing systems}, 30, 2017.

\bibitem[Zhang et~al.(2021)Zhang, Yang, Ji, and Du]{zhang2021improved}
Zihan Zhang, Jiaqi Yang, Xiangyang Ji, and Simon~S Du.
\newblock Improved variance-aware confidence sets for linear bandits and linear
  mixture mdp.
\newblock \emph{Advances in Neural Information Processing Systems},
  34:\penalty0 4342--4355, 2021.

\bibitem[Zhou et~al.(2021)Zhou, Gu, and Szepesvari]{zhou2021nearly}
Dongruo Zhou, Quanquan Gu, and Csaba Szepesvari.
\newblock Nearly minimax optimal reinforcement learning for linear mixture
  markov decision processes.
\newblock In \emph{Conference on Learning Theory}, pages 4532--4576. PMLR,
  2021.

\end{thebibliography}

\clearpage
\section*{Checklist}


\begin{enumerate}

\item For all authors...
\begin{enumerate}
  \item Do the main claims made in the abstract and introduction accurately reflect the paper's contributions and scope?
    \answerYes{}
  \item Did you describe the limitations of your work?
    \answerYes{See \cref{sec:TPVM} and \cref{sec:conclusion}.}
  \item Did you discuss any potential negative societal impacts of your work?
    \answerNA{Since this work is mainly about online learning models, algorithms and analytical techniques, there is no foreseeable societal impact.}
  \item Have you read the ethics review guidelines and ensured that your paper conforms to them?
    \answerYes{}
\end{enumerate}

\item If you are including theoretical results...
\begin{enumerate}
  \item Did you state the full set of assumptions of all theoretical results?
    \answerYes{See \cref{sec:TPVM} and \cref{sec:independent}.}
        \item Did you include complete proofs of all theoretical results?
    \answerYes{See the Appendix.}
\end{enumerate}

\item If you ran experiments...
\begin{enumerate}
  \item Did you include the code, data, and instructions needed to reproduce the main experimental results (either in the supplemental material or as a URL)?
    \answerYes{}
  \item Did you specify all the training details (e.g., data splits, hyperparameters, how they were chosen)?
    \answerYes{}
        \item Did you report error bars (e.g., with respect to the random seed after running experiments multiple times)?
    \answerYes{}
        \item Did you include the total amount of compute and the type of resources used (e.g., type of GPUs, internal cluster, or cloud provider)?
    \answerYes{}
\end{enumerate}

\item If you are using existing assets (e.g., code, data, models) or curating/releasing new assets...
\begin{enumerate}
  \item If your work uses existing assets, did you cite the creators?
    \answerYes{}
  \item Did you mention the license of the assets?
    \answerNA{We only use open-sourced datasets and codes.}
  \item Did you include any new assets either in the supplemental material or as a URL?
    \answerNo{}
  \item Did you discuss whether and how consent was obtained from people whose data you're using/curating?
    \answerNA{We only use open-sourced datasets and codes.}
  \item Did you discuss whether the data you are using/curating contains personally identifiable information or offensive content?
    \answerNA{}
\end{enumerate}

\item If you used crowdsourcing or conducted research with human subjects...
\begin{enumerate}
  \item Did you include the full text of instructions given to participants and screenshots, if applicable?
    \answerNA{}
  \item Did you describe any potential participant risks, with links to Institutional Review Board (IRB) approvals, if applicable?
    \answerNA{}
  \item Did you include the estimated hourly wage paid to participants and the total amount spent on participant compensation?
    \answerNA{}
\end{enumerate}

\end{enumerate}

\appendix
\clearpage
\onecolumn
\section*{Appendix}

The Appendix is organized as follows. 
We first compare the {\TPVMm} Condition (Condition~\ref{cond:TPVMm}) with TPM Condition (Condition~\ref{cond:TPM}) in~\cref{app:TPVMvsTPM}. 
The comparisons between Gini-smoothness Condition with the TPVM (Condition~\ref{cond:TPVMm}) and VM (Condition~\ref{cond:VM}) Condition are introduced in \cref{app:gini}.
The detailed proofs for \cref{thm:reg_lambda1} together with some discussions are in~\cref{apdx_sec:main_regret_analysis}.
The detailed proofs for \cref{thm:independent} are in \cref{apdx_sec:independent}.
The application details and the proof details related to \cref{lem:app_tab} are in~\cref{apdx_sec:proof_app}.
The experiments for different applications are included in~\cref{apdx_sec:exp}.

\section{Comparing the {\TPVMm} Condition (Condition~\ref{cond:TPVMm}) with the TPM Condition (Condition~\ref{cond:TPM})}
 \label{app:TPVMvsTPM}
 
As discussed in Section~\ref{sec:TPVM}, the {\TPVMm} condition (Condition~\ref{cond:TPVMm}) does not imply the TPM condition (Condition~\ref{cond:TPM}) in general.
In this section, we show that under some additional conditions, Condition~\ref{cond:TPVMm} does imply Condition~\ref{cond:TPM}.
\begin{lemma}
Assume that the followings are true:
(a) Condition~\ref{cond:mono} holds; 
(b) there exist $D_{\vee}, D_{\wedge} \in \cD$ with mean vectors $\bmu_{\vee} = \bmu \vee \bmu'$ and $\bmu_{\wedge} = \bmu \wedge \bmu'$ respectively, where operation $\vee$ and $\wedge$ means taking 
coordinate wise $\max$ and $\min$; and
(c) $p_i^{D,S}$ increases as the mean vector of $D$ increases.
Then Condition~\ref{cond:TPVMm} implies Condition~\ref{cond:TPM} with the same $B_1$ coefficient.
\end{lemma}
\begin{proof}
First, when setting $\boldzeta = \bzero$ in Condition~\ref{cond:TPVMm}, we obtain the directional TPM condition.
Then we prove with the following derivation that with the three assumptions stated in the lemma, directional TPM condition implies the undirectional TPM condition (Condition~\ref{cond:TPM}).
For any $D,D' \in \cD$ with mean vectors $\bmu, \bmu'$, without loss of generality, we assume that $r(S;\bmu') \ge r(S;\bmu)$.
we have
\begin{align*}
& |r(S;\bmu')-r(S;\bmu)|  \\
& = r(S;\bmu')-r(S;\bmu) \\
& \le r(S;\bmu_{\vee})-r(S;\bmu_{\wedge}) 						& \mbox{by Assumptions (a) and (b)} \\
& \le B_1 \sum_{i\in[m]}p_i^{D_{\wedge},S}|\mu_i - \mu_i'| 		& \mbox{by Assumption (b) and the directional TPM condition}\\
& \le B_1 \sum_{i\in[m]}p_i^{D,S}|\mu_i - \mu_i'|.				& \mbox{By Assumption (c)}
\end{align*}
Therefore, the undirectional TPM condition (Condition~\ref{cond:TPM}) holds.
\end{proof}

It is not difficult to verify that for the online influence maximization application discussed in Section~\ref{sec:app}, all three assumptions in the lemma holds.

\section{Comparing the Gini-smoothness Condition \cite{merlis2019batch} with the {\TPVMm} Condition (Condition~\ref{cond:TPVMm}) and VM Condition (Condition~\ref{cond:VM})}
\label{app:gini}

\citet{merlis2019batch} define the following Gini-smoothness condition. In this section, we provide comparisons between this condition and our {\TPVMm} condition (Condition~\ref{cond:TPVMm}) and VM condition (Condition~\ref{cond:VM}).
\begin{condition}[Gini-smoothness Condition, Restated, \cite{merlis2019batch}]\label{cond:gini-smooth}
	Let $f(S;\bx):\cS \times [0,1]^m \rightarrow \R$ be a differentiable function in $x\in (0,1)^m$ and continuous in $x\in[0,1]^m$, for any $S \in \cS$. The function $f(S;x)$ is said to be monotonic Gini-smooth, with smoothness parameters ($\gamma_g$, $\gamma_\infty$) if:
	
	1. For any $S \in \cS$, the function $f(S;\bx)$ is monotonically increasing with bounded gradient, i.e., for any $i \in S$ and $x \in (0,1)^m$, $0 \le \frac{\partial f(S;\bx)}{\partial x_i } \le \gamma_\infty $.
	If $i \notin S$, then $\frac{\partial f(S;x)}{\partial x_i }=0$ for all $x\in(0,1)^L$.
	
	2. For any $S \in \cS$ and $x \in (0,1)^m$, it holds that
	\begin{equation}
		\sqrt{\sum_{i \in S}x_i(1-x_i)(\frac{\partial f(S;x)}{\partial x_i})} \le \gamma_g.
	\end{equation}
\end{condition}

\subsection{Triggering Probability Modulated Gini-smoothness Condition Does Not Imply {\TPVMm} Condition}\label{apdx_sec:direct_tpvm}

The original Gini-smoothness condition (Condition~\ref{cond:gini-smooth}) does not work directly with probabilistically triggered arms.
Thus, our first attempt is to add triggering probability modulation to the Gini-smoothness condition as given below, in hope that it would extend the result in \cite{merlis2019batch} to the CMAB-T framework.


\begin{condition}[TPM Gini-smoothness]\label{def:gini-smoothness-TPM}
    For a CMAB-T problem instance $([m], \cS, \cD, D_{\text{trig}}, R)$,
	assume the reward function $r(S;\bmu):\cS \times [0,1]^m \rightarrow \R$ is a differentiable function in $\bmu \in (0,1)^m$ and continuous in $\bmu\in[0,1]^m$, for any $S \in \cS$. 
	The reward function $r(S;\bmu)$ is said to be monotonic Triggering Probability Modulated (TPM) Gini-smooth, with smoothness parameters $(\gamma_g$, $\gamma_\infty, \lambda\ge 1)$ if:
	
	1. For any distribution $D\in \cD$ with mean vector $\bmu \in (0,1)^{m}$ and any action $S \in \cS$, the function $r(S;\bmu)$ is monotonically increasing with bounded gradient: For any $i \in [m]$, if $p_{i}^{D,S} > 0$, then $0 \le \frac{\partial r(S;\bmu)}{\partial \mu_i } \frac{1}{p_{i}^{D,S}}\le \gamma_\infty $;
	If $p_{i}^{D,S}=0$, then $\frac{\partial r(S;\bmu)}{\partial \mu_i }=0$ for all $\bmu\in(0,1)^m$.
	
	2. For any distribution $D \in \cD$ with mean vector $\bmu\in (0,1)^m$ and any action $S \in \cS$, it holds that
	\begin{equation}
		\sqrt{\sum_{i \in \tilde{S}}\mu_i(1-\mu_i)(\frac{\partial r(S;\bmu)}{\partial \mu_i})^2 \frac{1}{(p_{i}^{D,S})^\lambda}} \le \gamma_g.
	\end{equation}
\end{condition}

However, using the above TPM Gini-smoothness condition, we cannot derive a desirable regret bound.
In particular, the above condition only guarantees the following lemma (following the analysis of Lemma 6 in \cite{merlis2019batch}), 
	which is weaker than our {\TPVMm} condition (Condition \ref{cond:TPVMm}), leading to a weaker regret with an additional
	factor $\max_{S\in \cS, i \in [m]}p_i^{\max,S}/p_i^{D,S}$.
Such a factor could be exponentially large and undesirable in applications, similar to the factor being avoided in \cite{chen2016combinatorial} by introducing the TPM condition.

%


\begin{restatable}{lemma}{apdx_lemma6_trigger}\label{apdx_lem:lemma6_trigger}
	Let $r(S;\bx)$ be a monotonic $(\gamma_g, \gamma_{\infty},  \lambda)$ TPM gini-smooth function. 
	For any $\bmu,\bmu', \bmu'' \in [0,1]^m$, with $\boldzeta=\bmu'-\bmu, \boldeta=\bmu''-\bmu'$, let $\cD_{\bmu, \bmu', \bmu''}=\{D \in \cD \text{ with mean vector }\bx: \forall i\in[m], \min\{\mu_i, \mu'_i, \mu''_i\}\le x_i \le \max\{\mu_i, \mu'_i,\mu''_i\}\}$ and $p_i^{\max, S}=\max_{D \in \cD_{\bmu, \bmu', \bmu''}}p_i^{D,S}$, it holds that 
	\begin{equation}
		|r(S; \bmu'')- r(S;\bmu)| \le 3\sqrt{2}\gamma_g\sqrt{\sum_{i \in S}\left(\frac{|\zeta_i|}{\sqrt{(1-\mu_i)\mu_i}}\right)^2 p^{\max, S}_i} + \sum_{i \in \tilde{S}}p_i^{\max,S}\abs{\eta_i},
	\end{equation}
\end{restatable}
\begin{proof}
	
	\textbf{For the $\abs{r(S; \bmu')- r(S;\bmu)}$ term,}
	
	First, we define two functions $g,h$, where
	\begin{equation}
		g(z)=\int_{0}^z\frac{dy}{\sqrt{y(1-y)}}, h(z)= \int_{0}^z \frac{dy}{\sqrt{y} \wedge \sqrt{1-y}}.
	\end{equation}
	For $g(z)$, $g'(z)>0$ so the inverse function $g^{-1}$ is well defined.
	We also know that $h(z)$ has the following closed form,
	\begin{equation}
		h(z)=\begin{cases}
			2\sqrt{z}, &\text{if $z \le 1/2$}\\
			2\sqrt{2}-2\sqrt{1-z}, &\text{if $z \ge 1/2$}.\\
		\end{cases}
	\end{equation}
	Note that these two functions are closely related: $h'(z) \le g'(z) \le \sqrt{2} h'(z)$ with $g(0)=h(0)$.
	Therefore, $h(z) \le g(z) \le \sqrt{2}h(z)$ and $g(z_2)-g(z_1) \le \sqrt{2}(h(z_2)-h(z_1))$ for any $z_1 \le z_2 \in [0,1]$.
	Now we set up a parameterization $z(t)$ for $t \in [0,1]$ such that $z_i(0)=\mu_i, z_i(1)=\mu'_i$. Specifically, we choose the parameterization to be 
	\begin{equation}
		z_i(t) = g^{-1}([g(\mu'_i)-g(\mu_i)]t + g(\mu_i)]),
	\end{equation}
	Then its gradient is 
	\begin{equation}
		z'_i(t)= \frac{g(\mu_i')-g(\mu_i)}{g'(z_i(t))}=(g(\mu_i')-g(\mu_i))\sqrt{z_i(t)(1-z_i(t))}.
	\end{equation}
	Then we can use the gradient theorem to bound $\abs{r(S;\bmu')-r(S;\bmu)}$ as 
	
	\begin{align}
		|r(S;\bmu')-r(S;\bmu)| & =\abs{ \int_{\bx=\bmu}^{\bmu'} \nabla r(S;\bx)\cdot d\bx }= \abs{\int_{t=0}^{1}\sum_{i \in \tilde{S}} \frac{\partial r(S;z(t))}{\partial x_i} z_i'(t) dt}\\
		&\le \int_{0}^1\sqrt{\sum_{i \in \tilde{S}} (g(\mu'_i)-g(\mu_i))^2(p_i^{z(t),S})^{\lambda}} \sqrt{\sum_{i \in \tilde{S}} (\frac{\partial r(S;z(t))}{\partial x_i})^2 \frac{z_i(t)(1-z_i(t))}{(p_i^{z(t),S})^{\lambda}}}dt\\
		&\le \int_{0}^1  \gamma_g\sqrt{\sum_{i \in \tilde{S}} (g(\mu'_i)-g(\mu_i))^2(p_i^{z(t),S})^{\lambda}} dt \label{eq:gamma_g_gradient_variable_p}\\
		&\le  \gamma_g\sqrt{\sum_{i \in \tilde{S}} (g(\mu'_i)-g(\mu_i))^2(p_i^{\max,S})^{\lambda}},
	\end{align}
	Following the similar derivation for \cref{apdx_eq:vm_1} in next subsection, we have
	\begin{equation}\label{apdx_eq:tpvm_1}
		|r(S; \bmu')- r(S;\bmu)| \le 3\sqrt{2}\gamma_{g}\sqrt{\sum_{i \in \tilde{S}}\left(\frac{|\mu'_i-\mu_i|}{\sqrt{(1-\mu_i)\mu_i}}\right)^2(p_i^{\max,S})^{\lambda}}.
	\end{equation}

	\textbf{For the $\abs{r(S; \bmu'')- r(S;\bmu')}$ term,}
	
	We can use the gradient theorem to bound, let $\bx(t)$ with $x_i(t)=t(\mu''_i-\mu'_i)+\mu'_i$.
	
	\begin{align}
		\abs{r(S;\bmu'')-r(S;\bmu')} &=\abs{\int_{t=0}^{1}\sum_{i \in \tilde{S}} \frac{\partial r(S;\bx(t))}{\partial x_i} x_i'(t) dt}\notag\\
		&\le \abs{\int_{t=0}^{1}\sum_{i \in \tilde{S}} \frac{\partial r(S;\bx(t))}{\partial x_ip_i^{\bx(t), S}} p_i^{\bx(t), S}(\mu''_i-\mu'_i) dt}\notag\\
		&\le \int_{t=0}^{1}\sum_{i \in \tilde{S}} \abs{\frac{\partial r(S;\bx(t))}{\partial x_i p_i^{\bx(t), S}}} \abs{p_i^{\bx(t), S}(\mu''_i-\mu'_i)} dt\notag\\
		&\le \int_{t=0}^{1}\sum_{i \in \tilde{S}}\gamma_{\infty} \abs{p_i^{\bx(t), S}(\mu''_i-\mu'_i)} dt\notag\\
		&\le \int_{t=0}^{1}\sum_{i \in \tilde{S}}\gamma_{\infty} \abs{p_i^{\max, S}(\mu''_i-\mu'_i)}dt\notag\\
		&= \gamma_{\infty}\sum_{i \in \tilde{S}}p_i^{\max, S}|\eta_i|.\label{apdx_eq:tpvm_2}
	\end{align}
	Combining \cref{apdx_eq:tpvm_1} and \cref{apdx_eq:tpvm_2}, we conclude the lemma.
\end{proof}

The above lemma indicates that directly extending the Gini-smoothness condition may not be strong enough for the probabilistic triggering setting.
This motivates us to define the new {\TPVMm} condition not based on the differential form, but directly on the difference form $\abs{r(S;\bmu')-r(S;\bmu)}$.
This can be viewed as incorporating triggering probability properly into the result of Lemma 6 in \cite{merlis2019batch}.

\subsection{Gini-smoothness Condition Implies VM Condition}\label{apdx_sec:derive_vm}

In this section, we show in the following lemma that the original Gini-smoothness condition (Condition~\ref{cond:gini-smooth}) implies the VM condition (Condition~\ref{cond:VM}),
	with $(B_v,B_1)=(3\sqrt{2}\gamma_g,\gamma_{\infty})$.
The proof of this lemma is similar to \citep[Lemma 6]{merlis2019batch}, 
	but we need to extend it to the undirectional case, where $\bmu''$ is not necessarily larger than $\bmu$ in all dimensions.

\begin{restatable}{lemma}{apdx_lemma6_no_trigger}\label{lem:lemma6_no_trigger}
Let $r(S;\bmu)$ be a monotonic $(\gamma_g,\gamma_{\infty})$ gini-smooth function as given in Condition \ref{cond:gini-smooth}. 
For any $\bmu,\bmu', \bmu'' \in [0,1]^m$, with $\boldzeta=\bmu'-\bmu, \boldeta=\bmu''-\bmu'$, it holds that
\begin{equation}
    \abs{r(S; \bmu'')- r(S;\bmu)} \le 3\sqrt{2}\gamma_g\sqrt{\sum_{i \in S}\left(\frac{|\zeta_i|}{\sqrt{(1-\mu_i)\mu_i}}\right)^2} + \gamma_{\infty}\sum_{i \in S}\abs{\eta_i}.
\end{equation}
\end{restatable}
\begin{proof}
We use $\abs{r(S; \bmu'')- r(S;\bmu)} \le \abs{r(S; \bmu')- r(S;\bmu)} + \abs{r(S; \bmu'')- r(S;\bmu')}$ and separately bound two terms in the LHS.

\textbf{For the $\abs{r(S; \bmu')- r(S;\bmu)}$ term,}

We define two functions $g,h$, where
\begin{equation}
    g(z)=\int_{0}^z\frac{dy}{\sqrt{y(1-y)}}, h(z)= \int_{0}^z \frac{dy}{\sqrt{y} \wedge \sqrt{1-y}}.
\end{equation}
For $g(z)$, $g'(z)>0$ so the inverse function $g^{-1}$ is well defined.
We also know that $h(z)$ has the following closed form,
\begin{equation}
    h(z)=\begin{cases}
2\sqrt{z}, &\text{if $z \le 1/2$}\\
2\sqrt{2}-2\sqrt{1-z}, &\text{if $z \ge 1/2$}.\\
\end{cases}
\end{equation}
Note that these two functions are closely related: $h'(z) \le g'(z) \le \sqrt{2} h'(z)$ with $g(0)=h(0)$.
Therefore, $h(z) \le g(z) \le \sqrt{2}h(z)$ and $g(z_2)-g(z_1) \le \sqrt{2}(h(z_2)-h(z_1))$ for any $z_1 \le z_2 \in [0,1]$.
Now we set up a parameterization $z(t)$ for $t \in [0,1]$ such that $z_i(0)=\mu_i, z_i(1)=\mu'_i$. Specifically, we choose the parameterization to be 
\begin{equation}
    z_i(t) = g^{-1}([g(\mu'_i)-g(\mu_i)]t + g(\mu_i)]),
\end{equation}
Then its gradient is 
\begin{equation}
    z'_i(t)= \frac{g(\mu_i')-g(\mu_i)}{g'(z_i(t))}=(g(\mu_i')-g(\mu_i))\sqrt{z_i(t)(1-z_i(t))}.
\end{equation}
Then we can use the gradient theorem to bound $r(S;\bmu')-r(S;\bmu)$ as 
\begin{align}
    |r(S;\bmu')-r(S;\bmu)| & =\abs{ \int_{\bx=\bmu}^{\bmu'} \nabla r(S;\bx)\cdot d\bx }= \abs{\int_{t=0}^{1}\sum_{i \in S} \frac{\partial r(S;z(t))}{\partial x_i} z_i'(t) dt}\\
    &\le \int_{t=0}^{1}\sum_{i \in S} \abs{g(\mu_i')-g(\mu_i)}\abs{\frac{\partial r(S;z(t))}{\partial x_i}}\sqrt{z_i(t)(1-z_i(t))} dt\\
    &\le \int_{0}^1\sqrt{\sum_{i \in S} (g(\mu'_i)-g(\mu_i))^2} \sqrt{\sum_{i \in S} (\frac{\partial r(S;z(t))}{\partial x_i})^2 z_i(t)(1-z_i(t))}dt\\
    &\le \int_{0}^1  \gamma_g\sqrt{\sum_{i \in S} (g(\mu'_i)-g(\mu_i))^2} dt \label{eq:gamma_g_gradient}
\end{align}
To calculate the bound, we use the relation between $g$ and $h$, and calculate the difference over $h$ for the following cases:

\textbf{Case 1:}
When $\mu_i \le  \mu_i' \le 1/2$, then $|g(\mu'_i)-g(\mu_i)| = g(\mu'_i)-g(\mu_i) \le \sqrt{2} (h(\mu'_i) - h(\mu_i))$.
\begin{align*}
    h(\mu'_i)-h(\mu_i)&= 2\sqrt{\mu'_i} - 2\sqrt{\mu_i} =  2\sqrt{\mu_i} (\sqrt{1+\frac{|\mu'_i-\mu_i|}{\mu_i}}-1) \\
    &\le 2\sqrt{\mu_i} \frac{|\mu'_i-\mu_i|}{2\mu_i}\le \frac{|\mu'_i-\mu_i|}{\sqrt{(1-\mu_i)\mu_i}},
\end{align*}
where the first inequality uses the fact that $\sqrt{1+x} \le 1 + x/2$, for any $x > -1$.
So $|g(\mu'_i)-g(\mu_i)| \le \sqrt{2} \frac{|\mu'_i-\mu_i|}{\sqrt{(1-\mu_i)\mu_i}}.$

\textbf{Case 2:} When  $ \mu_i' \le \mu_i \le  1/2$, then $|g(\mu'_i)-g(\mu_i)| = g(\mu_i)-g(\mu'_i) \le \sqrt{2} (h(\mu_i) - h(\mu'_i))$.
\begin{align*}
    h(\mu_i)-h(\mu'_i)&= 2\sqrt{\mu_i} - 2\sqrt{\mu'_i} =  2\sqrt{\mu_i} (1-\sqrt{1-\frac{|\mu'_i-\mu_i|}{\mu_i}}) \\
    &\le 2\sqrt{\mu_i} \frac{|\mu'_i-\mu_i|}{\mu_i}\le \frac{2|\mu'_i-\mu_i|}{\sqrt{(1-\mu_i)\mu_i}}.
\end{align*}
where the first inequality uses the fact that $1 - \sqrt{1-x} \le x$ for $x \in [0,1]$.
So $|g(\mu'_i)-g(\mu_i)| \le 2\sqrt{2} \frac{|\mu'_i-\mu_i|}{\sqrt{(1-\mu_i)\mu_i}}.$

\textbf{Case 3:} When $1/2 \le \mu_i \le \mu_i'$, then $|g(\mu'_i)-g(\mu_i)| = g(\mu'_i)-g(\mu_i) \le \sqrt{2} (h(\mu'_i) - h(\mu_i))$.
\begin{align*}
    h(\mu'_i)-h(\mu_i)&= 2\sqrt{1-\mu_i} - 2\sqrt{1-\mu'_i} =  2\sqrt{1-\mu_i} (1-\sqrt{1-\frac{|\mu'_i-\mu_i|}{1-\mu_i}}) \\
    &\le 2\sqrt{1-\mu_i} \frac{|\mu'_i-\mu_i|}{(1-\mu_i)}\le \frac{2|\mu'_i-\mu_i|}{\sqrt{(1-\mu_i)\mu_i}},
\end{align*}
So $|g(\mu'_i)-g(\mu_i)| \le 2\sqrt{2} \frac{|\mu'_i-\mu_i|}{\sqrt{(1-\mu_i)\mu_i}}.$

\textbf{Case 4:} When $1/2 \le \mu'_i \le \mu_i$, then $|g(\mu'_i)-g(\mu_i)| = g(\mu_i)-g(\mu'_i) \le \sqrt{2} (h(\mu_i) - h(\mu'_i))$.
\begin{align*}
    h(\mu_i)-h(\mu'_i)&= 2\sqrt{1-\mu'_i} - 2\sqrt{1-\mu_i} =  2\sqrt{1-\mu_i} (\sqrt{1+\frac{|\mu'_i-\mu_i|}{1-\mu_i}}-1) \\
    &\le \sqrt{1-\mu_i} \frac{|\mu'_i-\mu_i|}{(1-\mu_i)}\le \frac{|\mu'_i-\mu_i|}{\sqrt{(1-\mu_i)\mu_i}},
\end{align*}
So $|g(\mu'_i)-g(\mu_i)| \le \sqrt{2} \frac{|\mu'_i-\mu_i|}{\sqrt{(1-\mu_i)\mu_i}}.$

\textbf{Case 5:} When $1/2 \le \mu'_i, \mu_i  \le 1/2$, then $|g(\mu'_i)-g(\mu_i)| = g(\mu'_i)-g(\mu_i) \le \sqrt{2} (h(\mu'_i) - h(\mu_i))$.
\begin{align*}
    h(\mu'_i)-h(\mu_i)&= (h(\mu'_i)-h(1/2)) + (h(\mu_i)-h(1/2))\\
    &\le   2\frac{\mu'_i-1/2}{\sqrt{1-1/2}} + \frac{1/2-\mu_i}{\sqrt{\mu_i}} \\
    &\le \frac{\mu'_i-\mu_i}{\sqrt{\mu_i}}+ \frac{\mu'_i-\mu_i}{\sqrt{\mu_i}} \\
    &\le 3\frac{|\mu'_i-\mu_i|}{\sqrt{\mu_i(1-\mu_i)}},
\end{align*}
where the first inequality uses the results for $1/2 \le \mu_i \le \mu_i'$ and for $\mu_i \le  \mu_i' \le 1/2$, the second inequality uses the relation that $\mu_i \le 1/2$ and $\mu'_i \ge 1/2$.
So $|g(\mu'_i)-g(\mu_i)| \le 3\sqrt{2} \frac{|\mu'_i-\mu_i|}{\sqrt{(1-\mu_i)\mu_i}}.$

\textbf{Case 6:} When $1/2 \le \mu_i, \mu'_i  \le 1/2$, then $|g(\mu'_i)-g(\mu_i)| = g(\mu_i)-g(\mu'_i) \le \sqrt{2} (h(\mu_i) - h(\mu'_i))$.
\begin{align*}
    h(\mu_i)-h(\mu'_i)&= (h(\mu_i)-h(1/2)) + (h(\mu'_i)-h(1/2))\\
    &\le   \frac{\mu_i-1/2}{\sqrt{1-\mu_i}} + 2\frac{1/2-\mu'_i}{\sqrt{1/2}} \\
    &\le   \frac{\mu_i-\mu'_i}{\sqrt{1-\mu_i}} + 2\frac{\mu_i-\mu'_i}{\sqrt{1-\mu_i}} \\
    &\le 3\frac{|\mu'_i-\mu_i|}{\sqrt{\mu_i(1-\mu_i)}},
\end{align*}
where the first inequality uses the results for $1/2 \le \mu'_i \le \mu_i$ and for $\mu'_i \le  \mu_i \le 1/2$, the second inequality uses the relation that $\mu'_i \ge 1/2$ and $1-\mu_i \ge 1/2$.
So $|g(\mu'_i)-g(\mu_i)| \le 3\sqrt{2} \frac{|\mu'_i-\mu_i|}{\sqrt{(1-\mu_i)\mu_i}}.$

By above cases, we have $|g(\mu'_i)-g(\mu_i)| \le 3\sqrt{2} \frac{|\mu'_i-\mu_i|}{\sqrt{(1-\mu_i)\mu_i}}.$ Putting back this inequality into \cref{eq:gamma_g_gradient}, we have 
\begin{equation}\label{apdx_eq:vm_1}
    |r(S; \bmu')- r(S;\bmu)| \le 3\sqrt{2}\gamma_{g}\sqrt{\sum_{i \in S}\left(\frac{|\mu'_i-\mu_i|}{\sqrt{(1-\mu_i)\mu_i}}\right)^2}.
\end{equation}

\textbf{For the $\abs{r(S; \bmu'')- r(S;\bmu')}$ term,}

We can use the gradient theorem to bound,

\begin{align}
    \abs{r(S;\bmu'')-r(S;\bmu')}&= \abs{\int_{\bx=\bmu'}^{\bmu''} \nabla r(S;\bx)\cdot d\bx}\notag\\
    &\le \sup_{\bx}\norm{\nabla r(S;\bx)}_{\infty}\sum_{i \in S}\abs{\mu''_i-\mu'_i}\notag\\
    &\le \gamma_{\infty}\sum_{i \in S}|\eta_i|.\label{apdx_eq:vm_2}
\end{align}

Combining \cref{apdx_eq:vm_1} and \cref{apdx_eq:vm_2}, we conclude the lemma.
\end{proof}

\section{Regret Analysis for CMAB-T with TPVM Bounded Smoothness (Proofs Related to Theorem \ref{thm:reg_lambda1})}\label{apdx_sec:main_regret_analysis}
In this section, we provide detailed proofs for \cref{thm:reg_lambda1} and give some discussions for the distribution-independent regret bounds as well as the lower bound results.

For the structure of this section, we first introduce some useful tools in \cref{apdx_sec:tpvm_useful} that will be helpful for our analysis. Next we transform the total regret to the regret terms filtered by some events in \cref{apdx_sec:decompose_events}. Then we provide regret bounds for all these regret terms.
For these regret terms, we give two different proofs for the leading regret term: the proof giving~\cref{thm:reg_lambda1} that uses the reverse amortization trick (see \cref{eq:ra_inf_prob_e1} and \cref{eq:ra_inf_prob_e2}) are in \cref{apdx_sec:reverse_amt}, while the proof \cref{apdx_sec:inf_many_events} directly follows~\cite{merlis2019batch}.
Recall that former proof improves the latter by a factor of $O(\log K)$ and readers can skip the latter one if you are not interested. It is notable that this trick can be used to improve \citet{degenne2016combinatorial, merlis2019batch,perrault2020budgeted} in a similar way, owing to the fact that their error terms have the similar form as ours shown in \cref{apdx_eq:critical_error_term} (except without triggering probability modulation). Lastly, we summarize the detailed distribution-dependent regret, distribution-independent regret bounds and lower bounds in  \cref{apdx_sec:summary_reg}.

\subsection{Useful Concentration Bounds, Definitions and Inequalities}\label{apdx_sec:tpvm_useful}

We use the following tail bound for the construction of the confidence radius and our analysis.
\begin{lemma}[Empirical Bernstein Inequality \cite{audibert2009exploration}]\label{lem:empirical_bern}  Let $(X_i)_{i \in [n]}$ be $n$ i.i.d random variables with bounded support $[0,1]$ and mean $\E[X_i]=\mu$. Let $\hat{X}_n\triangleq \frac{1}{n}\sum_{i \in [n]}X_i$ and $\hat{V}_n\triangleq\frac{1}{n}\sum_{i \in[n]}(X_i-\hat{X}_n)^2$ be the empirical mean and empirical variance of $(X_i)_{i \in [n]}$. Then for any $n \in \mathbb{N}$ and $y > 0$, it holds that
\begin{equation}
    \Pr\left[|\hat{X}_n-\mu| \ge \sqrt{\frac{2\hat{V}_n y}{n}} + \frac{3y}{n}\right] \le 3 e^{-y}
\end{equation}
\end{lemma}

We use the following Bernstein Inequality to bound the difference between the empirical variance and the true variance.

\begin{lemma}[Bernstein Inequality \cite{dubhashi2009concentration}]\label{lem:tail_bound_bern} Let $(X_i)_{i \in [n]}$ be $n$ independent random variables in $[0,1]$ with mean $\E[X_i]=\mu$ and variance $\text{Var}[X_i]\triangleq \E[X^2_i]-(\E[X_i])^2=V$. Then with probability $1-\delta$:
\begin{align}
    \frac{1}{n}\sum_{i \in [n]}X_i\le \mu + \frac{2\log 1/\delta}{3n} + \sqrt{\frac{2V\log 1/\delta}{n}}.
\end{align}

\end{lemma}


Similar to \cite{wang2017improving}, we define the event-filtered regret, the triggering group, the counter, the nice triggering event and the nice sampling event to help our analysis. 

\begin{definition}[Event-Filtered Regret]\label{apdx_def:event_filter_reg} For any series of events $(\cE_t)_{t\ge [T]}$ indexed by round number $t$, we define the $Reg^{A}_{\alpha,\bmu}(T,(\cE_t)_{t\ge [T]})$ as the regret filtered by events $(\cE_t)_{t\ge [T]}$, or the regret is only counted in $t$ if $\cE$ happens in $t$. Formally, 
\begin{align}
    Reg^{A}_{\alpha,\bmu}(T,(\cE_t)_{t\ge [T]}) = \E\left[\sum_{t\in[T]}\I(\cE_t)(\alpha\cdot r(S^*;\bmu)-r(S_t;\bmu))\right].
\end{align}
For simplicity, we will omit $A,\alpha,\bmu,T$ and rewrite $Reg^{A}_{\alpha,\bmu}(T,(\cE_t)_{t\ge [T]})$ as $Reg(T, \cE_t)$ when contexts are clear.

\end{definition}

\begin{definition}[Triggering Probability (TP) group]
For any arm $i$ and index $j$, define the triggering probability (TP) group (of actions) as
\begin{equation}
    \cS^D_{i,j}=\{S\in \cS: 2^{-j} < p_{i}^{D,S} \le 2^{-j+1}\}.
\end{equation}
Notice $\{\cS^D_{i,j}\}$ forms a partition of $\{S \in \cS: p_{i}^{D,S}\}$.
\end{definition}

\begin{definition}[Counter]
For each TP group $S_{i,j}$, we define a counter $N_{i,j}$ which is initialized to $0$. In each round $t$, if the action $S_t$ is chosen, then we update $N_{i,j}$ to $N_{i,j}+1$ for $(i,j)$ that $S_t \in S^D_{i,j}$. We also denote $N_{i,j}$ at the end of round $t$ as $N_{t,i,j}$. Formally, we have the following recursive equation to define $N_{t,i,j}$ as follows:
\begin{equation}
N_{t,i,j} = \begin{cases}
0, &\text{if $t=0$}\\
N_{t-1,i,j} + 1, &\text{if $t>0$ and $S_{t} \in S^D_{i,j}$}\\
N_{t-1,i,j}, &\text{otherwise.}
\end{cases}
\end{equation}
\end{definition}

\begin{definition}[Nice triggering event $\mathcal{N}_{t}^t$]\label{apdx_def:nice_triggering}
Given a series integers $\{j_{i}^{\max}\}_{i \in [m]}$, we say that the triggering is nice at the beginning of round $t$, if for every triggered group identified by $(i,j)$, as long as $\frac{6 \ln t}{\frac{1}{3}N_{t-1,i,j}2^{-j}}\le 1$, there is $T_{t-1,i}\ge\frac{1}{3}N_{t-1,i, j} \cdot 2^{-j}$. We denote this event as $\mathcal{N}_{t}^t$. 
\end{definition}

\begin{lemma}[Appendix B.1, Lemma 4 \cite{wang2017improving}]\label{apdx_lem:prob_nice_triggering}
For a series of integers $(j_i^{\max})_{i \in [m]}$, we have $\Pr[\neg \mathcal{N}_t^t] \le \sum_{i \in [m]}j_{i}^{\max} t^{-2}$ for every round $t \in [T]$.
\end{lemma}
\begin{proof}
We refer the readers to Lemma 4 in Appendix B.1 from \citet{wang2017improving} for detailed proofs.
\end{proof}

\begin{definition}\label{apdx_def:nice_sampling} We say that the sampling is nice at the beginning of round $t$ if: (1) for every base arm $i \in [m]$, $|\hat{\mu}_{t-1, i}-\mu_i|\le \rho_{t,i}$, where $\rho_{t,i}=\sqrt{\frac{6\hat{V}_{t-1,i} \log t}{T_{t-1, i}}} + \frac{9\log t}{T_{t-1, i}}$; (2) for every base arm $i \in [m]$, $\hat{V}_{t-1,i} \le 2\mu_i(1-\mu_i) + \frac{3.5\log t}{T_{t-1,i}}$. We denote such event as $\cN_{t}^s$.
\end{definition}

The following lemma bounds the probability that $\cN_t^s$ does not happen.
\begin{lemma} \label{apdx_lem:prob_nice_sampling}
For each round $t$, $\Pr[\neg \cN_{t}^s] \le 4m t^{-2}$.
\end{lemma}
\begin{proof}
Let $\cN_t^{s,1}, \cN_t^{s,2}$ be the event (1) and event (2), where $\cN_t^{s}=\cN_t^{s,1}\bigcap \cN_t^{s,2}$.
We first bound the probability that $\cN_t^{s,1}$ does not happen, we have
\begin{align}
    \Pr[\neg\cN_t^{s,1}]&=\Pr\left[\exists i \in [m] \text{ s.t. } |\hat{\mu}_{t-1, i}-\mu_i|> \sqrt{\frac{6\hat{V}_{t-1,i} \log t}{T_{t-1, i}}} + \frac{9\log t}{T_{t-1, i}}\right]\\
    &\le \sum_{i \in [m]}\sum_{\tau \in [t]}\Pr\left[|\hat{\mu}_{t-1, i}-\mu_i|> \sqrt{\frac{6\hat{V}_{t-1,i} \log t}{\tau}} + \frac{9\log t}{\tau}, T_{t-1, i}=\tau\right]\label{eq_s1:union}\\
    &\le 3mt^{-2}\label{eq_s1:bern},
\end{align}
where \cref{eq_s1:union} is due to the union bound over $i,\tau$, \cref{eq_s1:bern} is due to \cref{lem:empirical_bern} by setting $y=3\log t$ and when $T_{t-1,i}=\tau, \hat{\mu}_{t-1, i}$ and $\hat{V}_{t-1,i}$ are the empirical mean and empirical variance of $\tau$ i.i.d random variables with mean $\mu_i$.

We then bound the probability that second event $\cN_t^{s,2}$ does not happen using the similar proof of \citep[Eq. (7)]{merlis2019batch}. Fix $T_{t-1,i}=\tau$ and consider $(Y_i^{1}, ..., Y_i^{\tau})$, where $Y_i^k=(X_i^k-\mu_i)^2\in [0,1]$ and $X_i^k$ is the random outcome of the $k$-th i.i.d trial. Since $X_i^k$ are independent across $k$, $Y_i^k$ are independent across $k$ as well. In this case, one can verify that $\hat{V}_{t-1,i}= \frac{1}{\tau}\sum_{k=1}^{\tau}(X_i^k-\mu_i)^2-(\frac{1}{\tau}\sum_{k=1}^{\tau}X_i^k-\mu_i)^2\le \frac{1}{\tau}\sum_{k=1}^{\tau}(X_i^k-\mu_i)^2=\frac{1}{\tau}\sum_{k=1}^{\tau}Y_{i}^k$; $\E[Y_i^k]=\E[(X_i^k)^2]-\mu_i^2\le \E[X_i^k] \cdot 1-\mu_i^2 = (1-\mu_i)\mu_i$; and $\text{Var}[Y_i] = \E[(Y_i^k)^2]-(\E[Y_i^k])^2\le \E[(Y_i^k)^2]\le \E[Y_i^k] \le (1-\mu_i)\mu_i$.
By \cref{lem:tail_bound_bern} over $\tau$ i.i.d random variable $(Y_i^{k})_{k \in \tau}$, it holds with probability at least $1-t^{-3}$ that
\begin{align}
    \frac{1}{\tau}\sum_{k=1}^{\tau}Y_i^k \le \E[Y_i^k]+  \frac{2\log t}{\tau} + \sqrt{\frac{6\text{Var}[Y_i^k]\log t}{\tau}}
\end{align}
This implies 
\begin{align}
    \hat{V}_{t-1,i} &\le \frac{1}{\tau}\sum_{k=1}^{\tau}Y_i^k \le \E[Y_i^k]+  \frac{2\log t}{\tau} + \sqrt{\frac{6\text{Var}[Y_i^k]\log t}{\tau}}\\
    &\le \mu_i(1-\mu_i) + \frac{2\log t}{\tau} + \sqrt{\frac{6(1-\mu_i)\mu_i\log t}{\tau}}\label{eq_s2:square}\\
    &\le \mu_i(1-\mu_i) + \frac{2\log t}{\tau} + \mu_i(1-\mu_i) + \frac{3\log t}{2\tau}\\
    &=2\mu_i(1-\mu_i) + \frac{3.5\log t}{\tau}
\end{align}
where \cref{eq_s2:square} is using $2ab \le a^2 + b^2$ and $a=\sqrt{2\mu_i(1-\mu_i)}, b=\sqrt{\frac{3\log t}{n}}$.

Now by applying union bound over $i \in [m]$ and $\tau \in [t]$, we have $\Pr[\neg \cN_t^{s,2}] \le mt^{-2}$.
Lastly, applying union bound over $\cN_t^{s,1}$ and $\cN_t^{s,2}$, we have $\Pr[\neg \cN_t^{s}] \le 4 mt^{-2}$.
\end{proof}

After setting up all above definitions, we can prove \cref{lem:conf_rad} about the confidence radius, which appears in the main content.
\begin{restatable}{lemma}{radius}\label{lem:conf_rad} Fix every base arm $i$ and every time $t$, with probability at least $1-4mt^{-3}$, it holds that
\begin{equation}\label{eq:lem_conf}
     \mu_{i} \le \bar{\mu}_{t,i} \le \min\{\mu_{i} + 2 \rho_{t,i},1\}  \le  \min\left\{\mu_{i} + 4\sqrt{3}\sqrt{\frac{\mu_i(1-\mu_i) \log t}{T_{t-1, i}}} + \frac{28\log t}{T_{t-1, i}},1\right\}.
\end{equation}
\end{restatable}
\begin{proof}
Recall that $\bar{\mu}_{t,i}=\min\{\hat{\mu}_{t-1,i} + \rho_{t,i}, 1\}=\min\{\hat{\mu}_{t-1,i}+\sqrt{\frac{6\hat{V}_{t-1, i} \log t}{T_{t-1, i}}} + \frac{9\log t}{T_{t-1, i}},1\}$.
Under event $N_t^{s,1}$, we have $|\mu_i-\hat{\mu}_{t,i}|\le \rho_{t,i}$ by \cref{apdx_lem:prob_nice_sampling}, hence the first and the second inequality in \cref{lem:conf_rad} holds.

For the last inequality, under event $N_t^{s,2}$, it holds that
\begin{align}
    \mu_{i} + 2 \rho_{t,i}  &=  \mu_i + 2\left( \sqrt{\frac{6\hat{V}_{t-1, i} \log t}{T_{t-1, i}}} + \frac{9\log t}{T_{t-1, i}}\right)\\
    &\le \mu_i + 2\left(\sqrt{\frac{6\cdot (2\mu_i(1-\mu_i) + \frac{3.5\log t}{T_{t-1,i}}) \log t}{T_{t-1, i}}} + \frac{9\log t}{T_{t-1, i}}\right)\\
    &\le \mu_i + 4\sqrt{3}\sqrt{\frac{ \mu_i(1-\mu_i)\log t}{T_{t-1, i}}} +  2\sqrt{21}\frac{\log t}{T_{t-1, i}} + \frac{18\log t}{T_{t-1, i}}\label{apdx_eq:lem_conf}\\
    &\le\mu_{t-1,i} + 4\sqrt{3}\sqrt{\frac{\mu_i(1-\mu_i) \log t}{T_{t-1, i}}} + \frac{28\log t}{T_{t-1, i}},
\end{align}
where \cref{apdx_eq:lem_conf} uses $\sqrt{a+b}\le\sqrt{a}+\sqrt{b}$. 

Since $\cN_t^{s}=\cN_t^{s,1}\bigcap \cN_t^{s,2}$ and by \cref{apdx_lem:prob_nice_sampling}, \cref{eq:lem_conf} holds with probability at least $1-4mt^{-2}$.
\end{proof}

\subsection{Decompose the Total Regret to Event-Filtered Regrets}\label{apdx_sec:decompose_events}
In this section, we decompose the regret $Reg(T,\{\}) = Reg(T, \cN_t^s, \cN_t^o) + Reg(T, \neg(\cN_t^s \bigcap \cN_t^o))\le Reg(T, \cN_t^s, \cN_t^o) + Reg(T, \neg\cN_t^s)  + Reg(T,\neg \cN_t^o)$, where $\cN_t^s$ is defined in \cref{apdx_def:nice_sampling}, 
$\cN_t^o$ denotes the event where oracle successfully outputs an $\alpha$-approximate solution (with probability at least $\beta$).
We have the following lemma to do the decomposition.
\begin{restatable}{lemma}{error_term}[Leading Regret Term]\label{apdx_lem:leading_regret}
Let $r(S;\bmu)$ be TPVM smoothness with coefficients $(B_v,B_1,\lambda)$, and define the error term \begin{align}\label{apdx_eq:critical_error_term}
    e_t(S_t)=4\sqrt{3}B_v\sqrt{\sum_{i\in \tilde{S}_t}(\frac{\log t}{T_{t-1,i}}\wedge \frac{1}{28})(p_{i}^{D,S_t})^\lambda} + 28B_1\sum_{i \in \tilde{S}_t}(\frac{\log t}{T_{t-1,i}}\wedge \frac{1}{28})(p_{i}^{D,S_t})
\end{align} and event $E_t = \I\{\Delta_{S_t} \le e_t(S_t)\}.$
The regret of Algorithm~\ref{alg:BCUCB-T}, when used with $(\alpha, \beta)$ approximation oracle is bounded by
\begin{equation}\label{eq:e_t}
    Reg(T) \le Reg(T,E_t) + \frac{2\pi^2}{3}m\Delta_{\max}.
\end{equation}  
\end{restatable}
\begin{proof}
    Under event $\cN_{t}^{s}, \cN_t^o$, by \cref{lem:conf_rad}, it is easily to check that 
    \begin{align}
        \bar{\mu}_{t,i}&\le\min\{\mu_{t-1,i} + 4\sqrt{3}\sqrt{\frac{\mu_i(1-\mu_i) \log t}{T_{t-1, i}}} + \frac{28\log t}{T_{t-1, i}},1\}\notag\\
        &\le \mu_{t-1,i} + 4\sqrt{3}\sqrt{\mu_i(1-\mu_i) (\frac{\log t}{T_{t-1, i}} \wedge \frac{1}{28})} + 28(\frac{\log t}{T_{t-1, i}}\wedge \frac{1}{28})\label{apdx_eq:leading_0}
    \end{align}
Therefore, it holds that
\begin{align}
    \alpha r(S^*; \bmu) &\le \alpha r(S^*; \bar{\bmu}_t) \le r(S_t; \bar{\bmu}_t)\label{apdx_eq:lem_leading_1}\\
    &\le r(S_t;\bmu) + 4\sqrt{3}B_v\sqrt{\sum_{i\in \tilde{S}_t}(\frac{\log t}{T_{t-1,i}}\wedge \frac{1}{28})(p_{i}^{D,S_t})^\lambda} + 28B_1\sum_{i \in \tilde{S}_t}(\frac{\log t}{T_{t-1,i}}\wedge \frac{1}{28})(p_{i}^{D,S_t})\label{apdx_eq:lem_leading_2},
\end{align}
where the first inequality in \cref{apdx_eq:lem_leading_1} is due to monotonicity condition (\cref{cond:mono}) and second inequality in \cref{apdx_eq:lem_leading_1} is due to event $\cN_{t}^o$, \cref{apdx_eq:lem_leading_2} is because of \cref{apdx_eq:leading_0} and the TPVM condition (\cref{cond:TPVMm}) by plugging in $\zeta_i=4\sqrt{3}\sqrt{\mu_i(1-\mu_i) (\frac{\log t}{T_{t-1, i}} \wedge \frac{1}{28})}$ and $\eta_i=28(\frac{\log t}{T_{t-1, i}}\wedge \frac{1}{28})$.

So $Reg(T,\cN_{t}^{s}, \cN_t^o)\le Reg(T,E_t)$.
Now for $Reg(T,\neg\cN_t^{s})$, by \cref{apdx_lem:prob_nice_sampling} it holds that
\begin{align}
    Reg(T,\neg\cN_t^{s}) \le \sum_{t=1}^T \Pr[\neg\cN_t^s] \le \sum_{t=1}^T 4mt^{-2}\le \frac{2\pi^2}{3}m\Delta_{\max}.
\end{align}

Similarly by definition, it holds that 
\begin{align}\label{apdx_eq:oracle_fail}
    Reg(T,\neg \cN_t^o) \le (1-\beta) T \Delta_{\max}.
\end{align}

Therefore $Reg(T,\{\})\le Reg(T,E_t) + \frac{2\pi^2}{3}m\Delta_{\max} +  (1-\beta) T \Delta_{\max}$. And we have $Reg(T)=Reg(T,\{\})-(1-\beta)T\Delta{\max} \le Reg(T,E_t)+\frac{2\pi^2}{3}\Delta_{\max}$, which concludes \cref{apdx_lem:leading_regret}.

\end{proof}

Recall that event $E_t=\{\Delta_{S_t} \le e_t(S_t)\}$, where $ e_t(S_t)=4\sqrt{3}B_v\sqrt{\sum_{i\in \tilde{S}_t}(\frac{\log t}{T_{t-1,i}}\wedge \frac{1}{28})(p_{i}^{D,S_t})^\lambda} + 28B_1\sum_{i \in \tilde{S}_t}(\frac{\log t}{T_{t-1,i}}\wedge \frac{1}{28})(p_{i}^{D,S_t})$. We will further decompose the event-filtered regret $Reg(T, E_t)$ into two event-filtered regret $Reg(T, E_{t,1})$ and $Reg(T,E_{t,2})$,
\begin{align}\label{apdx_eq:decompose}
    Reg(T, E_t) \le Reg(T, E_{t,1})+Reg(T, E_{t,2}) ,
\end{align}
where $E_{t,1}=\{\Delta_{S_t} \le 2e_{t,1}(S_t)\}$, $E_{t,2}=\{\Delta_{S_t} \le 2e_{t,2}(S_t)\}$,  $e_{t,1}(S_t)=4\sqrt{3}B_v\sqrt{\sum_{i\in \tilde{S}_t}(\frac{\log t}{T_{t-1,i}}\wedge \frac{1}{28})(p_{i}^{D,S_t})^\lambda}$,$ e_{t,2}(S_t)=28B_1\sum_{i \in \tilde{S}_t}(\frac{\log t}{T_{t-1,i}}\wedge \frac{1}{28})(p_{i}^{D,S_t})$.
The above inequality holds since the following facts: We can observe $e_{t,1}(S_t)+e_{t,2}(S_t)=e_{t}(S_t)$. From $E_t$, we know either $E_{t,1}$ holds or $E_{t,2}$ holds. So $E_t$ implies that $1 \le \I\{E_{t,1}\} + \I\{E_{t,2}\}$, and thus $\Delta_{S_t}\I\{E_t\} \le \Delta_{S_t} \I\{E_{t,1}\} + \Delta_{S_t} \I\{E_{t,2}\}$, which concludes $Reg(T, E_t) \le Reg(T, E_{t,1})+Reg(T, E_{t,2})$.
The next two sections will provide two different proofs for $Reg(T, E_{t,1}),Reg(T, E_{t,2})$ separately, where the second improves the first by a factor of $O(\log K)$.

\subsection{Our Improved Analysis Using the Reverse Amortized Trick}\label{apdx_sec:reverse_amt}
In this section, we are going to bound the $Reg(T, E_{t,1})$ and $Reg(T, E_{t,2})$ separately under the event $\cN_{t}^t$, similar to \cref{apdx_sec:inf_many_events}. 
The idea is to use a refined reverse amortization trick originated in \cite{wang2017improving} and to allocate the regret $\Delta_{S_t}$ to each base arm according to carefully designed thresholds. Note that it is highly non-trivial to derive the right thresholds and regret allocation strategy so that the $K, T$ factors are as small as possible, which is our main contribution.
\subsubsection{Upper bound for $Reg(T, E_{t,1})$}\label{apdx_sec:improve_e1}
We first break $Reg(T,E_{t,1})$ into two parts and bound them separately: $Reg(T, E_{t,1}\bigcap \cN_{t}^t)$ and $Reg(T, \neg \cN_t^t)$.

For $Reg(T, E_{t,1}\bigcap \cN_{t}^t)$, under the event $\cN_{t}^t$, let $c_1=4\sqrt{3}$ and we set $j_i^{\max}=\frac{1}{\lambda}(\lceil \log_2\frac{c_1^2B_v^2 K}{(\Delta_{i}^{\min})^2}\rceil+1)$.
We first define a regret allocation function 
\begin{equation}
\kappa_{i,j,T}(\ell) = \begin{cases}
\frac{c_1^2 B_v^2 2^{(-j+1)({\lambda-1})}}{\Delta_i^{\min}}, &\text{if $\ell=0$ and $j \le j_i^{\max}$,}\\
2\sqrt{\frac{24c_1^2 B_v^2 2^{(-j+1)({\lambda-1})}\log T}{\ell }}, &\text{if $1\le\ell\le L_{i,j,T,1}$ and $j \le j_i^{\max}$,}\\
\frac{48c_1^2 B_v^2 2^{(-j+1)({\lambda-1})}\log T}{\Delta_i^{\min} }\frac{1}{\ell}, &\text{if $L_{i,j,T,1} < \ell \le L_{i,j,T,2}$ and $j \le j_i^{\max}$,}\\
0, &\text{if $\ell > L_{i,j,T,2}$ or $ j > j_i^{\max}$,}
\end{cases}
\end{equation}
where $L_{i,j,T,1}=\frac{24c_1^2 B_v^2 2^{(-j+1)({\lambda-1})}\log T}{(\Delta_{i}^{\min})^2}$, $L_{i,j,T,2}=\frac{48c_1^2 B_v^2 2^{(-j+1)({\lambda-1})}K\log T}{(\Delta_i^{\min})^2}$.
\begin{lemma}
For any time $t \in [T]$, if $N_{t}^t$ and $E_{t,1}$ hold, we have 
\begin{equation}\label{apdx_eq:kappa_e1}
    \Delta_{S_t} \le \sum_{i \in \tilde{S}_t} \kappa_{i,j_i^{S_t}, T} (N_{t-1, i,j_i^{S_t}}),
\end{equation}
where $j_{i}^{S_t}$ is the index of the triggering group $S_{i,j}$ such that $2^{-j_{i}^{S_t}}< p_{i}^{D,S_t} \le 2^{-j_{i}^{S_t}+1}$.
\end{lemma}
\begin{proof}
By event $E_{t,1}$, which is defined in \cref{apdx_eq:critical_error_term}, we apply the reverse amortization (\cref{eq:ra_inf_prob_e1})
\begin{align}
     \Delta_{S_t} &\le \sum_{i \in \tilde{S}_t} \frac{ 4c_1^2 B_v^2 (p_{i}^{D,S_t})^\lambda \min\{\frac{\log t}{T_{i,t-1}}, \frac{1}{28}\}}{\Delta_{S_t}} \label{eq:ra_def_e1}\\
    &\le -\Delta_{S_t} +  2\sum_{i \in \tilde{S}_t} \frac{ 4c_1^2 B_v^2 (p_{i}^{D,S_t})^\lambda \min\{\frac{\log t}{T_{i,t-1}}, \frac{1}{28}\}}{\Delta_{S_t}} \label{eq:ra_double}\\
     &\le \sum_{i \in \tilde{S}_t}\left( \frac{8c_1^2 B_v^2 (p_{i}^{D,S_t})^\lambda \min\{\frac{\log t}{T_{i,t-1}}, \frac{1}{28}\}}{\Delta_{S_t}} -\frac{\Delta_{S_t}}{|\tilde{S}_t|}\right)\label{eq:ra_def_reverse}\\
    &\le  \sum_{i \in \tilde{S}_t}\left(\frac{8c_1^2 B_v^2 (p_{i}^{D,S_t})^\lambda\min\{\frac{\log t}{\frac{1}{3}N_{t-1, i, j_i^{S_t}}2^{-j_i^{S_t}}}, \frac{1}{28}\}}{\Delta_{S_t}} -\frac{\Delta_{S_t}}{|\tilde{S}_t|}\right) \label{eq:ra_inf_def_e1}\\
    &\le  \sum_{i \in \tilde{S}_t}\underbrace{\left(\frac{8c_1^2 B_v^2 (2^{-j_i^{S_t}+1})^\lambda\min\{\frac{\log t}{\frac{1}{3}N_{t-1, i, j_i^{S_t}}2^{-j_i^{S_t}}}, \frac{1}{28}\}}{\Delta_{S_t}} -\frac{\Delta_{S_t}}{K}\right)}_{(\ref{eq:ra_inf_prob_e1}, i)} \label{eq:ra_inf_prob_e1},
\end{align}
where \cref{eq:ra_def_e1} is by the definition of $E_{t,1}$ which says $\Delta^2_{S_t} \le \sum_{i \in \tilde{S}_t} 4c_1^2 B_v^2 (p_{i}^{D,S_t})^\lambda \min\{\frac{\log t}{T_{i,t-1}}, \frac{1}{28}\}$ and by dividing both sides by $\Delta_{S_t}>0$,
\cref{eq:ra_double} is because we double the LHS and RHS of \cref{eq:ra_def_e1} at the same time and then put one into the RHS,
\cref{eq:ra_def_reverse} is by putting $-\Delta_{S_t}$ inside the summation,
\cref{eq:ra_inf_def_e1} is due to the same reason of \cref{eq:inf_def_e1} under event $\cN_{t}^t$,
\cref{eq:ra_inf_prob_e1} is due to $p_i^{D,S_t}\le 2^{-j_{i}^{S_t}+1}$ given by the definition of $j_i^{S_t}$ and $|\tilde{S}_t|\le K$.

Note that the \cref{eq:ra_double} is called the reverse amortization trick, since we allocate two times of the total regret and then minus the $\Delta_{S_t}$ term to amortize the regret when $\ell > L_{i,j,T,2}$ or $ j > j_i^{\max}$ in \cref{apdx_eq:kappa_e1}, which saves the analysis for arms that are sufficiently triggered. Now we bound (\ref{eq:ra_inf_prob_e1}, i) under different cases.

\textbf{When $j >j_i^{\max}$}, 

we have $(\ref{eq:ra_inf_prob_e1}, i)\le \frac{8c_1^2 B_v^2 (2^{-j_i^{S_t}+1})^\lambda}{\Delta_{S_t}}\cdot \frac{1}{28} -\frac{\Delta_{S_t}}{K} \le \frac{8c_1^2 B_v^2  }{\Delta_{S_t}} \frac{(\Delta_i^{\min})^2}{c_1^2 B_v^2 K}\cdot \frac{1}{28} -\frac{\Delta_{S_t}}{K}\le \frac{\Delta_{i}^{\min}}{K}\cdot \frac{8}{28}-\frac{\Delta_{S_t}}{K}\le 0 =\kappa_{i,j_i^{S_t}, T} (N_{t-1, i,j_i^{S_t}})$.

\textbf{When $N_{t-1, i,j_i^{S_t}} > L_{i,j_i^{S_t}, T, 2}$,}

we have $(\ref{eq:ra_inf_prob_e1}, i)\le  \frac{8c_1^2 B_v^2 (2^{-j_i^{S_t}+1})^\lambda\log t}{\frac{1}{3}N_{t-1, i, j_i^{S_t}}\cdot 2^{-j_i^{S_t}}\Delta_{S_t}} -\frac{\Delta_{S_t}}{K} \le \frac{48c_1^2 B_v^2 2^{(-j_i^{S_t}+1)({\lambda-1})}\log T}{\Delta_{S_t}} \frac{1}{N_{t-1, i, j_i^{S_t}}} -\frac{\Delta_{S_t}}{K} <\frac{(\Delta_i^{\min})^2}{K\Delta_{S_t}}-\frac{\Delta_{S_t}}{K}\le  0 = \kappa_{i,j_i^{S_t}, T} (N_{t-1, i,j_i^{S_t}})$.

\textbf{When $ L_{i,j_i^{S_t}, T, 1} < N_{t-1, i,j_i^{S_t}} \le L_{i,j_i^{S_t}, T, 2}$ and $j \le j_i^{\max}$,}

We have $(\ref{eq:ra_inf_prob_e1}, i)\le  \frac{8c_1^2 B_v^2 (2^{-j_i^{S_t}+1})^\lambda\log t}{\frac{1}{3}N_{t-1, i, j_i^{S_t}}2^{-j_i^{S_t}}{\Delta_{S_t}}} -\frac{\Delta_{S_t}}{K} \le \frac{48c_1^2 B_v^2 2^{(-j_i^{S_t}+1)({\lambda-1})}\log T}{\Delta_{S_t}} \frac{1}{N_{t-1, i, j_i^{S_t}}} -\frac{\Delta_{S_t}}{K} < \frac{48c_1^2 B_v^2 2^{(-j_i^{S_t}+1)({\lambda-1})}\log T}{\Delta_{S_t}} \frac{1}{N_{t-1, i, j_i^{S_t}}} \le \frac{48c_1^2 B_v^2 2^{(-j_i^{S_t}+1)({\lambda-1})}\log T}{\Delta_i^{\min}}  \frac{1}{N_{t-1, i, j_i^{S_t}}} =   \kappa_{i,j_i^{S_t}, T} (N_{t-1, i,j_i^{S_t}})$.

\textbf{When $N_{t-1, i,j_i^{S_t}} \le L_{i,j_i^{S_t}, T, 1}$ and $j \le j_i^{\max}$,}

We further consider two different cases $N_{t-1, i,j_i^{S_t}} \le \frac{24c_1^2 B_v^2 2^{(-j_i^{S_t}+1)({\lambda-1})}\log T}{(\Delta_{S_t})^2}$ or $\frac{24c_1^2 B_v^2 2^{(-j_i^{S_t}+1)({\lambda-1})}\log T}{(\Delta_{S_t})^2} < N_{t-1, i,j_i^{S_t}} \le L_{i,j_i^{S_t},T,1}= \frac{24c_1^2 B_v^2 2^{(-j_i^{S_t}+1)({\lambda-1})}\log T}{(\Delta_{i}^{\min})^2}$.

For the former case, if there exists $i\in \tilde{S}_t$ so that 
$N_{t-1, i,j_i^{S_t}} \le \frac{24c_1^2 B_v^2 2^{(-j_i^{S_t}+1)({\lambda-1})}\log T}{(\Delta_{S_t})^2}$, then we know  $\sum_{q \in \tilde{S}_t} \kappa_{i,j_q^{S_t}, T} (N_{t-1, q,j_q^{S_t}}) \ge \kappa_{i,j_i^{S_t}, T} (N_{t-1, i,j_i^{S_t}})=2\sqrt{\frac{24c_1^2 B_v^2 2^{(-j_i^{S_t}+1)({\lambda-1})}\log T}{N_{t-1, i,j_i^{S_t}}}}  \ge 2\Delta_{S_t} > \Delta_{S_t}$, which makes \cref{apdx_eq:kappa_e1} holds no matter what. This means we do not need to consider this case for good.


For the later case, when $\frac{24c_1^2 B_v^2 2^{(-j_i^{S_t}+1)({\lambda-1})}\log T}{(\Delta_{S_t})^2} < N_{t-1, i,j_i^{S_t}}$, we know that $(\ref{eq:ra_inf_prob_e1}, i)\le\frac{48c_1^2 B_v^2 2^{(-j_i^{S_t}+1)({\lambda-1})}\log T}{\Delta_{S_t}} \frac{1}{N_{t-1, i, j_i^{S_t}}} =2\sqrt{\frac{24c_1^2 B_v^2 2^{(-j_i^{S_t}+1)({\lambda-1})}\log T}{(\Delta_{S_t})^2} \frac{1}{N_{t-1, i, j_i^{S_t}}}}\sqrt{\frac{24c_1^2 B_v^2 2^{(-j_i^{S_t}+1)({\lambda-1})}\log T}{N_{t-1, i, j_i^{S_t}}}} \le 2\sqrt{\frac{24c_1^2 B_v^2 2^{(-j_i^{S_t}+1)({\lambda-1})}\log T}{N_{t-1, i, j_i^{S_t}}}}=\kappa_{i,j_i^{S_t}, T} (N_{t-1, i,j_i^{S_t}})$. 

\textbf{When  $\ell=0$ and $j \le j_i^{\max}$,}

We have $(\ref{eq:ra_inf_prob_e1}, i)\le \frac{8c_1^2 B_v^2 (2^{-j_i^{S_t}+1})^\lambda}{\Delta_{S_t}}\cdot \frac{1}{28} -\frac{\Delta_{S_t}}{K} \le \frac{c_1^2 B_v^2 (2^{-j_i^{S_t}+1})^\lambda}{\Delta_{S_t}}\le \frac{c_1^2 B_v^2 (2^{-j_i^{S_t}+1})^\lambda}{\Delta_i^{\min}}=\kappa_{i,j_i^{S_t}, T} (N_{t-1, i,j_i^{S_t}})$.

Combining all above cases, we have  $\Delta_{S_t} \le \sum_{i \in \tilde{S}_t} \kappa_{i,j_i^{S_t}, T} (N_{t-1, i,j_i^{S_t}})$.
\end{proof}

Since $N_{t,i, j_{i}^{S_t}}$ is increased if and only if $i \in \tilde{S_t}$ and consider all possible $N_{t,i,j_{i}^{S}}$ where $\kappa_{i, j_{i}^{S}, T}(S, N_{t-1,i,j^{S}})>0$, we have
\begin{align}
    &Reg(T, E_{t,1} \bigcap N_{t}^t) \notag\\
    &\le \sum_{t\in [T]}\sum_{i \in \tilde{S}_t} \kappa_{i, j_{i}^{S_t}, T}(N_{t-1,i,j^{S_t}})\label{apdx_eq:reverse_regret_e1_1}\\
    &\le \sum_{i\in [m]}\sum_{j=1}^{j_i^{\max}}\frac{c_1^2 B_v^2 (2^{-j+1})^\lambda}{\Delta_i^{\min}}+\sum_{i\in [m]}\sum_{j=1}^{j_i^{\max}} \sum_{\ell=1}^{L_{i,j,T,1}}2\sqrt{\frac{24c_1^2 B_v^2 2^{(-j+1)({\lambda-1})}\log T}{\ell }}\notag\\
    & +  \sum_{i\in [m]}\sum_{j=1}^{j_i^{\max}} \sum_{L_{i,j,T,1}+1}^{L_{i,j,T,2}}\frac{48c_1^2 B_v^2 2^{(-j+1)({\lambda-1})}\log T}{\Delta_{i}^{\min} }\frac{1}{\ell}\label{apdx_eq:replace_delta_s_e1}\\
    &\le \sum_{i\in [m]}\sum_{j=1}^{j_i^{\max}}\frac{c_1^2 B_v^2 (2^{-j+1})^\lambda}{\Delta_i^{\min}}+ \sum_{i\in [m]}\sum_{j=1}^{j_i^{\max}} \frac{96 c_1^2 B_v^2 2^{(-j+1)({\lambda-1})}\log T}{\Delta_{i}^{\min}} \notag\\
    &+ \sum_{i\in [m]}\sum_{j=1}^{j_i^{\max}} \frac{48c_1^2 B_v^2 2^{(-j+1)({\lambda-1})}\log T}{\Delta_{i}^{\min} } (1+ \log K)\\
    &=\sum_{i\in [m]}\sum_{j=1}^{j_i^{\max}}\frac{c_1^2 B_v^2 (2^{-j+1})^\lambda}{\Delta_i^{\min}}+\sum_{i\in [m]}\sum_{j=1}^{j_i^{\max}}\frac{48c_1^2 B_v^2 2^{(-j+1)({\lambda-1})}\log T}{\Delta_{i}^{\min} } (3+ \log K)
\end{align}
When $\lambda > 1$, we have $Reg(T, E_{t,1} \bigcap N_{t}^t) \le\sum_{i\in [m]}\sum_{j=1}^{\infty}\frac{c_1^2 B_v^2 2^{-j+1}}{\Delta_i^{\min}}+ \sum_{i\in [m]}\sum_{j=1}^{\infty}\frac{48c_1^2 B_v^2 2^{(-j+1)({\lambda-1})}\log T}{1-2^{(\lambda-1)}\Delta_{i}^{\min} } (3+ \log K) =\sum_{i\in [m]}\frac{2c_1^2 B_v^2 }{\Delta_i^{\min}}+  \sum_{i \in [m]} \frac{48c_1^2 B_v^2 \log T}{\Delta_{i}^{\min} } (3+ \log K) $.

When $\lambda = 1$, we have $Reg(T, E_{t,1} \bigcap N_{t}^t) \le \sum_{i\in [m]}\sum_{j=1}^{\infty}\frac{c_1^2 B_v^2 2^{-j+1}}{\Delta_i^{\min}} + \sum_{i\in [m]}j_i^{\max}\frac{48c_1^2 B_v^2 \log T}{\Delta_{i}^{\min} } (3+ \log K) = \sum_{i\in [m]}\frac{2c_1^2 B_v^2 }{\Delta_i^{\min}}+ \sum_{i \in [m]} \log\frac{c_1^2B_v^2 K}{(\Delta_{i}^{\min})^2}\frac{48c_1^2 B_v^2 \log T}{\Delta_{i}^{\min} } (3+ \log K)$.

Similar to \cref{apdx_eq:additional_Ntt}, with additional $Reg(T, \neg\cN_t^t)$, we have the following inequality holds:

When $\lambda > 1$, we have $Reg(T, E_{t,1}) \le  \sum_{i\in [m]}\frac{2c_1^2 B_v^2 }{\Delta_i^{\min}}+ \sum_{i \in [m]} \frac{48c_1^2 B_v^2 \log T}{\Delta_{i}^{\min} } (3+ \log K) + \frac{m\pi^2}{6}\log_2\left(\frac{c_1^2B_v^2 K}{\lambda(\Delta_{\min})^2}\right) \Delta_{\max}$.

When $\lambda = 1$, we have $Reg(T, E_{t,1}) \le    \sum_{i\in [m]}\frac{2c_1^2 B_v^2 }{\Delta_i^{\min}}+\sum_{i \in [m]} \log\frac{c_1^2B_v^2 K}{(\Delta_{i}^{\min})^2}\frac{48c_1^2 B_v^2 \log T}{\Delta_{i}^{\min} } (3+ \log K) + \frac{m\pi^2}{6}\log_2\left(\frac{c_1^2B_v^2 K}{(\Delta_{\min})^2}\right) \Delta_{\max}$.

\subsubsection{Upper bound for $Reg(T, E_{t,2})$}\label{apdx_sec:improve_e2}
As usual, we first break $Reg(T,E_{t,2})$ into two parts and bound them separately: $Reg(T, E_{t,2}\bigcap \cN_{t}^t)$ and $Reg(T, \neg \cN_t^t)$.

For $Reg(T, E_{t,2}\bigcap \cN_{t}^t)$, under the event $\cN_{t}^t$, let $c_2=28$ be a constant and $K= \max_{S \in \cS} |\tilde{S}|$.
We set $j_{i}^{\max}=\lceil \log_2\frac{4B_1 c_2 K}{\Delta_{i}^{\min}}\rceil+1$.
We first define a regret allocation function 
\begin{equation}
\kappa_{i,j,T}(\ell) = \begin{cases}
\Delta_{i}^{\max}, &\text{if $0\le\ell\le L_{i,j,T,1} $ and $j \le j_i^{\max}$}\\
\frac{24c_2 B_1 \log T}{\ell}, &\text{if $L_{i,j,T,1} < \ell \le L_{i,j,T,2}$ and $j \le j_i^{\max}$}\\
0, &\text{if $\ell > L_{i,j,T,2} + 1$ or $ j > j_i^{\max}$,}
\end{cases}
\end{equation}
where $L_{i,j,T,1}=\frac{24c_2 B_1\log T}{\Delta_{i}^{\max}}$, $L_{i,j,T,2}=\frac{24c_2 B_1 K\log T}{\Delta_i^{\min}}$.
\begin{lemma}
For any time $t \in [T]$, if $N_{t}^t$ and $E_{t,2}$ hold, we have 
\begin{equation}\label{apdx_eq:kappa_e2}
    \Delta_{S_t} \le \sum_{i \in \tilde{S}_t} \kappa_{i,j_i^{S_t}, T} (N_{t-1, i,j_i^{S_t}}),
\end{equation}
where $j_{i}^{S_t}$ is the index of the triggering group $S_{i,j}$ such that $2^{-j_{i}^{S_t}}< p_{i}^{D,S_t} \le 2^{-j_{i}^{S_t}+1}$.
\end{lemma}
\begin{proof}
By event $E_{t,2}$, we have 
\begin{align}
     \Delta_{S_t} &\le \sum_{i \in \tilde{S}_t} 2c_2 B_1 p_{i}^{D,S_t} \min\{\frac{\log t}{T_{i,t-1}}, \frac{1}{28}\} \label{eq:ra_def_e2}\\
    &\le -\Delta_{S_t} + 2\sum_{i \in \tilde{S}_t} 2c_2 B_1 p_{i}^{D,S_t} \min\{\frac{\log t}{T_{i,t-1}}, \frac{1}{28}\} \label{eq:ra_double_e2}\\
     &\le \sum_{i \in \tilde{S}_t}\left( 4c_2 B_1 p_{i}^{D,S_t} \min\{\frac{\log t}{T_{i,t-1}}, \frac{1}{28}\} -\frac{\Delta_{S_t}}{|\tilde{S}_t|}\right)\label{eq:ra_def_reverse_e2}\\
    &\le  \sum_{i \in \tilde{S}_t}\left(4c_2 B_1 p_{i}^{D,S_t}\min\{\frac{\log t}{\frac{1}{3}N_{t-1, i, j_i^{S_t}}2^{-j_i^{S_t}}}, \frac{1}{28}\} -\frac{\Delta_{S_t}}{|\tilde{S}_t|}\right) \label{eq:ra_inf_def_e2}\\
    &\le  \sum_{i \in \tilde{S}_t}\underbrace{\left(4c_2 B_1 2^{-j_i^{S_t}+1}\min\{\frac{\log t}{\frac{1}{3}N_{t-1, i, j_i^{S_t}}2^{-j_i^{S_t}}}, \frac{1}{28}\} -\frac{\Delta_{S_t}}{K}\right)}_{(\ref{eq:ra_inf_prob_e2}, i)} \label{eq:ra_inf_prob_e2},
\end{align}
where \cref{eq:ra_def_e2} is by the definition of $E_{t,1}$ which says $\Delta_{S_t} \le \sum_{i \in \tilde{S}_t} 2c_2 B_1 p_{i}^{D,S_t} \min\{\frac{\log t}{T_{i,t-1}}, \frac{1}{28}\}$ and by dividing both sides by $\Delta_{S_t}>0$,
\cref{eq:ra_double_e2} is because we double the LHS and RHS of \cref{eq:ra_def_e2} at the same time and then put one into the RHS,
\cref{eq:ra_def_reverse_e2} is by putting $-\Delta_{S_t}$ inside the summation,
\cref{eq:ra_inf_def_e2} is due to the same reason of \cref{eq:inf_def_e1} under event $\cN_{t}^t$,
\cref{eq:ra_inf_prob_e2} is due to $p_i^{D,S_t}\le 2^{-j_{i}^{S_t}+1}$ given by the definition of $j_i^{S_t}$ and $|\tilde{S}|\le K$.

Similar to \cref{eq:ra_inf_prob_e2}, \cref{eq:ra_double_e2} is called the reverse amortization. Now we bound (\ref{eq:ra_inf_prob_e2}, i) under different cases.

\textbf{When $j >j_i^{\max}$}, 

we have $(\ref{eq:ra_inf_prob_e2}, i)\le 4c_2 B_1 2^{-j_i^{S_t}+1} -\frac{\Delta_{S_t}}{K} \le 4c_2 B_1 \frac{\Delta_i^{\min}}{c_2 B_1 K} -\frac{\Delta_{S_t}}{K}\le \frac{\Delta_{i}^{\min}}{K}\frac{4}{28}-\frac{\Delta_{S_t}}{K}\le 0 =\kappa_{i,j_i^{S_t}, T} (N_{t-1, i,j_i^{S_t}})$.

\textbf{When $N_{t-1, i,j_i^{S_t}} > L_{i,j_i^{S_t}, T, 2}$,}

we have $(\ref{eq:ra_inf_prob_e2}, i)\le 4c_2 B_1 2^{-j_i^{S_t}+1} \frac{\log t}{\frac{1}{3}N_{t-1, i, j_i^{S_t}}2^{-j_i^{S_t}}} -\frac{\Delta_{S_t}}{K} \le \frac{24c_2 B_1\log T}{N_{t-1, i, 
j_i^{S_t}}} -\frac{\Delta_{S_t}}{K} 
<\frac{\Delta_i^{\min}}{K}-\frac{\Delta_{S_t}}{K}\le  0 = \kappa_{i,j_i^{S_t}, T} (N_{t-1, i,j_i^{S_t}})$.

\textbf{When $N_{t-1, i,j_i^{S_t}} \le L_{i,j_i^{S_t}, T, 2}$ and $j \le j_i^{\max}$,}

We have $(\ref{eq:ra_inf_prob_e2}, i)\le 4c_2 B_1 2^{-j_i^{S_t}+1} \frac{\log t}{\frac{1}{3}N_{t-1, i, j_i^{S_t}}2^{-j_i^{S_t}}} -\frac{\Delta_{S_t}}{K} = \frac{24c_2 B_1\log T}{N_{t-1, i, j_i^{S_t}}} -\frac{\Delta_{S_t}}{K} <\frac{24c_2 B_1\log T}{N_{t-1, i, j_i^{S_t}}} = \kappa_{i,j_i^{S_t}, T} (N_{t-1, i,j_i^{S_t}})$.

\textbf{When $N_{t-1, i,j_i^{S_t}} \le L_{i,j_i^{S_t}, T, 1}$ and $j \le j_i^{\max}$,}

If there exists $i \in \tilde{S}_t$ so that $N_{t-1, i,j_i^{S_t}}\le  L_{i,j_i^{S_t}, T, 1}=$, then we know  $\sum_{q \in \tilde{S}_t} \kappa_{i,j_q^{S_t}, T} (N_{t-1, q,j_q^{S_t}}) \ge \kappa_{i,j_i^{S_t}, T} (N_{t-1, i,j_i^{S_t}})=\Delta_{i}^{\max} \ge \Delta_{S_t}$, which makes \cref{apdx_eq:kappa_e2} holds no matter what. This means we do not need to consider this case for good.

Combining all above cases, we have  $\Delta_{S_t} \le \sum_{i \in \tilde{S}_t} \kappa_{i,j_i^{S_t}, T} (N_{t-1, i,j_i^{S_t}})$.
\end{proof}

Since $N_{t,i, j_{i}^{S_t}}$ is increased if and only if $i \in \tilde{S_t}$ and consider all possible $i, j_i^{S}$ and $N_{t,i,j_{i}^{S}}$ where $\kappa_{i, j_{i}^{S}, T}( N_{t-1,i,j^{S}})>0$, we have
\begin{align*}
    &Reg(T, E_{t,2} \bigcap N_{t}^t) \\
    &\le \sum_{t\in [T]}\sum_{i \in \tilde{S}_t} \kappa_{i, j_{i}^{S_t}, T}(N_{t-1,i,j^{S_t}})\\
    &\le \sum_{i\in [m]}\sum_{j=0}^{j_i^{\max}} \sum_{\ell=1}^{L_{i,j,T,1}}\Delta_i^{\max} +  \sum_{i\in [m]}\sum_{j=1}^{j_i^{\max}} \sum_{L_{i,j,T,1}+1}^{L_{i,j,T,2}}\frac{24c_2 B_1 \log T}{\ell}\\
    &\le \sum_{i\in [m]}\sum_{j=1}^{j_i^{\max}} 24c_2B_1\log T + \sum_{i\in [m]}\sum_{j=1}^{j_i^{\max}} 24c_2 B_1 \log (\frac{K\Delta_i^{\max}}{\Delta_{i}^{\min}}
    ) \log T\\
    &=\sum_{i\in [m]}\sum_{j=1}^{j_i^{\max}} 24c_2 B_1 \left(1+\log (\frac{K\Delta_i^{\max}}{\Delta_{i}^{\min}})\right)\log T\\
    &\le \sum_{i\in [m]} 24 c_2 B_1 \left( \log_2\frac{B_1 c_2 K}{\Delta_{i}^{\min}}\right) \left(1+\log (\frac{K\Delta_i^{\max}}{\Delta_{i}^{\min}})\right)\log T
\end{align*}

Similar to \cref{apdx_eq:additional_Ntt}, with additional $Reg(T, \neg\cN_t^t)$, we have

We have $Reg(T, E_{t,2}) \le \sum_{i\in [m]} 24 c_2 B_1 \left( \log_2\frac{B_1 c_2 K}{\Delta_{i}^{\min}}\right) \left(1+\log (\frac{K\Delta_i^{\max}}{\Delta_{i}^{\min}})\right)\log T + \frac{m\pi^2}{6}\log_2\frac{4B_1 c_2 K}{\Delta_{i}^{\min}} \Delta_{\max}$.

\subsection{The Proof Following~\cite{merlis2019batch} Using Infinitely Many Events With an Additional Factor of $O(\log K)$}\label{apdx_sec:inf_many_events}
Now we can separately bound these two event-filtered regrets. 
Recall that $\tilde{S}_t=\{i \in [m]: p_{i}^{D, S_t} >0\}$ is the set of arms that could be triggered in round $t$. 
Let $K=\max_{t \in [T]}|\tilde{S}_t|$ be the maximum number of base arms that can be triggered in any rounds.
In round $t$, given base arm $i$ and action $S_t$, we denote $j_i^{S_t}$ to be the corresponding index of the triggering group $\cS_{i,j}$ so that $2^{-j_{i}^{S_t}}< p_{i}^{D,S_t} \le 2^{-j_{i}^{S_t}+1}$.
Our strategy is to find (perhaps infinitely many) events that must happen when $E_{t,1}$ (or $E_{t,2}$) happens. 
Then we can show that the number of times these events can happen are bounded or otherwise $E_{t,1}$ (or $E_{t,2}$) will not hold anymore.

\subsubsection{Upper bound for $Reg(T, E_{t,1})$}\label{sec:upper_bound_Et1_ifm}
To upper bound $Reg(T, E_{t,1})$, we bound it by $Reg(T, E_{t,1}) \le Reg(T, E_{t,1} \bigcap \cN_t^{t})+ Reg(T, \neg \cN_t^{t})$. In the following, we will consider to first bound $Reg(T, E_{t,1}\bigcap N_t^t)$ and then $Reg(T, \neg \cN_t^t)$.

Recall that $e_{t,1}(S_t)=4\sqrt{3}B_v\sqrt{\sum_{i\in \tilde{S}_t}(\frac{\log t}{T_{t-1,i}}\wedge \frac{1}{28})(p_{i}^{D,S_t})^\lambda}$.
Let $c_1=4\sqrt{3}$ be a constant and $O_{t}=\{i \in \tilde{S}_t: j_i^{S_t} \le j_{i}^{\max}\}$ be the set of base arms whose triggering probabilities are not too small, where the threshold $j_i^{\max}=\frac{1}{\lambda}(\lceil \log_2\frac{c_1^2B_v^2 K}{(\Delta_{i}^{\min})^2}\rceil+1)$.
Let $\alpha_{1} > \alpha_{2} > ... > \alpha_{k} > ... > \alpha_{\infty}$ and $1=\beta_{0} > \beta_{1} > ... > \beta_{k} > ... > \beta_{\infty}$ be two infinite sequences of positive numbers that are decreasing and converge to $0$, which will be used later to define specific set of base arms $(A_{t,k})_{k=1}^{\infty}$ and events $(\G_{t,k})_{k=1}^{\infty}$.

For positive integers $k$ and $t$, we define $A_{t,k}=\{i \in \tilde{S}_t: N_{t-1, i,j_i^{S_t}} \le \alpha_{k} \frac{g(K, \Delta_{S_t})f(t)}{\Delta_{S_t}^2}, j_i^{S_t} \le j_{i}^{\max}\}=\{i \in \tilde{S}_t \cap O_{S_{t}}: N_{t-1, i,j_i^{S_t}} \le \alpha_{k} \frac{g(K, \Delta_{S_t})f(t)}{\Delta_{S_t}^2}  \}$, which is the set of arms in $\tilde{S}_t$ that are counted less that a threshold and whose triggering probabilities are not too small, where $g(K, \Delta_{S_t})$ and $f(t)$ are going to be tuned for later use.
Moreover, we define the complementary set $\bar{A}_{t,k}=\{i \in \tilde{S}_t\cap O_{t}: N_{t-1, i,j_i^{S_t}} > \alpha_{k} \frac{g(K, \Delta_{S_t})f(t)}{\Delta_{S_t}^2}\}$.

Now we are ready to define the events $\G_{t,k}=\{|A_{t,k}|\ge \beta_{k} K; \forall h < k, |A_{t,h}| < \beta_{h} K\}$. 
Note that $\G_{t,k}$ is true when at least $\beta_{k} K$ arms triggered are in the set $A_{t,k}$ but less than $\beta_{h} K$ arms triggered are in the set $A_{t,h}$ for $h<k$. 
Let $\G_t = \bigcup_{k=1}^{\infty} \G_{t,k}$ and by definition its complementary $\overline{\G}_t=\{|A_{t,k}|<\beta_{k} K, \forall k \ge 1\}$. 
We first introduce a lemma saying that if there exists $k_0>0$ such that $\beta_{k_0}$ is smaller than $1/K$, we can safely use finite many events to conclude infinitely many events.
\begin{lemma}\label{lem:finite_A}
If there exists $k_{0}$ such that $\beta_{k_{0}} \le 1/K$, then $\G_t=\bigcup_{k=1}^{k_0} \G_{t,k}$ and $\overline{\G}_t=\{|A_{t,k}|<\beta_{k} K, \forall 1 \le k \le k_0\}$.
\end{lemma}
\begin{proof}
Let $k_0$ such that $\beta_{k_0} \le 1/K$. THen for all $k > k_0$, $\G_{t,k}=\{|A_{t,k}|\ge 1; \forall h < k_0, |A_{t,h}|\le K\beta_{k}; \forall k_0\le h < k, |A_{t,h}|=0\}$. But as the sequence of sets $A_{t,k}$ is decreasing, $\{|A_{t,k_0}|= 0\}$ and $\{|A_{t,k}|\ge 1\}$ cannot happen at the same time. Thus, $|A_{t,k}|$ cannot happen for $k > k_0$. 
\end{proof}

Now we have the following lemma showing an upper bound of $e_{t,2}(S_t)$ when $\overline{\G}_t$ and $N^t_{t}$ happens.
\begin{lemma}\label{lem:bound_e1}
Under the event $\overline{\G}_t$ and $N^t_{t}$ and if $\exists \, k_0$ such that $\beta_{k_0} \le 1/K$, then 
\begin{equation}
    (e_{t,1}(S_t))^2 <\frac{6c_1^2B_v^22^{(-j_i^{S_t}+1)(\lambda-1)}\log t  \Delta_{S_t}^2 K}{g(K,\Delta_{S_t}) f(t)}  (\sum_{k=1}^{k_0}\frac{\beta_{k-1}-\beta_k}{\alpha_k}+\frac{\beta_{k_0}}{\alpha_{k_0}})+ \frac{\Delta^2_{S_t}}{8}
\end{equation}
\end{lemma}
\begin{proof}
\begin{align}
    (e_{t,1}(S_t))^2 &= \sum_{i \in \tilde{S}_t} c_1^2 B_v^2 (p_{i}^{D,S_t})^\lambda \min\{\frac{\log t}{T_{i,t-1}}, \frac{1}{28}\} \label{eq:def_e1}\\
    &\le  \sum_{i \in \tilde{S}_t}c_1^2 B_v^2 (p_{i}^{D,S_t})^\lambda\min\{\frac{\log t}{\frac{1}{3}N_{t-1, i, j_i^{S_t}}2^{-j_i^{S_t}}}, \frac{1}{28}\} \label{eq:inf_def_e1}\\
    &\le \sum_{i \in \tilde{S}_t \cap O_{S_t}}c_1^2 B_v^2 (2^{-j_i^{S_t}+1})^\lambda  \frac{\log t}{\frac{1}{3}N_{t-1,i,j_i^{S_t}}2^{-j_i^{S_t}}} +\frac{1}{28}\sum_{i \in \tilde{S}_t \cap \bar{O}_{S_t} }c_1^2B_v^2(2^{-j_i^{\max}+1})^\lambda \label{eq:decompose_e1}\\
    &\le \sum_{i \in \tilde{S}_t\cap O_{S_t}}   \frac{6c_1^2B_v^22^{(-j_i^{S_t}+1)(\lambda-1)}\log t }{N_{t-1,i,j_i^{S_t}}} + \sum_{i \in \tilde{S}_t \cap \bar{O}_{S_t} } B_v^2c_1^2\frac{(\Delta_{i}^{\min})^2}{8B_v^2 c_1^2 K}\label{eq_e1:def_jimax}\\
    &\le \sum_{k=1}^{k_0}\sum_{i\in \bar{A}_{t,k}\backslash \bar{A}_{t,k-1}}\frac{6c_1^2B_v^22^{(-j_i^{S_t}+1)(\lambda-1)}\log t }{N_{t-1,i,j_i^{S_t}}} + \frac{\Delta_{S_t}^2}{8}\label{line:inf_k0_e1}\\
    &< \sum_{k=1}^{k_0}\frac{6c_1^2B_v^22^{(-j_i^{S_t}+1)(\lambda-1)}\log t  \Delta_{S_t}^2 |\bar{A}_{t,k}\backslash \bar{A}_{t,k-1}|}{\alpha_k g(K,\Delta_{S_t}) f(t)}+ \frac{\Delta^2_{S_t}}{8}\label{line:inf_def_barAK_e1}\\
    &<\frac{6c_1^2B_v^22^{(-j_i^{S_t}+1)(\lambda-1)}\log t  \Delta_{S_t}^2 K}{g(K,\Delta_{S_t}) f(t)}  (\sum_{k=1}^{k_0}\frac{\beta_{k-1}-\beta_k}{\alpha_k}+\frac{\beta_{k_0}}{\alpha_{k_0}})+ \frac{\Delta^2_{S_t}}{8}\label{line:inf_pe_e1}
\end{align}
where \Cref{eq:def_e1} is by definition, \Cref{eq:inf_def_e1} holds because if $\frac{ \ln t}{\frac{1}{3}N_{i,j,t-1}2^{-j}} > 1/28$, then $\min\{\frac{\log t}{\frac{1}{3}N_{t-1, i, j_i^{S_t}}2^{-j_i^{S_t}}}, 1/28\} = 1/28$ and thus larger than $\min\{\frac{\log t}{T_{t-1,i}}, 1/28\}$, else we have $\frac{6\log t}{\frac{1}{3}N_{t-1, i, j_i^{S_t}}2^{-j_i^{S_t}}}<6/28<1$ and by  $\mathcal{N}_{t}^{t}$ we have $T_{t-1,i}\ge\frac{1}{3}N_{t-1,i, j^{S_t}} \cdot 2^{-j^{S_t}}$ and thus $\min\{\frac{\log t}{T_{t-1,i}}, 1/28\} < \frac{\log t}{\frac{1}{3}N_{t-1, i, j_i^{S_t}}2^{-j_i^{S_t}}} $, \cref{eq:decompose_e1} is by considering $O_{t}$ and $\bar{O}_{t}$ , \Cref{eq_e1:def_jimax} is due to definition of $j_i^{\max}$, \Cref{line:inf_k0_e1} is by setting $k_0$ be the largest number that $\beta_{k_0} \le 1/K$, \Cref{line:inf_def_barAK_e1} is by definition of $\bar{A}_{t,k}$, \Cref{line:inf_pe_e1} is due to the similar reason of Lemma 8 from \cite{degenne2016combinatorial}.
\end{proof}

Now we set $g(K,\Delta_{S_t})=2^{(-j_i^{S_t}+1)(\lambda-1)} K l$, where $l=\sum_{k=1}^{k_0}\frac{\beta_{k-1}-\beta_k}{\alpha_k}+\frac{\beta_{k_0}}{\alpha_{k_0}}$ and $f(t)=48c_1^2 B_v^2 \log t$.
By \Cref{lem:bound_e1}, we can show that $\Delta_{S_t} > 2e_{t,1} $ under event $\overline{\G}_t \bigcap \cN_{t}^t$.
In other words, under event $\cN_{t}^t$, if $E_{t,1}$ holds, then $\G_t$ must hold.

For any arm $i$, let arm related event $\G_{t,k,i}=\G_{t,k} \bigcap \{i \in \tilde{S}_t, N_{t-1, i,j_i^{S_t}} \le \alpha_k \frac{g(K, \Delta_{S_t})f(t)}{\Delta_{S_t}^2}, j_i^{S_t} \le j_i^{\max}\}$.
When $\G_{t,k}$ happens, we have $\I\{\G_{t,k}\}\le \frac{1}{\beta_k K}\sum_{i\in [m]}\{\G_{t,k,i}\}$.
We consider two cases when $\lambda>1$ and when $\lambda =1$.

\textbf{Case 1: When $\lambda>1$}

The $Reg(T, E_{t,1} \bigcap \cN_{t}^t)$ is bounded by,
\begin{align}
&Reg(T, E_{t,1} \bigcap\cN_{t}^t)\le \sum_{t=1}^T\sum_{k=1}^{k_0}\Delta_{S_t} \I\{\G_{t,k}\}\label{apdx_eq:inf_case1_1}\\
&\le \sum_{t=1}^T\sum_{k=1}^{k_0}\sum_{i=1}^m\frac{\Delta_{S_t}}{K \beta_k} \I\{\G_{t,k,i}\}\label{apdx_eq:inf_case1_2}\\
&\le \sum_{i=1}^m\sum_{k=1}^{k_0}\frac{1}{K \beta_k}\sum_{t=1}^T\Delta_{S_t} \I\{i \in \tilde{S}_t, N_{i,j_i^{S_t}, t-1} \le \frac{\theta_k 2^{(-j_i^{S_t}+1)(\lambda-1)}}{\Delta_{S_t}^2}, j_i^{S_t} \le j_i^{\max}\}\notag\\ &(\text{with } \theta_k=\alpha_k K l f(t))\notag\\
&\le \sum_{i=1}^m\sum_{j=1}^{\infty}\sum_{k=1}^{k_0}\frac{1}{K \beta_k}\sum_{t=1}^T\Delta_{S_t} \I\{i \in \tilde{S}_t, N_{i,j_i^{S_t}, t-1} \le \frac{\theta_k 2^{(-j+1)(\lambda-1)}}{\Delta_{S_t}^2}, j_i^{S_t}=j\} \quad\quad \label{apdx_eq:inf_case1_3}\\
&\le \sum_{i=1}^m\sum_{j=1}^{\infty}\sum_{k=1}^{k_0}\frac{1}{K \beta_k}\sum_{t=1}^T \sum_{n=1}^{D_i}\Delta_{i,n} \I\{i \in \tilde{S}_t, N_{i,j_i^{S_t}, t-1} \le \frac{\theta_k2^{(-j+1)(\lambda-1)}}{\Delta_{i,n}^2}, \Delta_{S_t}=\Delta_{i,n} , j_i^{S_t}=j\} \notag\\
&(\text{with } \Delta_{i,1}\ge \Delta_{i,2} \ge ... \ge \Delta_{i, D_i})\label{apdx_eq:inf_case1_4}\\
&\le \sum_{i=1}^m\sum_{j=1}^{\infty}\sum_{k=1}^{k_0}\frac{1}{K \beta_k}\sum_{t=1}^T \sum_{n=1}^{D_i}\sum_{p=1}^{n}\Delta_{i,p} \I\{i \in \tilde{S}_t, N_{i,j_i^{S_t}, t-1} \in (\frac{\theta_k2^{(-j+1)(\lambda-1)}}{\Delta_{i,p-1}^2}, \frac{\theta_k2^{(-j+1)(\lambda-1)}}{\Delta_{i,p}^2}],\notag\\ &\Delta_{S_t}=\Delta_{i,n} ,j_i^{S_t}=j\}\label{apdx_eq:inf_case1_5}\\
&\le \sum_{i=1}^m\sum_{j=1}^{\infty}\sum_{k=1}^{k_0}\frac{1}{K \beta_k}\sum_{t=1}^T \sum_{n=1}^{D_i}\sum_{p=1}^{D_i}\Delta_{i,p} \I\{i \in \tilde{S}_t, N_{i,j_i^{S_t}, t-1} \in (\frac{\theta_k2^{(-j+1)(\lambda-1)}}{\Delta_{i,p-1}^2}, \frac{\theta_k2^{(-j+1)(\lambda-1)}}{\Delta_{i,p}^2}],\notag\\ 
& \Delta_{S_t}=\Delta_{i,n} ,j_i^{S_t}=j\}\label{apdx_eq:inf_case1_6}\\
&\le \sum_{i=1}^m\sum_{j=1}^{\infty}\sum_{k=1}^{k_0}\frac{1}{K \beta_k}\sum_{t=1}^T \sum_{p=1}^{D_i}\Delta_{i,p} \I\{i \in \tilde{S}_t, N_{i,j_i^{S_t}, t-1} \in (\frac{\theta_k2^{(-j+1)(\lambda-1)}}{\Delta_{i,p-1}^2}, \frac{\theta_k2^{(-j+1)(\lambda-1)}}{\Delta_{i,p}^2}] \notag\\
&, \Delta_{S_t}=\Delta_{i,n} , \Delta_{S_t}>0, j_i^{S_t}=j\}\label{apdx_eq:inf_case1_7}\\
&\le \sum_{i=1}^m\sum_{j=1}^{\infty}\sum_{k=1}^{k_0}\frac{2^{(-j+1)(\lambda-1)}}{K \beta_k}
(\frac{\theta_k}{\Delta_{i,1}} + \theta_k \sum_{p=2}^{D_i}\Delta_{i,p}(\frac{1}{\Delta_{i,p}^2}-\frac{1}{\Delta_{i,p-1}^2}))\label{apdx_eq:inf_case1_8}\\
&= \sum_{i=1}^m\sum_{j=1}^{\infty}\sum_{k=1}^{k_0}\frac{2^{(-j+1)(\lambda-1)}}{K \beta_k}
(\frac{\theta_k}{\Delta_{i,D_i}}  + \theta_k \sum_{p=1}^{D_i-1}\frac{\Delta_{i,p}-\Delta_{i,p+1}}{\Delta_{i,p}^2})\label{apdx_eq:inf_case1_9}\\
&\le \sum_{i=1}^m\sum_{j=1}^{\infty}\sum_{k=1}^{k_0}\frac{2^{(-j+1)(\lambda-1)}}{K \beta_k}
(\frac{\theta_k}{\Delta_{i,D_i}} + \theta_k \int_{\Delta_{i,D_i}}^{\Delta_{i,1}}x^{-2}dx)\\
&\le \sum_{i=1}^m\sum_{j=1}^{\infty}\sum_{k=1}^{k_0}\frac{2\theta_k 2^{(-j+1)(\lambda-1)}}{K \beta_k \Delta_{i, D_i}} \label{ieq:lamda_adapt}\\
&\le \sum_{i=1}^m(\sum_{k=1}^{k_0}  \frac{2}{1-2^{-(\lambda-1)}} \frac{1}{\Delta_{i}^{\min}} \frac{\alpha_k}{\beta_k} l f(T))\label{apdx_eq:inf_case1_10}\\
&\le  \sum_{i=1}^m\left(\frac{480 c_1^2 B_v^2}{ (1-2^{-(\lambda-1)})}\right)\ceil{\frac{\log K}{1.61}}^2  \frac{\log T}{\Delta_{i}^{\min}}\label{apdx_eq:inf_case1_11}
\end{align}
where 
\cref{apdx_eq:inf_case1_1} is because  under event $\cN_{t}^t$, if $E_{t,1}$ holds then $\G_t$ must hold, 
\cref{apdx_eq:inf_case1_2} is because $\I\{\G_{t,k}\}\le \frac{1}{\beta_k K}\sum_{i\in [m]}\{\G_{t,k,i}\}$,  
\cref{apdx_eq:inf_case1_3} is by applying union bound over $j_i^{S_t}=1, ..., j_i^{\max}$,
\cref{apdx_eq:inf_case1_4} is by considering $D_i$ gaps for $\Delta_{S_t}$ and applying union bounds, 
\cref{apdx_eq:inf_case1_5} is by dividing $ N_{i,j_i^{S_t}, t-1} \le \frac{\theta_k2^{(-j+1)(\lambda-1)}}{\Delta_{i,n}^2}$ into non-overlapping sub-intervals,
\cref{apdx_eq:inf_case1_6} is by extending summation over $p$ to $D_i$,
\cref{apdx_eq:inf_case1_7} is by replacing summation over $n=1, ..., D_i$ to $\Delta S_t > 0$,
\cref{apdx_eq:inf_case1_8} is to bound the number of times the event happen to the length of interval,
\cref{apdx_eq:inf_case1_9} to \cref{apdx_eq:inf_case1_10} are math calculation by replacing summation by integrals,
\cref{apdx_eq:inf_case1_11} is similar to \cite[Lemma 11, Appendix C]{degenne2016combinatorial} by setting
$\alpha_k=\beta_k=0.2^k$ and $\sum_{k=1}^{k0} \frac{\alpha_k}{\beta_k}l \le 5 \ceil{\frac{\log K}{1.61}}^2$.

\textbf{Case 2: When $\lambda=1$}
The only difference is we have to sum over $j=1$ to $j=j^i_{\max}$ in \cref{ieq:lamda_adapt}, instead of $\infty$, so that we replace $\frac{1}{ (1-2^{-(\lambda-1)})}$ to $j^i_{\max}=\log_2\left(\frac{c_1^2B_v^2 K}{(\Delta_{i}^{\min})^2}\right)$.
we can bound the following inequality
\begin{align}
&Reg(T, E_{t,1} \wedge \cN_{t}^t)\le \sum_{i=1}^m\left(480 c_1^2 B_v^2\right)\log_2\left(\frac{c_1^2B_v^2 K}{(\Delta_{i}^{\min})^2}\right)\ceil{\frac{\log K}{1.61}}^2  \frac{\log T}{\Delta_{i}^{\min}}
\end{align}

Now for $Reg(T,\neg \cN_{t}^t)$, by \cref{apdx_lem:prob_nice_triggering},
\begin{align}\label{apdx_eq:additional_Ntt}
    Reg(T,\neg \cN_t^t) &\le \sum_{t=1}^T \sum_{i \in [m]}j_i^{\max}t^{-2} \Delta_{\max} \\
    &\le \frac{m\pi^2}{6}\log_2\left(\frac{c_1^2B_v^2 K}{(\Delta_{\min})^2}\right) \Delta_{\max}
\end{align}
So we have when $\lambda>1$,
\begin{align}
    Reg(T, E_{t,1})\le\sum_{i=1}^m\left(\frac{480 c_1^2 B_v^2}{ (1-2^{-(\lambda-1)})}\right)\ceil{\frac{\log K}{1.61}}^2  \frac{\log T}{\Delta_{i}^{\min}}+\frac{m\pi^2}{6}\log_2\left(\frac{c_1^2B_v^2 K}{\lambda(\Delta_{\min})^2}\right) \Delta_{\max};
\end{align}
and when $\lambda=1$,
\begin{align}
    Reg(T, E_{t,1})\le \sum_{i=1}^m\left(480 c_1^2 B_v^2\right)\log_2\left(\frac{c_1^2B_v^2 K}{(\Delta_{i}^{\min})^2}\right)\ceil{\frac{\log K}{1.61}}^2  \frac{\log T}{\Delta_{i}^{\min}} + \frac{m\pi^2}{6}\log_2\left(\frac{c_1^2B_v^2 K}{(\Delta_{\min})^2}\right) \Delta_{\max}
\end{align}

\subsubsection{Upper bound for $Reg(T, E_{t,2})$}\label{sec:upper_bound_Et2_ifm}
Let $c_2=28$ be a constant and $O_{t}=\{i \in \tilde{S}_t: j_i^{S_t} \le j_{i}^{\max}\}$ be the set of base arms whose triggering probabilities are not too small, where the threshold $j_{i}^{\max}=\lceil \log_2\frac{4B_1 c_2 K}{\Delta_{i}^{\min}}\rceil+1$.
Let $\alpha_{1} > \alpha_{2} > ... > \alpha_{k} > ... > \alpha_{\infty}$ and $1=\beta_{0} > \beta_{1} > ... > \beta_{k} > ... > \beta_{\infty}$ be two infinite sequences of positive numbers that are decreasing and converge to $0$, which will be used later to define specific set of base arms $(A_{t,k})_{k=1}^{\infty}$ and events $(\G_{t,k})_{k=1}^{\infty}$.

For positive integers $k$ and $t$, we define $A_{t,k}=\{i \in \tilde{S}_t: N_{t-1, i,j_i^{S_t}} \le \alpha_{k} \frac{g(K, \Delta_{S_t})f(t)}{\Delta_{S_t}^2}, j_i^{S_t} \le j_{i}^{\max}\}=\{i \in \tilde{S}_t \cap O_{S_{t}}: N_{t-1, i,j_i^{S_t}} \le \alpha_{k} \frac{g(K, \Delta_{S_t})f(t)}{\Delta_{S_t}^2}  \}$, which is the set of arms in $\tilde{S}_t$ that are counted less that a threshold and whose triggering probabilities are not too small, where $g(K, \Delta_{S_t})$ and $f(t)$ are going to be tuned for later use.
Moreover, we define the complementary set $\bar{A}_{t,k}=\{i \in \tilde{S}_t\cap O_{t}: N_{t-1, i,j_i^{S_t}} > \alpha_{k} \frac{g(K, \Delta_{S_t})f(t)}{\Delta_{S_t}^2}\}$.

Now we are ready to define the events $\G_{t,k}=\{|A_{t,k}|\ge \beta_{k} K; \forall h < k, |A_{t,h}| < \beta_{h} K\}$. 
Note that $\G_{t,k}$ is true when at least $\beta_{k} K$ arms triggered are in the set $A_{t,k}$ but less than $\beta_{h} K$ arms triggered are in the set $A_{t,h}$ for $h<k$. 
Let $\G_t = \bigcup_{k=1}^{\infty} \G_{t,k}$ and by definition its complementary $\overline{\G}_t=\{|A_{t,k}|<\beta_{k} K, \forall k \ge 1\}$. 
We first introduce a lemma saying that if there exists $k_0>0$ such that $\beta_{k_0}$ is smaller than $1/K$, we can safely use finite many events to conclude infinitely many events.
\begin{lemma}
If there exists $k_{0}$ such that $\beta_{k_{0}} \le 1/K$, then $\G_t=\bigcup_{k=1}^{k_0} \G_{t,k}$ and $\overline{\G}_t=\{|A_{t,k}|<\beta_{k} K, \forall 1 \le k \le k_0\}$.
\end{lemma}
\begin{proof}
By the same argument as \cref{lem:finite_A}, the lemma is proved.
\end{proof}

Now we have the following lemma showing an upper bound of $e_{t,2}(S_t)$ when $\overline{\G}_t$ and $N^t_{t}$ happens.
\begin{lemma}\label{lem:bound_e2}
Under the event $\overline{\G}_t$ and $N^t_{t}$ and if $\exists \, k_0$ such that $\beta_{k_0} \le 1/K$, then 
\begin{equation}
    e_{t,2}(S_t) <\frac{6c_2B_1 \log t \Delta_{S_t}^2 K}{g(K,\Delta_{S_t}) f(t)}  (\sum_{k=1}^{k_0}\frac{\beta_{k-1}-\beta_k}{\alpha_k}+\frac{\beta_{k_0}}{\alpha_{k_0}})+ \frac{\Delta_{S_t}}{4}
\end{equation}
\end{lemma}
\begin{proof}
\begin{align}
    e_{t,2}(S_t) &= \sum_{i \in \tilde{S}_t} c_2 B_1 p_{i}^{D,S_t} \min\{\frac{\log t}{T_{t-1,i}}, \frac{1}{28}\} \label{line:inf_def}\\
     &\le  \sum_{i \in \tilde{S}_t}c_2 B_1  p_{i}^{D,S_t}\min\{\frac{\log t}{\frac{1}{3}N_{t-1, i, j_i^{S_t}}2^{-j_i^{S_t}}}, \frac{1}{28}\} \label{line:inf_def_2}\\
    &\le \sum_{i \in \tilde{S}_t \cap O_{t}}c_2 B_1  2^{-j_i^{S_t}+1}  \frac{\log t}{\frac{1}{3}N_{t-1, i, j_i^{S_t}}2^{-j_i^{S_t}}} +\frac{1}{28}\sum_{i \in \tilde{S}_t \cap \bar{O}_{t} }B_1 c_2 2^{-j_i^{\max}+1} \label{line:inf_lamma_4}\\
    &\le \sum_{i \in \tilde{S}_t\cap O_{t}}   \frac{6c_2B_1\log t}{N_{t-1, i, j_i^{S_t}}} + \sum_{i \in \tilde{S}_t \cap \bar{O}_{t} } B_1 c_2\frac{\Delta_{i}^{\min}}{4B_1 c_2 K}\label{line:inf_jimax}\\
    &\le \sum_{k=1}^{k_0}\sum_{i\in \bar{A}_{t,k}\backslash \bar{A}_{t,k-1}}\frac{6c_2B_1 \log t}{N_{t-1, i, j_i^{S_t}}} + \frac{\Delta_{S_t}}{4}\label{line:inf_k0}\\
    &< \sum_{k=1}^{k_0}\frac{6c_2B_1 \log t \Delta_{S_t}^2 |\bar{A}_{t,k}\backslash \bar{A}_{t,k-1}|}{\alpha_k g(K,\Delta_{S_t}) f(t)}+ \frac{\Delta_{S_t}}{4}\label{line:inf_def_barAK}\\
    &<\frac{6c_2B_1 \log t \Delta_{S_t}^2 K}{g(K,\Delta_{S_t}) f(t)}  (\sum_{k=1}^{k_0}\frac{\beta_{k-1}-\beta_k}{\alpha_k}+\frac{\beta_{k_0}}{\alpha_{k_0}})+ \frac{\Delta_{S_t}}{4}\label{line:inf_pe}
\end{align}
where \Cref{line:inf_def} is by definition, \Cref{line:inf_def_2} holds because if $\frac{6 \ln t}{\frac{1}{3}N_{i,j,t-1}2^{-j}} > \frac{1}{28}$, then $\min\{\frac{6\log t}{\frac{1}{3}N_{t-1, i, j_i^{S_t}}2^{-j_i^{S_t}}}, 1\} = \frac{1}{28}$ and thus larger than $\min\{\frac{\log t}{T_{t-1,i}}, 1\}$, else we have $\frac{6\log t}{\frac{1}{3}N_{t-1, i, j_i^{S_t}}2^{-j_i^{S_t}}}<6/28<1$ and by  $\mathcal{N}_{t}^{t}$ we have $T_{t-1,i}\ge\frac{1}{3}N_{t-1,i, j^{S_t}} \cdot 2^{-j^{S_t}}$ and thus $\min\{\frac{\log t}{T_{t-1,i}}, 1/28\} < \frac{\log t}{\frac{1}{3}N_{t-1, i, j_i^{S_t}}2^{-j_i^{S_t}}} $, \cref{line:inf_lamma_4} is by considering $O_{t}$ and $\bar{O}_{t}$ , \Cref{line:inf_jimax} is due to definition of $j_i^{\max}$, \Cref{line:inf_k0} is by setting $k_0$ be the largest number that $\beta_{k_0} \le 1/K$, \Cref{line:inf_def_barAK} is by definition of $\bar{A}_{t,k}$, \Cref{line:inf_pe} is due to the proof of Lemma 8 of \cite{degenne2016combinatorial}.
\end{proof}

Now we set $g(K,\Delta_{S_t})=K\Delta_{S_t} l$, where $l=\sum_{k=1}^{k_0}\frac{\beta_{k-1}-\beta_k}{\alpha_k}+\frac{\beta_{k_0}}{\alpha_{k_0}}$ and $f(t)=24c_2 B_1 \log T$.
By \Cref{lem:bound_e2}, we can show that $\Delta_{S_t} > 2e_{t,2} $ under event $\overline{\G}_t \bigcap \cN_{t}^t$.
In other words, under event $\cN_{t}^t$, if $E_{t,2}$ holds, then $G_t$ must hold.

For any arm $i$, let arm related event $\G_{t,k,i}=\G_{t,k} \bigcap \{i \in \tilde{S}_t, N_{t-1, i,j_i^{S_t}} \le \alpha_k \frac{g(K, \Delta_{S_t})f(t)}{\Delta_{S_t}^2}, j_i^{S_t} \le j_i^{\max}\}$.
When $\G_{t,k}$ happens, we have $\I\{\G_{t,k}\}\le \frac{1}{\beta_k K}\sum_{i\in [m]}\{\G_{t,k,i}\}$.
So the $Reg(T, E_{t,2} \bigcap \cN_{t}^t)$ is bounded by,
\begin{align}
&Reg(T, E_{t,2} \bigcap \cN_{t}^t)\le \sum_{t=1}^T\sum_{k=1}^{k_0}\Delta_{S_t} \I\{\G_{t,k}\}\label{apdx_eq:inf_case2_1}\\
&\le \sum_{t=1}^T\sum_{k=1}^{k_0}\sum_{i=1}^m\frac{\Delta_{S_t}}{K \beta_k} \I\{G_{t,k,i}\}\label{apdx_eq:inf_case2_2}\\
&\le \sum_{i=1}^m\sum_{k=1}^{k_0}\frac{1}{K \beta_k}\sum_{t=1}^T\Delta_{S_t} \I\{i \in \tilde{S}_t, N_{t-1, i,j_i^{S_t}} \le \frac{\theta_k}{\Delta_{S_t}}, j_i^{S_t} \le j_i^{\max}\} \quad\quad (\text{with }\theta_k=\alpha_k K l f(t))\notag\\
&\le \sum_{i=1}^m\sum_{j=1}^{j_i^{\max}}\sum_{k=1}^{k_0}\frac{1}{K \beta_k}\sum_{t=1}^T\Delta_{S_t} \I\{i \in \tilde{S}_t, N_{t-1, i,j_i^{S_t}} \le \frac{\theta_k}{\Delta_{S_t}}, j_i^{S_t}=j\} \label{apdx_eq:inf_case2_3}\\
&\le \sum_{i=1}^m\sum_{j=1}^{j_i^{\max}}\sum_{k=1}^{k_0}\frac{1}{K \beta_k}\sum_{t=1}^T \sum_{n=1}^{D_i}\Delta_{i,n} \I\{i \in \tilde{S}_t, N_{t-1, i,j_i^{S_t}} \le \frac{\theta_k}{\Delta_{i,n}}, \Delta_{S_t}=\Delta_{i,n} , j_i^{S_t}=j\}\notag\\ 
&(\text{with }\Delta_{i,1}\ge \Delta_{i,2} \ge ... \ge \Delta_{i, D_i})\label{apdx_eq:inf_case2_4}\\
&\le \sum_{i=1}^m\sum_{j=1}^{j_i^{\max}}\sum_{k=1}^{k_0}\frac{1}{K \beta_k}\sum_{t=1}^T \sum_{n=1}^{D_i}\sum_{p=1}^{n}\Delta_{i,p} \I\{i \in \tilde{S}_t, N_{t-1, i,j_i^{S_t}} \in (\frac{\theta_k}{\Delta_{i,p-1}}, \frac{\theta_k}{\Delta_{i,p}}], \Delta_{S_t}=\Delta_{i,n} ,j_i^{S_t}=j\}\label{apdx_eq:inf_case2_5}\\
&\le \sum_{i=1}^m\sum_{j=1}^{j_i^{\max}}\sum_{k=1}^{k_0}\frac{1}{K \beta_k}\sum_{t=1}^T \sum_{n=1}^{D_i}\sum_{p=1}^{D_i}\Delta_{i,p} \I\{i \in \tilde{S}_t, N_{t-1, i,j_i^{S_t}} \in (\frac{\theta_k}{\Delta_{i,p-1}}, \frac{\theta_k}{\Delta_{i,p}}], \Delta_{S_t}=\Delta_{i,n} ,j_i^{S_t}=j\}\label{apdx_eq:inf_case2_6}\\
&\le \sum_{i=1}^m\sum_{j=1}^{j_i^{\max}}\sum_{k=1}^{k_0}\frac{1}{K \beta_k}\sum_{t=1}^T \sum_{p=1}^{D_i}\Delta_{i,p} \I\{i \in \tilde{S}_t, N_{t-1, i,j_i^{S_t}} \in (\frac{\theta_k}{\Delta_{i,p-1}}, \frac{\theta_k}{\Delta_{i,p}}] , \Delta_{S_t}=\Delta_{i,n},\notag\\ 
& \Delta_{S_t}>0, j_i^{S_t}=j\}\label{apdx_eq:inf_case2_7}\\
&\le \sum_{i=1}^m\sum_{j=1}^{j_i^{\max}}\sum_{k=1}^{k_0}\frac{1}{K \beta_k}
(\theta_k + \theta_k \sum_{p=2}^{D_i}\Delta_{i,p}(\frac{1}{\Delta_{i,p}}-\frac{1}{\Delta_{i,p-1}}))\label{apdx_eq:inf_case2_8}\\
&= \sum_{i=1}^m\sum_{j=1}^{j_i^{\max}}\sum_{k=1}^{k_0}\frac{1}{K \beta_k}
(\theta_k + \theta_k \sum_{p=1}^{D_i-1}\frac{\Delta_{i,p}-\Delta_{i,p+1}}{\Delta_{i,p}})\label{apdx_eq:inf_case2_9}\\
&\le \sum_{i=1}^m\sum_{j=1}^{j_i^{\max}}\sum_{k=1}^{k_0}\frac{1}{K \beta_k}
(\theta_k + \theta_k \int_{\Delta_{i,D_i}}^{\Delta_{i,1}}x^{-1}dx)\\
&\le \sum_{i=1}^m\sum_{j=1}^{j_i^{\max}}\sum_{k=1}^{k_0}\frac{\theta_k}{K \beta_k} (1+\log \frac{\Delta_{i}^{\max}}{\Delta_{i}^{\min}})\\
&\le \sum_{i=1}^m\sum_{j=1}^{j_i^{\max}}(\sum_{k=1}^{k_0} \frac{\alpha_k}{\beta_k} l f(T))(1+\log \frac{\Delta_{i}^{\max}}{\Delta_{i}^{\min}}) \label{apdx_eq:inf_case2_10}\\
&\le 120c_2 B_1\sum_{i=1}^m\left(\log_2\frac{c_2B_1 K}{\Delta_{i}^{\min}}\right) \left(1+\log \frac{\Delta_{i}^{\max}}{\Delta_{i}^{\min}}\right)\lceil{\frac{\log K}{1.61}\rceil}^2 \log T\label{apdx_eq:inf_case2_11}
\end{align}
where
\cref{apdx_eq:inf_case2_1} is because  under event $\cN_{t}^t$, if $E_{t,1}$ holds then $\G_t$ must hold, 
\cref{apdx_eq:inf_case2_2} is because $\I\{\G_{t,k}\}\le \frac{1}{\beta_k K}\sum_{i\in [m]}\{\G_{t,k,i}\}$,  
\cref{apdx_eq:inf_case2_3} is by applying union bound over $j_i^{S_t}=1, ..., j_i^{\max}$,
\cref{apdx_eq:inf_case2_4} is by considering $D_i$ gaps for $\Delta_{S_t}$ and applying union bounds, 
\cref{apdx_eq:inf_case2_5} is by dividing $ N_{i,j_i^{S_t}, t-1} \le \frac{\theta_k2^{(-j+1)(\lambda-1)}}{\Delta_{i,n}^2}$ into non-overlapping sub-intervals,
\cref{apdx_eq:inf_case2_6} is by extending summation over $p$ to $D_i$,
\cref{apdx_eq:inf_case2_7} is by replacing summation over $n=1, ..., D_i$ to $\Delta S_t > 0$,
\cref{apdx_eq:inf_case2_8} is to bound the number of times the event happen to the length of interval,
\cref{apdx_eq:inf_case2_9} to \cref{apdx_eq:inf_case2_10} are math calculation by replacing summation by integrals,
\cref{apdx_eq:inf_case2_11} is similar to \cite[Lemma 11, Appendix C]{degenne2016combinatorial} by setting
$\alpha_k=\beta_k=0.2^k$ and $\sum_{k=1}^{k0} \frac{\alpha_k}{\beta_k}l \le 5 \ceil{\frac{\log K}{1.61}}^2$.

Similarly, consider $Reg(T,\neg \cN_t^t)\le \frac{m\pi^2}{6}\log_2\left(\frac{c_2B_1 K}{\Delta_{\min}}\right) \Delta_{\max}$

We have 
\begin{align}
    &Reg(T,E_{t,2})\le 120c_2 B_1\sum_{i=1}^m\left(\log_2\frac{c_2B_1 K}{\Delta_{i}^{\min}}\right) \left(1+\log \frac{\Delta_{i}^{\max}}{\Delta_{i}^{\min}}\right)\lceil{\frac{\log K}{1.61}\rceil}^2 \log T \notag\\ &
    + \frac{m\pi^2}{6}\log_2\left(\frac{c_2B_1 K}{\Delta_{\min}}\right) \Delta_{\max}
\end{align}

\subsection{Summary of Regret Upper Bounds and Discussions on Distribution-Independent Bounds and Lower Bounds}\label{apdx_sec:summary_reg}

\subsubsection{Analysis using the reverse amortization tricks (\cref{apdx_sec:reverse_amt}).}\label{apdx_sec:sum_ra}

When using the \textbf{improved} analysis in \cref{apdx_sec:reverse_amt}, by \cref{eq:e_t}, \cref{apdx_sec:improve_e1}, \cref{apdx_sec:improve_e2},
the total regret is bounded as follows

(1) if $\lambda>1$,
\begin{align}
    &Reg(T) \le Reg(T, E_{t,1})+Reg(T,E_{t,2})+ \frac{2\pi^2}{3}m\Delta_{\max}\notag\\
    &\le   \sum_{i=1}^m \frac{48c_1^2 B_v^2 \log T}{\Delta_{i}^{\min} } (3+ \log K)  +  \sum_{i\in [m]} 24 c_2 B_1 \left( \log_2\frac{B_1 c_2 K}{\Delta_{i}^{\min}}\right) \left(1+\log (\frac{K\Delta_i^{\max}}{\Delta_{i}^{\min}})\right)\log T\notag\\
    &+   \sum_{i\in [m]}\frac{2c_1^2 B_v^2 }{\Delta_i^{\min}} +   \frac{m\pi^2}{6}\log_2\left(\frac{c_1^2B_v^2 K}{\lambda(\Delta_{\min})^2}\right) \Delta_{\max} +  \frac{m\pi^2}{6}\log_2\frac{4B_1 c_2 K}{\Delta_{i}^{\min}} \Delta_{\max} +\frac{2\pi^2}{3}m\Delta_{\max} + \notag\\
    &\le O\left(\sum_{i \in [m]}\frac{B_v^2 \log K \log T}{\Delta_i^{\min}} + \sum_{i \in [m]}B_1\log^2\left(\frac{B_1K}{\Delta_i^{\min}}\right)\log T\right)
\end{align}

(2) if $\lambda=1$,
\begin{align}
    &Reg(T) \le Reg(T, E_{t,1})+Reg(T,E_{t,2})+ \frac{2\pi^2}{3}m\Delta_{\max}\notag\\
    &\le  \sum_{i=1}^m \log\frac{c_1^2B_v^2 K}{(\Delta_{i}^{\min})^2}\frac{48c_1^2 B_v^2 \log T}{\Delta_{i}^{\min} } (3+ \log K) \notag\\
    &+\sum_{i\in [m]} 24 c_2 B_1 \left( \log_2\frac{B_1 c_2 K}{\Delta_{i}^{\min}}\right) \left(1+\log (\frac{K\Delta_i^{\max}}{\Delta_{i}^{\min}})\right)\log T \notag\\
    &+  \sum_{i\in [m]}\frac{2c_1^2 B_v^2 }{\Delta_i^{\min}} + \frac{m\pi^2}{6}\log_2\left(\frac{c_1^2B_v^2 K}{(\Delta_{\min})^2}\right) \Delta_{\max} +  \frac{m\pi^2}{6}\log_2\frac{4B_1 c_2 K}{\Delta_{i}^{\min}} \Delta_{\max}  +\frac{2\pi^2}{3}m\Delta_{\max}\notag\\
    &\le O\left(\sum_{i \in [m]}\frac{B_v^2 \log\left(\frac{B_vK}{\Delta_i^{\min}}\right) \log K \log T}{\Delta_i^{\min}} + \sum_{i \in [m]}B_1\log^2\left(\frac{B_1K}{\Delta_i^{\min}}\right)\log T\right)
\end{align}

\subsubsection{Regret Bound Using the Infinitely Many Events (\cref{apdx_sec:inf_many_events}).}

When using the analysis in \cref{apdx_sec:inf_many_events}, by \cref{eq:e_t}, \cref{sec:upper_bound_Et1_ifm}, \cref{sec:upper_bound_Et2_ifm},
the total regret is bounded as follows

(1) if $\lambda>1$,
\begin{align}
    &Reg(T) \le Reg(T, E_{t,1})+Reg(T,E_{t,2})+ \frac{2\pi^2}{3}m\Delta_{\max}\notag\\
    &\le\sum_{i=1}^m\left(\frac{480 c_1^2 B_v^2}{ (1-2^{-(\lambda-1)})}\right)\ceil{\frac{\log K}{1.61}}^2  \frac{\log T}{\Delta_{i}^{\min}}+120c_2 B_1\sum_{i=1}^m\left(\log_2\frac{c_2B_1 K}{\Delta_{i}^{\min}}\right) \left(1+\log \frac{\Delta_{i}^{\max}}{\Delta_{i}^{\min}}\right)\notag\\ &\lceil{\frac{\log K}{1.61}\rceil}^2 \log T+     \frac{m\pi^2}{6}\log_2\left(\frac{c_1^2B_v^2 K}{\lambda(\Delta_{\min})^2}\right) \Delta_{\max} + \frac{m\pi^2}{6}\log_2\left(\frac{c_2B_1 K}{\Delta_{\min}}\right) \Delta_{\max} +\frac{2\pi^2}{3}m\Delta_{\max}\notag\\
    &\le O\left(\sum_{i \in [m]}\frac{B_v^2 \log^2 K \log T}{\Delta_i^{\min}} + \sum_{i \in [m]}B_1\log^2\left(\frac{B_1K}{\Delta_i^{\min}}\right)\log^2 K\log T\right)
\end{align}

(2) if $\lambda=1$,
\begin{align}
    Reg(T) &\le Reg(T, E_{t,1})+Reg(T,E_{t,2})+ \frac{2\pi^2}{3}m\Delta_{\max}\notag\\
    &\le \sum_{i=1}^m\left(480 c_1^2 B_v^2\right)\log_2\left(\frac{c_1^2B_v^2 K}{(\Delta_{i}^{\min})^2}\right)\ceil{\frac{\log K}{1.61}}^2  \frac{\log T}{\Delta_{i}^{\min}} \notag\\
    &+120c_2 B_1\sum_{i=1}^m\left(\log_2\frac{c_2B_1 K}{\Delta_{i}^{\min}}\right) \left(1+\log \frac{\Delta_{i}^{\max}}{\Delta_{i}^{\min}}\right)\lceil{\frac{\log K}{1.61}\rceil}^2 \log T \notag\\
    &+ \frac{m\pi^2}{6}\log_2\left(\frac{c_1^2B_v^2 K}{(\Delta_{\min})^2}\right) \Delta_{\max} + \frac{m\pi^2}{6}\log_2\left(\frac{c_2B_1 K}{\Delta_{\min}}\right) \Delta_{\max}+\frac{2\pi^2}{3}m\Delta_{\max}\notag\\
    &\le O\left(\sum_{i \in [m]}\frac{B_v^2 \log\left(\frac{B_vK}{\Delta_i^{\min}}\right) \log^2 K \log T}{\Delta_i^{\min}} + \sum_{i \in [m]}B_1\log^2\left(\frac{B_1K}{\Delta_i^{\min}}\right)\log^2 K\log T\right)
\end{align}

\subsubsection{Discussion on the Distribution-Independent Bounds}\label{apdx_sec:ind_reg}
Similar to \citep[Appendix B.3]{wang2017improving}, for the distribution-independent regret bound, we fix a gap $\Delta$ to be decided later and we consider two events on $\Delta_{S_t}$: $\{\Delta_{S_t}\le \Delta\}$ and $\{\Delta_{S_t}> \Delta\}$.

For the former case, the regret is trivially $Reg(T, \{\Delta_{S_t}\le \Delta\})\le T\Delta$. For the later case, under $\{\Delta_{S_t}> \Delta\}$ it is also straight-forward to replace all $\Delta_i^{\min}$ with $\Delta$ in \cref{apdx_sec:sum_ra} and derive $Reg(T,\{\Delta_{S_t}> \Delta\})\le O\left(\frac{mB_v^2\log K \log T}{\Delta}+mB_1\log^2(\frac{B_1K}{\Delta})\log T \right)$ if $\lambda>1$ and $Reg(T,\{\Delta_{S_t}> \Delta\})\le O\left(\frac{mB_v^2\log (\frac{B_vK}{\Delta})\log K \log T}{\Delta}+mB_1\log^2(\frac{B_1K}{\Delta})\log T \right)$ if $\lambda=1$.

Therefore, for $\lambda>1$, by selecting $\Delta=\Theta\left(\sqrt{\frac{mB_v^2\log T \log K}{T}}+\frac{B_1m\log K\log T}{T}\right)$, we have
\begin{align}
    Reg(T) \le O\left( B_v\sqrt{m(\log K) T\log T} + B_1m \log^2(KT)\log T\right)
\end{align} 

For $\lambda=1$, by selecting $\Delta=\sqrt{\frac{mB_v^2\log^2 T \log K}{T}}+\frac{B_1m\log K\log T}{T}$, we have
\begin{align}
    Reg(T) \le O\left( B_v\sqrt{m(\log K) T} \log (KT) + B_1m \log^2(KT)\log T\right)
\end{align}

\subsection{Discussion on the Lower Bounds}
We consider the degenerate case from the lower bound result \cite{merlis2020tight}, our regret bound is tight (up to polylogaritmic factors in $K$). More specifically, \citet{merlis2020tight} consider the special non-triggering CMAB (where $\Delta_{\min}=\Delta$ and $\bmu$ are not exponentially close to $0$ or $1$), and they prove $\Omega(\frac{m\gamma_g^2\log T}{\Delta})$ and $\Omega(\frac{m\gamma_g^2\log T}{\log K\Delta})$ regret lower bounds for non-monotone and monotone reward functions, respectively. In our paper, this setting is the same as letting $p_i^{D,S}=1$ for $i\in S$ and $p_i^{D,S}=0$ otherwise (i.e., TPVM condition degenerates to VM condition). According to the Remark 4 in \cref{sec:TPVM}, we know $B_v=3\sqrt{2}\gamma_g$ and $\lambda=2$ so this gives an $O(\frac{m\gamma_g^2\log K\log T}{\Delta})$ bound, which is tight to the lower bound up to a $O(\log^2 K)$ factor.



\section{Regret Analysis for CMAB with Independent Arms (Proofs Related to Theorem \ref{thm:independent})}\label{apdx_sec:independent}
\subsection{Useful definitions and Inequalities}
We first give the formal definition, the properties and the tail bounds for sub-Gaussian and sub-Exponential random variables, which helps our analysis.
\begin{definition}[Sub-Gaussian Random Variable, \cite{vershynin2018high}]
A random variable with mean $\mu=\E[X]$ is sub-Gaussian with parameter $\sigma^2$ if 
\begin{equation}
    \E[e^{\lambda(X-\mu)}] \le e^{\frac{\lambda^2\sigma^2}{2}} \text{ for any } \lambda \in \R.
\end{equation}
In this case, we write $X \in \text{SG}(\sigma^2)$.
\end{definition}

\begin{definition}[Sub-Exponential Random Variable, \cite{vershynin2018high}]
A random variable with mean $\mu=\E[X]$ is sub-Exponential with parameter $(\nu^2,b)$ if 
\begin{equation}
    \E[e^{\lambda(X-\mu)}] \le e^{\frac{\lambda^2\nu^2}{2}} \text{ for any } |\lambda| < \frac{1}{b}.
\end{equation}
In this case, we write $X \in SE(\nu^2, b)$.
\end{definition}

\begin{lemma}[Tail bounds for sub-Exponential random variables, \cite{vershynin2018high}]\label{lem:exp_tail}
Let $Y \in \text{SE}(\nu^2, b)$ with mean $\mu=\E[Y]$. Then
\begin{equation}
    \Pr[|Y-\mu| \ge \tau] \le \begin{cases}
2e^{-\tau^2/(2\nu^2)}, &\text{if $0 < \tau \le \frac{\nu^2}{b}$}\\
2e^{-\tau/(2b)}, &\text{if $\tau > \frac{\nu^2}{b}$}.\\
\end{cases}
\end{equation}
\end{lemma}

\begin{lemma}\label{lem:subE}(Square of Sub-Gaussian Random Variable is Sub-Exponential \citep[Appendix B]{honorio2014tight}) For $X \in SG(\sigma^2)$ and let $Y=X^2$, then 
\begin{equation}
    \E[e^{\lambda(Y-E[Y])}]\le 16\lambda^2 \sigma^4, \text{ for any } |\lambda| \le \frac{1}{4\sigma^2}.
\end{equation}
Thus, $X^2 \in \text{SE}(\nu^2, b)$ with $\nu=4\sqrt{2}\sigma^2, b = 4\sigma^2$.
\end{lemma}

\begin{lemma}[Composition of independent sub-Exponential random variables, \cite{vershynin2018high}]\label{lem:SE_comp}
Let $Y_1,...,Y_n$ be independent sub-Exponential random variables $Y_i \in \text{SE}(\nu_i^2, b_i)$ with $\E[Y_i]=\mu_i$. Then 
\begin{equation}
    \sum_{i=1}^n(Y_i-\mu_i) \in \text{SE}\left(\sum_{i=1}^n \nu_i^2, \max_{i} b_i\right)
\end{equation}

\end{lemma}

\subsection{Proof of \cref{thm:independent}}
Recall that at time t, $T_{t-1,i}$ is the number of times base arm $i$ is observed and $\hat{\mu}_{t-1,i}$ is the empirical mean of arm $i$.
Let $\delta_{t,i}=\hat{\mu}_{t-1,i}-\mu_i$ and by Condition~\ref{cond:subg}, $\delta_{t,i}$ is a sub-Gaussian random variable $\text{SE}(\frac{C_1 (1-\mu_i)\mu_i}{T_{t-1,i}})$ with mean $0$.

Let $u_{t,i}=\frac{\delta_{t,i}}{\sqrt{(1-\mu_i)\mu_i}}$, then $u_{t,i}$ is also sub-Gaussian $\text{SE}( \frac{C_1}{T_{t-1, i}})$.
By Condition \ref{cond:VM}, we have that $|r(S;\hat{\bmu}_{t-1})-r(S;\bmu)|\le B_v \sqrt{\sum_{i \in S}\left(\frac{|\hat{\mu}_{t,i}-\mu_i|}{\sqrt{(1-\mu_i)\mu_i}}\right)^2}= B_v \sqrt{\sum_{i \in S}u_{t,i}^2}$.
Fix a super arm $S$, we will focus on random variable $Y_{t,S}\triangleq \sum_{i \in S} u_{t,i}^2$.

Since $u_{t,i}\in \text{SG}(\frac{C_1}{T_{t-1,i}})$ with $\E[u_{t,i}]=0$, and $u_{t,i}$ are independent across $i \in [m]$, we know
\begin{equation}
    Y_{t,S} \in \text{SE}(C_1^2 C_3\sum_{i \in S}\frac{1}{T_{t-1,i}^2}, C_1C_4\frac{1}{T^{\min}_{t-1,S}})
\end{equation} due to \cref{lem:subE} and \cref{lem:SE_comp}, where $T^{\min}_{t-1,S} = \min_{i \in S} T_{t-1,i}$, constants $C_3 = 32, C_4=4$.

For the mean of $Y_{t,S}$, we can also show that $\E[Y_{t,S}]=\sum_{i \in [S]}\E[u_{t,i}^2]= \sum_{i \in S}(\text{Var}[u_{t,i}]+(\E[u_{t,i}])^2) =  \sum_{i \in S}\text{Var}[u_{t,i}] \le \sum_{i \in S}  \frac{C_1}{T_{t-1,i}}$, where the last inequality is because the variance of any sub-Gaussian random variable $X \in \text{SE}(\sigma^2)$ is smaller than $\sigma^2$~\cite{vershynin2018high}.

For such a sub-Exponential random variable, we can give the confidence interval for $Y_{t,S}$ based on the tail bound \cref{lem:exp_tail}.
For any action $S \in \cS$, any time $t\in [T]$, it holds with probability $1-\delta$,
\begin{equation}\label{eq:SG_interval}
    Y_{t,S} \le \begin{cases}
\E[Y_{t,S}] + \sqrt{2C_1^2 C_3\log(\frac{2}{\delta})\sum_{i \in S}\frac{1}{T_{t-1,i}^2}}, &\text{if $\frac{C_1C_4}{T^{\min}_{t-1, S}}\sqrt{2\log(\frac{2}{\delta})} \le \sqrt{C_1^2C_3\sum_{i \in S}\frac{1}{T_{t-1,i}^2}}$}\\
\E[Y_{t,S}] + 2C_1C_4\log(\frac{2}{\delta})\frac{1}{T^{\min}_{t-1, S}}, &\text{otherwise}.\\
\end{cases}
\end{equation}
Equivalently, we can rewrite the above inequality by merging the above two segments as $Y_{t,S} \le C_1\sum_{i \in S}\frac{1}{T_{t-1,i}} + \max\{\sqrt{2C_1^2 C_3\log(\frac{2}{\delta})\sum_{i \in S}\frac{1}{T_{t-1,i}^2}}, 2C_1C_4\log(\frac{2}{\delta})\frac{1}{T^{\min}_{t-1, S}}\}$ and with probability at least $1-\delta$, it holds that
\begin{equation}\label{apdx_eq:se_conf_interv}
    |r(S;\hat{\bmu}_{t-1})-r(S;\bmu)| \le \rho_t(S),
\end{equation}

where $\rho_{t}(S)=B_v\sqrt{\sum_{i \in S} \frac{C_1}{T_{t-1},i}+ \max\left\{\sqrt{2C_1^2C_3\sum_{i \in S}\frac{\log(\frac{2}{\delta})}{T_{t-1,i}^2}}, \frac{2C_1C_4\log(\frac{2}{\delta})}{T^{\min}_{t-1, S}}\right\}}$. If $S$ is selected as the action in any round $t$, then
{
\begin{align}
\Delta_S &= \alpha r(S^*; \boldsymbol{\mu}) - r(S;\boldsymbol{\mu})\\
&\le \alpha ( r(S^*;\hat{\boldsymbol{\mu}}_{t-1}) + \rho_t(S^*)) - r(S; \boldsymbol{\mu}) & \mbox{by \cref{apdx_eq:se_conf_interv} over $S^*$} \\
&\le \bar{r}_t(S) - r(S; \boldsymbol{\mu}) & \mbox{$S$ is produced by $\bar{O}$ in line \ref{line:SE_oracle} of \cref{alg:SECUCB}}\\
&= r(S;\hat{\boldsymbol{\mu}}_{t-1}) + \rho_t(S) - r(S; \boldsymbol{\mu})  \\
&\le 2\rho_t(S) & \mbox{by \cref{apdx_eq:se_conf_interv} over $S$},
\end{align}}

In other words, $S$ can only be selected when $\rho_t(S) > \Delta_S/2$.

Now we consider two different cases based on the $\sum_{i \in S}\frac{1}{T_{t-1,i}}$.

\textbf{Case 1:} When $\sum_{i \in S}\frac{1}{T_{t-1,i}} < \frac{C_4 \Delta_S^2}{4C_1(C_3+C_4)(B_v)^2}$,

We first show by contraction that the confidence interval lies in the second part of \cref{eq:SG_interval}, i.e. $\rho_t(S)=B_v\sqrt{C_1\sum_{i \in S} \frac{1}{T_{t-1},i}+ 8C_1\log(\frac{2}{\delta})\frac{1}{T^{\min}_{t-1, S}}}$. Based on \cref{lem:exp_tail}, if the confidence interval lies in the first part, then 
\begin{align}
    \tau &\le \frac{\nu^2}{b} = \frac{C_1^2 C_3\sum_{i \in S}\frac{1}{T_{t-1,i}^2}}{C_1C_4\frac{1}{T^{\min}_{t-1,S}}}=\sum_{i \in S}\frac{C_1C_3}{C_4}\frac{T^{\min}_{t-1,S}}{T^2_{t-1, i}}\\
    &\le \sum_{i \in S}\frac{C_1C_3}{C_4}\frac{1}{T_{t-1, i}}\\
    &\le \frac{C_1C_3}{C_4}\frac{C_4 \Delta_S^2}{4C_1(C_3+C_4)(B_v)^2}=\frac{C_3 \Delta_S^2}{4(C_3+C_4)(B_v)^2}.
\end{align}
This indicates that 
\begin{align}
    \rho_t(S)&= B_v \sqrt{\E[Y_{t,S}]+\tau} \\
    &\le B_v \sqrt{\E[Y_{t,S}]+\frac{C_3 \Delta_S^2}{4(C_3+C_4)(B_v)^2}}\\
    &\le  B_v\sqrt{C_1 \sum_{i \in S}\frac{1}{T_{t-1,i}} + \frac{C_3 \Delta_S^2}{4(C_3+C_4)(B_v)^2}}\\
    &\le  B_v\sqrt{\frac{C_4 \Delta_S^2}{4(C_3+C_4)B_v^2} + \frac{C_3 \Delta_S^2}{4(C_3+C_4)(B_v)^2}}\\
    \label{apdx_eq:contradiction}&= \Delta_S/2,
\end{align}
which contradicts the requirement  $\rho_t(S) > \Delta_S/2$.

Therefore, the confidence interval lies in the second part of \cref{eq:SG_interval} with $\tau = \frac{2C_1C_4\log(\frac{2}{\delta})}{T^{\min}_{t-1, S}}$.
By the same analysis, we need $\tau \ge \frac{C_3 \Delta_S^2}{4(C_3+C_4)(B_v)^2}$ (or otherwise we suffer from the same contradiction as shown in \cref{apdx_eq:contradiction}), hence we require $\Delta_S \le \sqrt{\frac{8(C_3+C_4)(B_v)^2C_1C_4 \log (\frac{2}{\delta})}{C_3 T^{\min}_{t-1, S}}}$.

Equivalently, we require $\Delta_S \le 2\sqrt{\frac{8(C_3+C_4)(B_v)^2C_1C_4 \log (\frac{2}{\delta})}{C_3 T^{\min}_{t-1, S}}} - \Delta_S$ whenever $S$ is selected, due to the use of the reverse amortization trick as in \cref{eq:ra_def_reverse}.

Now we can define the regret allocation, We first define a regret allocation function 
\begin{equation}
\kappa_{i,\delta}(S, \ell) = \begin{cases}
2\sqrt{\frac{8(C_3+C_4)B_v^2C_1C_4 \log (\frac{2}{\delta})}{C_3 \ell}}, &\text{if $1\le\ell\le L_{i,\delta}$ and $i = \argmin_{j \in S}{T_{t-1,j}}$}\\
0, &\text{otherwise,}
\end{cases}
\end{equation}
where $L_{i,\delta}=\frac{32(C_3+C_4)B_v^2C_1C_4 \log (\frac{2}{\delta})}{C_3 (\Delta_i^{\min})^2}$. Also note that if there are multiple arms that achieves the minimum, select the one with minimum index as the min-arm.

It can be easily shown that $\Delta_{S_t} \le \sum_{i \in S_t} \kappa_{i,\delta}(S_t, T_{t-1,i})$ as follows.

Let $j=\argmin_{i \in S_t}{T_{t-1,i}}$.
If $T_{t-1,j} > L_{j,\delta}$, then $\Delta_{S_t} \le 2\sqrt{\frac{8(C_3+C_4)B_v^2C_1C_4 \log (\frac{2}{\delta})}{C_3 T^{\min}_{t-1, S}}}-\Delta_{S_t} \le 2\sqrt{\frac{8(C_3+C_4)B_v^2C_1C_4 \log (\frac{2}{\delta})}{C_3 T_{t-1, j}}} -\Delta_{S_t}\le \Delta_{j}^{\min}-\Delta_{S_t}\le0=\sum_{i \in S_t}\kappa_{i,\delta}(S_t, T_{t-1,i})$.
If $T_{t-1,j} \le L_{j,\delta}$, then $\Delta_{S_t} \le 2\sqrt{\frac{8(C_3+C_4)B_v^2C_1C_4 \log (\frac{2}{\delta})}{C_3 T^{\min}_{t-1, S}}}-\Delta_{S_t}  \le 2\sqrt{\frac{8(C_3+C_4)B_v^2C_1C_4 \log (\frac{2}{\delta})}{C_3 T^{\min}_{t-1, S}}}= 2\sqrt{\frac{8(C_3+C_4)B_v^2C_1C_4 \log (\frac{2}{\delta})}{C_3 T_{t-1,j}}} + \sum_{i \in  S_t \backslash \{j\}} 0 =\sum_{i \in S_t}\kappa_{i,\delta}(S_t, T_{t-1,i}) $

Hence the regret for case 1 is upper bounded by 
\begin{align}
    \sum_{t=1}^T\sum_{i \in S_t}\Delta_{S_t}&\le \sum_{t=1}^T\sum_{i \in S_t}\kappa_{i,\delta}(S_t,T_{t-1,i})\\
    &\le\sum_{i \in [m]}\sum_{\ell=1}^{L_{i,\delta}}2\sqrt{\frac{8(C_3+C_4)B_v^2C_1C_4 \log (\frac{2}{\delta})}{C_3 \ell}}\\
    &\le \sum_{i \in [m]}2\sqrt{\frac{8L_{i,\delta}(C_3+C_4)B_v^2C_1C_4 \log (\frac{2}{\delta})}{C_3 }}\int_{x=0}^{L_{i,\delta}}\sqrt{\frac{1}{x}}\\
    &\le \sum_{i \in [m]} \frac{64(C_3+C_4)B_v^2C_1C_4 \log (\frac{2}{\delta})}{C_3\Delta_i^{ \min}}\label{apdx_eq:ind_case_1_regret}
\end{align}

\textbf{Case 2:} When $\sum_{i \in S}\frac{1}{T_{t-1,i}} \ge \frac{C_4 \Delta_S^2}{4C_1(C_3+C_4)B_v^2}$,

This case implies that $\frac{K}{T^{\min}_{t-1,S}}\ge\sum_{i \in S}\frac{1}{T_{t-1,i}} \ge \frac{C_4 \Delta_S^2}{4C_1(C_3+C_4)B_v^2}$.
Rewriting the inequality, we have $\Delta_S\le \sqrt{\frac{4KC_1(C_3+C_4)B_v^2}{C_4 T^{\min}_{t-1,S}}}$ and following the similar regret allocation and argument from case 1 (where we require $\Delta_S \le \sqrt{\frac{8(C_3+C_4)(B_v)^2C_1C_4 \log (\frac{2}{\delta})}{C_3 T^{\min}_{t-1, S}}}$), we have the case 2 contributes at most $\frac{32mKC_1(C_3+C_4)B_v^2}{C_4 \Delta_{\min}}$, which is irrelevant of the time horizon $T$.

Now for the first $m$ rounds so that each arm is observed at least once (i.e., counter $T_{t-1,i} \ge 1$ from $t \ge m+1$ rounds) and consider the bad events when there exists $t \in [T], S \in \cS$ such that \cref{eq:SG_interval} does not hold, by setting $\delta=(|\cS|T)^{-1}$ and using the union bound, the additional regret is upper bounded by
\begin{align}
    m\Delta_{\max} + \sum_{t\in[T]}\sum_{S \in \cS} \frac{1}{|\cS|T}\Delta_{\max} \le (m+1)\Delta_{\max}. 
\end{align}

Therefore the total regret is upper bounded by (using the similar proof following \cref{apdx_eq:oracle_fail} for the failure of oracle),
\begin{align}
    &Reg(T)\le \sum_{i \in [m]} \frac{64(C_3+C_4)B_v^2C_1C_4 \log (2|\cS|T)}{C_3\Delta_i^{ \min}}+\frac{32mKC_1(C_3+C_4)B_v^2}{C_4 \Delta_{\min}}+(m+1)\Delta_{\max}\\
    &\le \sum_{i \in [m]} \frac{64(C_3+C_4)B_v^2C_1C_4 \log (2T)}{C_3\Delta_i^{ \min}}\notag\\
    &+\sum_{i \in [m]} \frac{64(C_3+C_4)B_v^2C_1C_4 \log (|\cS|)}{C_3\Delta_i^{ \min}}+\frac{32mKC_1(C_3+C_4)B_v^2}{C_4 \Delta_{\min}}+(m+1)\Delta_{\max}\\
    &\le O\left(\sum_{i \in [m]}\frac{B_v^2\log T}{\Delta_i^{\min}}\right)\label{apdx_eq:reg_indpendent_final}
\end{align}

For the distribution-independent regret, similar to \cref{apdx_sec:ind_reg}, $Reg(T)\le \frac{m B_v^2\log T}{\Delta} + T\Delta$, when $T \rightarrow \infty$. By setting $\Delta=\Theta\left(\sqrt{
\frac{mB_v^2\log T}{T}}\right)$, we have
\begin{align}
    Reg(T)\le O\left(B_v\sqrt{mT\log T}\right).
\end{align}

\subsection{Computational Efficient Oracle for SESCB}\label{apdx_sec:eff_escb}
{Recall that $\rho'_t(S)=B_v\sqrt{\sum_{i \in S} \frac{C_1}{T_{t-1,i}}+ 8C_1\sqrt{\sum_{i \in S}\frac{\log(2|\mathcal{S}|T)}{T_{t-1,i}^2}}+ \frac{8C_1\log(2|\mathcal{S}|T)}{T^{\min}_{t-1, S}} }$ and ${r}'_t(S)=r(S;\hat{\boldsymbol{\mu}}_{t-1})+\rho'_t(S)$. For the submodularity of $r'_t(S)$, it suffices to show $\rho'_t(S)$ is monotone submodular when $r(S;\bmu)$ is  monotone submodular.  We know that $g(f(S))$ is submodular if $f(S)$ is submodular and $g$ is a non-decreasing concave function, so it suffices to show three terms within the (non-decreasing concave) square root  in $\rho'_t(S)$ are submodular. The first term is a modular function, the second term is the square root of a modular function, and the third term can be rewritten as $\max_{i \in S}\frac{8C_1\log(2|\mathcal{S}|T)}{T_{t-1, i}}$, which is also submodular.

For the regret bound when using $\rho'_t(S)$ instead of $\rho_t(S)$, it can be seen that $\rho_t(S) \le \rho'_t(S) \le \sqrt{2}\rho_t(S)$ for all $S \in \cS$, since $\max\{a,b\} \le a + b \le 2 \{a,b\}$ for any $a,b \in \R$. So we can equivalently use $B'_v=\sqrt{2}B_v$ to replace $B_v$ and repeat the same proof in \cref{thm:independent} with an additional factor of $2$ in \cref{apdx_eq:reg_indpendent_final}. }

\section{Proof of TPVM Smoothness Conditions for Various Applications (Related to Theorem \ref{lem:app_tab})}\label{apdx_sec:proof_app}
For convenience, we show our table again in this section.
        \begin{table*}[h]
	\caption{Summary of the coefficients, regret  bounds and improvements for various applications.}	\label{tab:app_restate}	
		\resizebox{1.00\columnwidth}{!}{
	\centering
	\begin{threeparttable}
	\begin{tabular}{|ccccc|}
		\hline
		\textbf{Application}&\textbf{Condition}& \textbf{$(B_v, B_1, \lambda)$} & \textbf{Regret}& \textbf{Improvement}\\
	\hline

		Disjunctive Cascading Bandits~\citep{kveton2015combinatorial} & {\TPVMm} & $(1,1,2)$ & $O(\sum_{i \in [m]}\frac{\log K\log T}{\Delta_i^{\min}})$ & $O(K/\log K)$ \\
			\hline
		Conjunctive Cascading Bandits~\citep{kveton2015combinatorial} & TPVM & $(1,1,1)$ & $O( \sum_{i \in [m]}\log \frac{K}{\Delta_{\min}}\frac{ \log K\log T}{\Delta_{i,\min}})$ & $O(K/(\log K\log \frac{K}{\Delta_{\min}}))$\\
			\hline
		Multi-layered Network Exploration~\citep{liu2021multi} & TPVM & $(\sqrt{1.25|V|},1,2)$ $^\dagger$ & $O(\sum_{i \in \cA}\frac{|V|\log (n|V|)\log T}{\Delta_i^{\min}})$ & $O(n/\log (n|V|))$\\
			\hline
	    Influence Maximization on DAG \cite{wang2017improving} & {\TPVMm} &$(\sqrt{L}|V|,|V|,1)$ $^{\dagger}$  & $O( \sum_{i \in [m]}\log \frac{|E|}{\Delta_{\min}}\frac{L|V|^2\log |E| \log T}{\Delta_i^{\min}})$ & $O(|E|/(L\log |E|\log \frac{|E|}{\Delta_{\min}}))$\\
	    \hline
	    Probabilistic Maximum Coverage \cite{merlis2019batch}$^*$ & VM & $(3\sqrt{2|V|},1,-)$ & $O(\sum_{i \in [m]}\frac{|V|\log T}{\Delta_i^{\min}})$ & $O(\log^2 k)$.\\
	    \hline
	\end{tabular}
	  \begin{tablenotes}[para, online,flushleft]
	\footnotesize
	\item[]\hspace*{-\fontdimen2\font}$^*$ This row is for the application in \cref{sec:independent} and the rest of rows are for \cref{sec:TPVM}; 
	\item[]\hspace*{-\fontdimen2\font}$^{\dagger}$ $|V|, |E|, n, k, L$ denotes the number of target nodes, the number of edges that can be triggered by the set of seed nodes, the number of layers, the number of seed nodes and the length of the longest directed path, respectively.
	\end{tablenotes}
			\end{threeparttable}
	}
\end{table*}

\subsection{Combinatorial cascading bandits}
Combinatorial cascading bandits has two categories: conjunctive cascading bandits and disjunctive cascading bandits~\cite{kveton2015combinatorial}.

\textbf{Disjunctive form}
For the disjunctive form, we want to select an ordered list $S$ of $K$ items from total $L$ items, so as to maximize the probability that at least one of the outcomes of the selected items are $1$.
Each item is associated with a Bernoulli random variable with mean $\mu_i$, indicating whether the user will be satisfied with the item if he scans the item.
This setting models the movie recommendation system where the user sequentially scans a list of recommended items and the system is rewarded when the user is satisfied with any recommended item.
After the user is satisfied with any item or scans all $K$ items but is not satisfied with any of them, the user leaves the system.
Due to this stopping rule, the agent can only observe the outcome of items until (including) the first item whose outcome is $1$. If there are no satisfactory items, the outcomes must be all $0$. In other words, the triggered set is the prefix set of items until the stopping condition holds.
For this application, we have the following lemma. 

\begin{lemma}\label{lem:app_cas_dis}
Disjunctive conjunctive cascading bandit problem satisfies {\TPVMm} bounded smoothness condition with coefficient $(B_v,B_1,\lambda)=(1,1,2)$.
\end{lemma}

\begin{proof}
Without loss of generality, let the action be $\{1,...,K\}$, then the reward function is $r(S;\bmu)=1-\prod_{j=1}^K(1-\mu_j)$ and the triggering probability is $p_{i}^{D,S}=\prod_{j=1}^{i-1}(1-\mu_j)$.
Let $\bar{\bmu}=(\bar{\mu}_1, ..., \bar{\mu}_K)$ and $\bmu=(\mu_1, ..., \mu_K)$, where $\bar{\bmu}=\bmu+\boldzeta + \boldeta$ with $\bar{\bmu}, \bmu \in (0,1)^K, \boldzeta, \boldeta \in [0,1]^K$.
\begin{align}
    &\abs{r(S;\bar{\bmu})-r(S;\bmu)}\notag\\
    &=\prod_{i=1}^K(1-\mu_i)-\prod_{i=1}^K(1-\bar{\mu}_i)\label{apdx_eq:cas_dis_mono}\\
    &=\sum_{i \in [K]} (\bar{\mu}_i-\mu_i)((1-\mu_1)...(1-\mu_{i-1}) (1-\bar{\mu}_{i+1})...(1-\bar{\mu}_{K}))\label{apdx_eq:cas_dis_tlscp}\\
    &=\sum_{i \in [K]} (\zeta_i+\eta_i)((1-\mu_1)...(1-\mu_{i-1}) (1-\bar{\mu}_{i+1})...(1-\bar{\mu}_{K}))\label{apdx_eq:cas_dis_def_zeta}\\
    &\le \sum_{i \in [K]} (\zeta_i)((1-\mu_1)...(1-\mu_{i-1})(1-\mu_{i+1})...(1-\mu_{K})) +  \sum_{i \in [K]} (\eta_i)((1-\mu_1)...(1-\mu_{i-1}))\label{apdx_eq:cas_dis_bigger}\\
    &\le \sum_{i \in [K]} \frac{\zeta_i p_i^{D,S}}{\sqrt{(1-\mu_i)\mu_i}} \sqrt{(1-\mu_{i+1})...(1-\mu_{K}) \mu_i} + \sum_{i \in [K]} \eta_i p_i^{D,S}\label{apdx_eq:cas_dis_two_sides}\\
    &\le\sqrt{ \sum_{i \in [K]} \frac{\zeta_i^2}{(1-\mu_i)\mu_i}(p_i^{D,S})^2 }\sqrt{ \sum_{i \in [K]} (1-\mu_{i+1})...(1-\mu_{K}) \mu_i } + \sum_{i \in [K]} \eta_i p_i^{D,S}\label{apdx_eq:cas_dis_cauchy}\\
    &\le \sqrt{ \sum_{i \in [K]} \frac{\zeta_i^2}{(1-\mu_i)\mu_i}(p_i^{D,S})^2 }\sqrt{ 1-(1-\mu_1)...(1-\mu_K) } + \sum_{i \in [K]} \eta_i p_i^{D,S}\label{apdx_eq:cas_dis_ineq}\\
    &\le\sqrt{ \sum_{i \in [K]} \frac{\zeta_i^2}{(1-\mu_i)\mu_i}(p_i^{D,S})^2 } + \sum_{i \in [K]} \eta_i p_i^{D,S},
\end{align}
where \cref{apdx_eq:cas_dis_mono} uses $\mu_i \le \bar{\mu}_i, i \in [m]$, \cref{apdx_eq:cas_dis_tlscp} is by telescoping the reward difference, \cref{apdx_eq:cas_dis_def_zeta} is by definition of $\boldzeta,\boldeta$, \cref{apdx_eq:cas_dis_bigger} is due to $\bar{\mu}_i \ge {\mu}_i, i \in [m]$ , \cref{apdx_eq:cas_dis_two_sides} is due to the definition of $p_i^{D,S}$ and we multiply the first term by $\sqrt{\mu_i}$ but divide it by $\sqrt{(1-\mu_i)\mu_i}$, \cref{apdx_eq:cas_dis_cauchy} is due to the Cauchy–Schwarz inequality on the first term, \label{apdx_eq:cas_dis_ineq} is by math calculation.
Hence, $(B_g, B_1, \lambda)=(1,1,2)$.

\end{proof}

\textbf{Conjunctive form.}
For the conjunctive form, the learning agent wants to select $K$ paths from total $L$ paths (i.e., base arms) so as to maximize the probability that the outcomes of the selected paths are all $1$.
Each item is associated with a Bernoulli random variable with mean $\mu_i$, indicating whether the path will be live if the package will transmit via this path.
Such a setting models the network routing problem~\cite{kveton2015combinatorial}, where the items are routing paths and the package is delivered when all paths are alive. 
The learning agent will observe the outcome of the first few paths till the first one that is down, since the transmission will stop if any of the path is down. In other words, the triggered set is the prefix set of paths until the stopping condition holds.
We have the following lemma.
\begin{lemma}\label{lem:app_cas_con}
Conjunctive cascading bandit problem satisfies TPVM bounded smoothness condition with coefficient $(B_v,B_1,\lambda)=(1,1,1)$.
\end{lemma}

\begin{proof}
Without loss of generality, suppose the selected base arms are $\{1,...,K\}$, then the reward function is $r(S;\bmu)=\prod_{j=1}^K\mu_j$ and the triggering probability is $p_{i}^{\mu,S}=\prod_{j=1}^{i-1}\mu_j$.
Let $\bar{\bmu}=(\bar{\mu}_1, ..., \bar{\mu}_K)$ and $\bmu=(\mu_1, ..., \mu_K)$, where $\bar{\bmu}=\bmu+\boldzeta + \boldeta$ with $\bar{\bmu}, \bmu \in (0,1)^K, \boldzeta, \boldeta \in [-1,1]^K$.

\begin{align}
    &\abs{r(S;\bar{\bmu})-r(S;\bmu)}\notag\\
    &=\abs{\prod_{i=1}^K\bar{\mu}_i-\prod_{i=1}^K\mu_i}\notag\\
    &\le\abs{\sum_{i \in [K]} (\bar{\mu}_i-\mu_i)(\mu_1...\mu_{i-1} \bar{\mu}_{i+1}...\bar{\mu}_{K})}\notag\\
    &=\sum_{i \in [K]} (\abs{\zeta_i}+\abs{\eta_i})(\mu_1...\mu_{i-1} \bar{\mu}_{i+1}...\bar{\mu}_{K})\notag\\
    &\le \sum_{i \in [K]} \abs{\zeta_i}(\mu_1...\mu_{i-1}) +  \sum_{i \in [K]} \abs{\eta_i}(\mu_1...\mu_{i-1})\label{apdx_eq:cas_con_mono}\\
    &\le \sum_{i \in [K]} \frac{\abs{\zeta_i}}{\sqrt{(1-\mu_i)\mu_i}}\sqrt{p_i^{D,S}} \sqrt{(\mu_1...\mu_{i-1}) (1-\mu_i)} + \sum_{i \in [K]} \abs{\eta_i}p_i^{D,S}\label{apdx_eq:cas_con_bigger}\\
    &\le\sqrt{ \sum_{i \in [K]} \frac{\zeta_i^2}{(1-\mu_i)\mu_i}p_i^{D,S} }\sqrt{ \sum_{i \in [K]} (\mu_1...\mu_{i-1}) (1-\mu_i) } + \sum_{i \in [K]} \abs{\eta_i}p_i^{D,S}\label{apdx_eq:cas_con_cauchy}\\
    &\le \sqrt{ \sum_{i \in [K]} \frac{\zeta_i^2}{(1-\mu_i)\mu_i}p_i^{D,S} }\sqrt{ 1-\mu_1...\mu_K } + \sum_{i \in [K]}\abs{\eta_i}p_i^{D,S}\label{apdx_eq:cas_con_ieq}\\
    &\le\sqrt{ \sum_{i \in [K]} \frac{\zeta_i^2}{(1-\mu_i)\mu_i}p_i^{D,S} } + \sum_{i \in [K]} \abs{\eta_i} p_i^{D,S}\label{apdx_eq:cas_con_ieq_2}
\end{align}
where \cref{apdx_eq:cas_con_mono} uses $ \bar{\mu}_i \in [0,1], i \in [m]$, \cref{apdx_eq:cas_con_bigger} is due to the definition of $p_i^{D,S}$ and multiply the first term by $\sqrt{1-\mu_i}$ but divided it by $\sqrt{(1-\mu_i)\mu_i}$, \cref{apdx_eq:cas_con_cauchy} is due to the Cauchy–Schwarz inequality on the first term, \cref{apdx_eq:cas_con_ieq} is by math calculation and \cref{apdx_eq:cas_con_ieq_2} is due to $\mu_i \in [0,1], i \in [m]$.
Hence, $(B_g,B_1,\lambda)=(1,1,1)$.
\end{proof}





\subsection{Multi-layered Network Exploration Problem (MuLaNE)~\cite{liu2021multi}}
We consider the MuLaNE problem with random node weights.
After we apply the bipartite coverage graph, the corresponding graph is a tri-partite graph $(n,V,R)$ (i.e., a 3-layered graph where the first layer and the second layer forms a bipartite graph, and the second and the third layer forms another bipartite graph), where the left nodes represent $n$ random walkers; Middle nodes are $|V|$ possible targets $V$ to be explored; Right nodes $R$ are $V$ nodes, each of which has only one edge connecting the middle edge. The MuLaNE task is to allocate $B$ budgets into $n$ layers to explore target nodes $V$ and the base arms are $\cA=\{(i,u,b): i \in [n], u \in V, b \in [B]\}$.

With budget allocation $k_1, ..., k_L$, the (effective) base arms consists of two parts: 

(1) $\{(i,j): i\in[n], j \in V\}$, each of which is associated with visiting probability $x_{i,j}\in[0,1]$ indicating whether node $j$ will be visited by explorer $i$ given $k_i$ budgets. All these base arms corresponds to budget $k_i, i\in [n]$ are triggered.

(2) $y_j\in[0,1]$ for $j \in V$ represents the random node weight. The triggering probability $p_{j}^{D, S}=1-\prod_{i \in [n]}{(1-x_{i,j})}$.

We have the following lemma.
\begin{lemma}\label{lem:app_mulane}
MuLaNE problem satisfies TPVM bounded smoothness condition with coefficient $(B_v,B_1,\lambda)=(\sqrt{1.25|V|},1,2)$, where $|V|$ is the total number of vertices to be explored.
\end{lemma}

\begin{proof}
Let effective base arms $\bmu=(\bx,\by) \in (0,1)^{(n|V|+|V|)}, \bar{\bmu}=(\bar{\bx}, \bar{\by})\in (0,1)^{(n|V|+|V|)}$, where $\bar{\bx}=\boldzeta_{x}+\boldeta_{x}+\bx, \bar{\by}=\boldzeta_{y}+\boldeta_{y}+\by$, for $\boldzeta, \boldeta \in [-1,1]^{(n|V|+|V|)}$.
For the target node $j \in V$, the per-target reward function $r_j(S;\bx,\by)=y_j(1-\prod_{i \in [n]}(1-x_{i,j}))$.
Denote $\bar{p}_j^{D,S}=1-\prod_{i \in [n]}{(1-\bar{x}_{i,j})}$.

\begin{align}
    &\abs{r(S;\bar{\bmu})-r(S;\bmu)}\notag\\
    &=\abs{\sum_{j \in V} r_j(S;\bar{\bx},\bar{\by}) - r_j(S;\bx,\by)}\notag\\
    &=\abs{ \sum_{j \in V} \bar{y}_j \bar{p}_j^{D,S} - \bar{y}_j p_{j}^{D, S} + \bar{y}_j p_{j}^{D, S} - y_j p_{j}^{D, S}}\notag\\
    &=\abs{\sum_{j \in V} \bar{y}_j \left(\prod_{i \in [n]}(1-x_{i,j})-\prod_{i \in [n]}(1-\bar{x}_{i,j}))\right) + \sum_{j \in V} (\bar{y}_j-y_j)p_{j}^{D, S}}\label{apdx_eq:mulane_tlscp}\\
    &\le\abs{\sum_{j \in V} \sum_{i \in [n]}(\bar{x}_{i,j}-x_{i,j})\left((1-x_{1,j})...(1-x_{i-1,j})(1-\bar{x}_{i+1,j})...(1-\bar{x}_{L,j})\right)\bar{y}_j}\notag\\
    &+\abs{\sum_{j \in V} (\bar{y}_j-y_j)p_{j}^{D, S}}\notag\\
    &\le\sum_{j \in V}\sum_{i \in [n]}(\abs{\zeta_{x,i,j}}+\abs{\eta_{x,i,j}})\left((1-x_{1,j})...(1-x_{i-1,j})\right)\bar{y}_j+\sum_{j \in V}(\abs{\zeta_{y,j}}+\abs{\eta_{y,j}})p_{j}^{D, S}\label{apdx_eq:mulane_bigger}\\
    &\le\sum_{j \in V}\sum_{i \in [n]}(\frac{\abs{\zeta_{x,i,j}}}{\sqrt{(1-x_{i,j})x_{i,j}}})\sqrt{\left((1-x_{1,j})...(1-x_{i-1,j})\right)x_{i,j}}\bar{y}_j\notag\\
    &+\sum_{j \in V}\left(\frac{\abs{\zeta_{y,j}}p_{j}^{D, S}}{\sqrt{(1-y_j)y_j}}\right)\sqrt{(1-y_j)y_j}+\left(\sum_{j \in V}\sum_{i \in [n]}\abs{\eta_{x,i,j}} + \sum_{j \in V}\abs{\eta_{y,j}}p_{j}^{D, S}\right)\label{apdx_eq:mulane_multiply}\\
    &\le \sqrt{ \sum_{j \in V}\sum_{i \in [n]}(\frac{\zeta_{x,i,j}^2}{(1-x_{i,j})x_{i,j}})+\sum_{j \in V}\frac{\zeta_{y,j}^2(p_{j}^{D, S})^2}{(1-y_j)y_j}}\notag\\
    &\cdot\sqrt{\sum_{j \in V}\sum_{i \in [n]}\left((1-x_{1,j})...(1-x_{i-1,j})\right)x_{i,j}\bar{y}^2_j + \sum_{j \in V}(1-y_j)y_j}\notag\\
    &+ \left(\sum_{j \in V}\sum_{i \in [n]}\abs{\eta_{x,i,j}} +\sum_{j \in V} \abs{\eta_{y,j}}p_{j}^{D, S}\right)\label{apdx_eq:mulane_cauchy}\\
    &= \sqrt{ \sum_{j \in V}\sum_{i \in [n]}(\frac{\zeta_{x,i,j}^2}{(1-x_{i,j})x_{i,j}})+\sum_{j \in V}\frac{\zeta_{y,j}^2(p_{j}^{D, S})^2}{(1-y_j)y_j}} \sqrt{\sum_{j \in V}\left(1-\prod_{i \in [n]}(1-x_{i,j})\right)\bar{y}^2_j + \sum_{j \in V}(1-y_j)y_j}\notag\\
    &+ \left(\sum_{j \in V}\sum_{i \in [n]}\abs{\eta_{x,i,j}} + \sum_{j \in V}\abs{\eta_{y,j}}p_{j}^{D, S}\right)\label{apdx_eq:mulane_fold}\\
    &\le \sqrt{\sum_{j \in V} \sum_{i \in [n]}(\frac{\zeta_{x,i,j}^2}{(1-x_{i,j})x_{i,j}})+\sum_{j \in V}\frac{\zeta_{y,j}^2(p_{j}^{D, S})^2}{(1-y_j)y_j}} \sqrt{\sum_{j \in V}\bar{y}^2_j + (1-y_j)y_j}+  \notag\\
    &\left(\sum_{j \in V}\sum_{i \in [n]}\abs{\eta_{x,i,j}} + \sum_{j \in V}\abs{\eta_{y,j}}p_{j}^{D, S}\right)\label{apdx_eq:mulane_bounded}\\
    &\le \sqrt{ \sum_{j \in V}\sum_{i \in [n]}(\frac{\zeta_{x,i,j}^2}{(1-x_{i,j})x_{i,j}})+\sum_{j \in V}\frac{\zeta_{y,j}^2(p_{j}^{D, S})^2}{(1-y_j)y_j}} \cdot \sqrt{1.25|V|} \notag\\
    &+ \left(\sum_{j \in V}\sum_{i \in [n]}\abs{\eta_{x,i,j}}+\sum_{j \in V}\abs{\eta_{y,j}}p_{j}^{D, S}\right),\label{apdx_eq:mulane_bounded_2}
\end{align}
where \cref{apdx_eq:mulane_tlscp} is by telescoping the difference of $\left(\prod_{i \in [n]}(1-x_{i,j})-\prod_{i \in [n]}(1-\bar{x}_{i,j}))\right)$, \cref{apdx_eq:mulane_bigger} is due to $\bar{x}_{i,j} \in [0,1]$ for any $i,j$, \cref{apdx_eq:mulane_multiply} is because we multiply the first term by $\sqrt{x_{i,j}}$ but divide it by $\sqrt{x_{i,j}(1-x_{i,j})}$  and we multiply the second term by $\sqrt{y_j}$ but divide it by $\sqrt{y_j(1-y_j)}$, \cref{apdx_eq:mulane_cauchy} is by applying the Cauchy-Schwarz inequality to the first term and the second term simultaneously, \cref{apdx_eq:mulane_fold} is due to the math calculation on , \cref{apdx_eq:mulane_bounded} is due to the support of $x_{i,j}$ are bounded with $[0,1]$ for any $i,j$ and \cref{apdx_eq:mulane_bounded_2} is due to the support of $y_{j}, \bar{y}_{j}$ are bounded with $[0,1]$ for any $j$.
Hence $(B_g, B_1, \lambda)=(\sqrt{1.25|V|}, 1, 2)$.
\end{proof}

\subsection{Online Influence Maximization Bandit~\cite{wang2017improving}}
In this section, we first introduce the general problem setting of online influence maximization (OIM). Then we start with a easier problem instance, i.e., OIM on tri-partite graph (TPG) to get some intuition, where reward functions have closed form solutions. Finally, we give the results for OIM on the directed acyclic graph (DAG), which are much more involved since the reward function have no closed form solutions.

\subsubsection{Setting of the Online Influence Maximization}
Following the setting of \citep[Section 2.1]{wang2017improving}, we consider a weighted directed graph $G(V,E,p)$, where $V$ is the set of vertices, $E$ is the set of directed edges, and each edge $(u,v)\in E$ is associated with a probability $p(u,v)\in [0,1]$. When the agent selects a set of seed nodes $S \subseteq V$, the influence propagates as follows: At time $0$, the seed nodes $S$ are activated; At time $t > 1$, a node $u$ activated at time $t-1$ will have one chance to activate its inactive out-neighbor $v$ with independent probability $p(u,v)$. The influence spread of $S$ is denoted as $\sigma(S)$ and is defined as the expected number of activated nodes after the propagation process ends. The problem of Influence Maximization is to find seed nodes $S$ with $|S|\le k$ so that the influence spread $\sigma(S)$ is maximized. 

For the problem of online influence maximization (OIM), we consider $T$ rounds repeated influence maximization tasks and the edge probabilities $p(u,v)$ are assumed to be unknown initially. For each round $t\in [T]$, the agent selects $k$ seed nodes as $S_t$, the influence propagation of $S_t$ is observed and the reward is the number of nodes activated in round $t$. The agent's goal is to accumulate as much reward as possible in $T$ rounds. The OIM fits into CMAB-T framework: the edges $E$ are the set of base arms $[m]$, the (unknown) outcome distribution $D$ is the joint of $m$ independent Bernoulli random variables for the edge set $E$, the action $S$ are any seed node sets with size $k$ at most $k$. For the arm triggering, the triggered set $\tau_t$ is the set of edges $(u,v)$ whose source node $u$ is reachable from $S_t$. Let $X_t$ be the outcomes of the edges $E$ according to probability $p(u,v)$ and the live-edge graph $G_t^{\text{live}}(V,E^{\text{live}})$ be a induced graph with edges that are alive, i.e., $e \in E^{\text{live}}$ iff $X_{t,e}=1$ for $e \in E$. The triggering probability distribution $D_{\text{trig}}(S_t,X_t)$ degenerates to a deterministic triggered set, i.e., $\tau_t$ is deterministically decided given $S_t$ and $X_t$. The reward $R(S_t, X_t, \tau_t)$ equals to the number activated nodes at the end of $t$, i.e., the nodes that are reachable from $S_t$ in the live-edge graph $G_t^{\text{live}}$. The offline oracle is a $(1-1/e-\varepsilon, 1/|V|)$-approximation algorithm given by the greedy algorithm from~\cite{kempe2003maximizing}.

\subsubsection{Online Influence Maximization Bandit on Tri-partite Graph (TPG)}
We consider a OIM scenario where the underlying graph is a graph with three layers: the seed node layer, the intermediate layer and the target node layer.
Specifically, the underlying graph is denoted as $(L,M,R)$, where the first layer consists of $L$ candidates for the seed node selection, the second layer consists of $M$ intermediate nodes and the third layer are $R$ target nodes.
Such a setting is of significant interest since the edges connecting the target nodes can only be triggered and cannot be observed by seed-node selection, which requires the notion of triggering arms.
Another favorable feature is that the reward function of this setting can be explicitly expressed, whereas for general IM, even calculating the explicit form of the reward function is NP-hard~\cite{chen2009efficient}.

In this application, the base arms consists of two parts:

(1) $\{(i,j): i\in[L], j \in [M]\}$, each of which is associated with probability $x_{i,j}\in[0,1]$ indicating the probability whether edge $(i,j)$ is alive when $i$ is selected as seed nodes. Without loss of generality, we assume the seed nodes are $S=\{1,...,K\}$.

(2)$\{(j,k): j\in[M], k \in [R]\}$, each of which is associated with probability $y_{j,k} \in [0,1]$ indicating whether edge $(j,k)$ is live when $j$ is triggered for the first time. The triggering probability $p_{j}^{D, S}=1-\prod_{i \in [K]}{(1-x_{i,j})}$.

Let $\bmu=(\bx,\by)\in (0,1)^{(LM+MR)}, \bar{\bmu}=(\bar{\bx}, \bar{\by})\in  (0,1)^{(LM+MR)}$, where $\bar{\bx}=\boldzeta_{x}+\boldeta_{x}+\bx, \bar{\by}=\boldzeta_{y}+\boldeta_{y}+\by$, for $\boldzeta, \boldeta \in [0,1]^{(LM+MR)}$.
Fix any target node $k$, the reward function $r_k(S;\bx,\by)=1-\prod_{j \in [M]}\left(1-y_{j,k}(1-\prod_{i \in [K]}(1-x_{i,j}))\right)$.

\begin{lemma}\label{lem:app_tbg}
OIM-TPG problem satisfies {\TPVMm} bounded smoothness condition with coefficient $(B_v,B_1,\lambda)=(\sqrt{2}R,R,1)$, where $R$ is the total number of target nodes.
\end{lemma}

\begin{proof}

For notational brevity, we denote $p_j\triangleq p_j^{D,S}=1-\prod_{i \in [L]}{(1-x_{i,j})}$, $\bar{p}_j=1-\prod_{i \in [K]}(1-\bar{x}_{i,j}), g_j=1-\bar{y}_{j,k} p_j, \bar{g}_j=1-\bar{y}_{j,k}\bar{p}_j$.

The difference of $r_k(S; \bar{\bmu}),r_k(S; \bmu)$ can be written as,
\begin{align}
    \abs{r(S; \bar{\bmu})-r(S; \bmu)}&= \sum_{k \in [R]}r_k(S; \bar{\bmu})-r_k(S; \bmu)\label{apdx_eq:tpg_mono}\\
    &= \sum_{k \in [R]}\underbrace{\left(1-\prod_{j\in [M]}(1-\bar{y}_{j,k}\bar{p}_j) \right) - \left(1-\prod_{j\in [M]}(1-\bar{y}_{j,k} p_j) \right)}_{\text{term (a)}} \notag\\
    &+ \sum_{k \in [R]}\underbrace{\left(1-\prod_{j\in [M]}(1-\bar{y}_{j,k} p_j) \right) - \left(1-\prod_{j\in [M]}(1-y_{j,k} p_j) \right)}_{\text{term} (b)}\label{apdx_eq:tpg_add_minus},
\end{align}
where \cref{apdx_eq:tpg_add_minus} is by adding and subtracting the same $\sum_{k \in [R]}\left(1-\prod_{j\in [M]}(1-\bar{y}_{j,k} p_j) \right)$.

For term $(a)$, it holds that
\begin{align}
    &\text{term} (a) = \sum_{j \in [M]}(g_j-\bar{g}_j)(g_1...g_{j-1} \bar{g}_{j+1}...\bar{g}_M)\label{apdx_eq:tpg_a_tlscp}\\
    &= \sum_{j \in [M]}\bar{y}_{j,k}(\bar{p}_j-p_j)(g_1...g_{j-1} \bar{g}_{j+1}...\bar{g}_M)\notag\\
    &= \sum_{j \in [M]}\bar{y}_{j,k}(g_1...g_{j-1} \bar{g}_{j+1}...\bar{g}_M)\left(\left(1-\prod_{i \in [K]}(1-\bar{x}_{i,j})\right)- \left( 1-\prod_{i \in [K]}(1-x_{i,j}) \right)\right)\notag\\
    &= \sum_{j \in [M]}\bar{y}_{j,k}(g_1...g_{j-1} \bar{g}_{j+1}...\bar{g}_M)\sum_{i \in [K]}(\bar{x}_{i,j}-x_{i,j})\left((1-x_{1,j})...(1-x_{i-1,j}) (1-x_{i+1,j})...(1-x_{K,j})\right)\label{apdx_eq:tpg_a_tlscp_p}\\
    &=\sum_{i \in [K]}\sum_{j \in [M]}(\zeta_{x,i,j}+\eta_{x,i,j})\bar{y}_{j,k}(g_1...g_{j-1} \bar{g}_{j+1}...\bar{g}_M)\left((1-x_{1,j})...(1-x_{i-1,j}) (1-x_{i+1,j})...(1-x_{K,j})\right)\notag\\
    &\le \underbrace{\sum_{i \in [K]}\sum_{j \in [M]}\left(\frac{\zeta_{x,i,j}}{\sqrt{(1-x_{i,j})x_{i,j}}}\right)\sqrt{x_{i,j}(1-x_{1,j})...(1-x_{i-1,j})\bar{y}_{j,k} (g_1...g_{j-1})}}_{\text{term} (c)} \notag\\
    &+ \underbrace{\sum_{i \in [K]}\sum_{j \in [M]}\eta_{x,i,j}}_{\text{term} (d)}\label{apdx_eq:tpg_a_multi_div},
\end{align}
where \cref{apdx_eq:tpg_a_tlscp} is by telescoping the term (b), \cref{apdx_eq:tpg_a_tlscp_p} is by telescoping $\bar{p}_j-p_j$, \cref{apdx_eq:tpg_a_multi_div} is because $\bar{g}_j\in [0,1]$ for all $j$ and multiplying the first term by $\sqrt{x_{i,j}}$ but dividing it by $\sqrt{(1-x_{i,j})x_{i,j}}$.

For term $(b)$, it holds that 
\begin{align}
    \text{term} (b) &= \left(1-\prod_{j \in [M]}(1-\bar{y}_{j,k} p_j) \right) -\left(1-\prod_{j \in [M]}(1-y_{j,k} p_j) \right)\notag\\
    &=\sum_{j \in [M]}p_j (\bar{y}_{j,k} - y_{j,k}) (1-y_{1,k} p_1)...(1-y_{j,k} p_j) (1-\bar{y}_{j+1,k}p_{j+1})... (1-\bar{y}_{M,k} p_M)\label{apdx_eq:tpg_b_tlscp_p}\\
    &\le \underbrace{\sum_{j \in [M]}\left(\frac{\zeta_{y,j,k}\sqrt{p_j}}{\sqrt{(1-y_{j,k})y_{j,k}}}\right)\sqrt{y_{j,k}p_j(1-y_{1,k} p_1)...(1-y_{j-1,k}p_{j-1})}}_{\text{term} (e)} + \underbrace{\sum_{j\in [M]}p_{j}\eta_{y,j,k}}_{\text{term} (f)}\label{apdx_eq:tpg_b_multi_div},
\end{align}
where \cref{apdx_eq:tpg_b_tlscp_p} is by telescoping on the term (b), \cref{apdx_eq:tpg_b_multi_div} is by multiplying the first term by $\sqrt{y_{i,j}}$ but dividing it by $\sqrt{(1-y_{i,j})y_{i,j}}$.

Next, we apply Cauchy–Schwarz inequality to term (c) and term (e) simultaneously,
\begin{align}
    &\sum_{k \in [R]}(\text{term} (c) + \text{term} (e))\notag\\ 
    &\le \sqrt{\left(\sum_{i \in [K], j \in [M], k \in [R]}\frac{\zeta^2_{x,i,j}}{(1-x_{i,j})x_{i,j}}\right) + \left(\sum_{j \in [M]}\sum_{k \in [R]}\frac{\zeta^2_{y,j,k}p_j}{(1-y_{j,k})y_{j,k}}\right)}\notag\\
    &\cdot \Bigg[\left(\sum_{i \in [K]}\sum_{j\in [M]}\sum_{k \in [R]}x_{i,j}(1-x_{1,j})...(1-x_{i-1,j})\bar{y}_{j,k} (g_1...g_{j-1})\right)\notag\\
    &+\left(\sum_{j \in [M]}\sum_{k \in [R]}y_{j,k}p_j(1-y_{1,k} p_1)...(1-y_{j-1,k}p_{j-1})\right)\Bigg]^{1/2}\notag\\
    &\le \sqrt{\left(\sum_{i \in [K], j \in [M], k \in [R]}\frac{\zeta^2_{x,i,j}}{(1-x_{i,j})x_{i,j}}\right) + \left(\sum_{j \in [M]}\sum_{k \in [R]}\frac{\zeta^2_{y,j,k}p_j}{(1-y_{j,k})y_{j,k}}\right)}\notag\\
    &\cdot \sqrt{\left(\sum_{j\in [M]}\sum_{k \in [R]}(1-\prod_{i \in [K]}(1-x_{i,j}))\bar{y}_{j,k} (g_1...g_{j-1})\right)+\sum_{k \in [R]}\left(1-\prod_{j\in [M]}(1-y_{j,k}p_{j})\right)}\label{apdx_eq:tpg_ce_math_x}\\
    &= \sqrt{\left(\sum_{i \in [K], j \in [M], k \in [R]}\frac{\zeta^2_{x,i,j}}{(1-x_{i,j})x_{i,j}}\right) + \left(\sum_{j \in [M]}\sum_{k \in [R]}\frac{\zeta^2_{y,j,k}p_j}{(1-y_{j,k})y_{j,k}}\right)}\notag\\
    &\cdot \sqrt{\sum_{k \in [R]}\left(\sum_{j\in [M]}(1-g_j) (g_1...g_{j-1})\right)+\sum_{k \in [R]}\left(1-\prod_{j\in [M]}(1-y_{j,k}p_{j})\right)}\label{apdx_eq:tpg_ce_def_g}\\
        &= \sqrt{\left(\sum_{i \in [K], j \in [M], k \in [R]}\frac{\zeta^2_{x,i,j}}{(1-x_{i,j})x_{i,j}}\right) + \left(\sum_{j \in [M]}\sum_{k \in [R]}\frac{\zeta^2_{y,j,k}p_j}{(1-y_{j,k})y_{j,k}}\right)}\notag\\
        &\cdot \sqrt{\sum_{k \in [R]}\left(1-\prod_{j\in [M]}(1-\bar{y}_{j,k}p_{j})\right)+\sum_{k \in [R]}\left(1-\prod_{j\in [M]}(1-y_{j,k}p_{j})\right)}\label{apdx_eq:tpg_ce_math_g}\\
             &\le \sqrt{R}\sqrt{\left(\sum_{i \in [K]}\sum_{j \in [M]}\frac{\zeta^2_{x,i,j}}{(1-x_{i,j})x_{i,j}}\right) + \left(\sum_{j \in [M]}\sum_{k \in [R]}\frac{\zeta^2_{y,j,k}p_j}{(1-y_{j,k})y_{j,k}}\right)}\notag\\
             &\cdot \sqrt{\sum_{k \in [R]}\left(1-\prod_{j\in [M]}(1-\bar{y}_{j,k}p_{j})\right)+\sum_{k \in [R]}\left(1-\prod_{j\in [M]}(1-y_{j,k}p_{j})\right)}\label{apdx_eq:tpg_ce_i_in_K}\\
        &\le \sqrt{\left(\sum_{i \in [K], j \in [M], k \in [R]}\frac{\zeta^2_{x,i,j}}{(1-x_{i,j})x_{i,j}}\right) + \left(\sum_{j \in [M]}\sum_{k \in [R]}\frac{\zeta^2_{y,j,k}p_j}{(1-y_{j,k})y_{j,k}}\right)}\cdot\sqrt{2}R\label{apdx_eq:tpg_ce_bounded},
\end{align}
where \cref{apdx_eq:tpg_ce_math_x} is due to the math calculation on the summation over $i \in [K]$, \cref{apdx_eq:tpg_ce_def_g} is due to the definition of $g_j$, \cref{apdx_eq:tpg_ce_math_g} is due to the math calculation on the summation over $j \in [M]$, \cref{apdx_eq:tpg_ce_i_in_K} is because we take one summation over $k \in [R]$ out for the first square root term and \cref{apdx_eq:tpg_ce_bounded} is because $y_{j,k},\bar{y}_{j,k}, p_j$ are all bounded with support $[0,1]$.

Combined with $\sum_{k \in [R]}(\text{term}(d)+ \text{term}(f))$, we can derive that 
\begin{align}
    \abs{r(S; \bar{\bmu})-r(S; \bmu)} &= \sqrt{\left(\sum_{i \in [K]}\sum_{j \in [M]}\frac{\zeta^2_{x,i,j}}{(1-x_{i,j})x_{i,j}}\right) + \left(\sum_{j \in [M]}\sum_{k \in [R]}\frac{\zeta^2_{y,j,k}p_j}{(1-y_{j,k})y_{j,k}}\right)}\cdot\sqrt{2}R \notag\\
    &+ \sum_{i \in [K]}\sum_{j \in [M]}\sum_{k \in [R]}\eta_{x,i,j} + \sum_{j\in [M]}\sum_{k \in [R]}p_{j}\eta_{y,j,k}\notag\\
    &\le\sqrt{\left(\sum_{i \in [K]}\sum_{j \in [M]}\frac{\zeta^2_{x,i,j}}{(1-x_{i,j})x_{i,j}}\right) + \left(\sum_{j \in [M]}\sum_{k \in [R]}\frac{\zeta^2_{y,j,k}p_j}{(1-y_{j,k})y_{j,k}}\right)}\cdot\sqrt{2}R \notag\\
    &+ R\left( \sum_{i \in [K]}\sum_{j \in [M]}\eta_{x,i,j} + \sum_{j\in [M]}\sum_{k \in [R]}p_{j}\eta_{y,j,k}\right)\label{apdx_eq:tpg_all_i_in_K},
\end{align}
where \cref{apdx_eq:tpg_all_i_in_K} is because the summation over $k \in [R]$ in the second term.

Hence we have $(B_g, B_1, \lambda)=(\sqrt{2}R, R, 1)$.
\end{proof}





\subsubsection{Influence Maximization on Directed Acyclic Graphs (DAG)}
In this section, we introduce how to derive $(B_g, B_1, \lambda)$ for DAGs. Note that TBG (or even multi-layered bipartite graphs) are special cases of DAG. However, the main challenge is that for general DAGs, there are no closed-form solutions for the reward function since it is NP-Hard to compute their influence spread \cite{chen2009efficient}. To deal with this challenge, we will use the live-edge graph and edge coupling technique as \cite{li2020online} and prove the following lemma.

\begin{lemma}\label{lem:app_dag}
OIM-DAG problem satisfies {\TPVMm} bounded smoothness condition with coefficient $(B_v,B_1,\lambda)=(\sqrt{L}|V|,|V|,1)$, where $L$ is the length of the longest directed path.
\end{lemma}
\begin{proof}
Denote the DAG graph as $G(V,E)$ where $V$ are nodes and $E$ are the edges. With a little abuse of the notation, we will denote $u(e), v(e)$ as the starting node and ending node of the edge $e$.
Inspired by previous applications with closed-form solutions, we will partition edges $E$ (i.e., base arms) into $L$ groups, where $L$ is the length of the longest directed path of $G$. More specifically, we apply the topological sort on $G$ and label each node with $l(v)\in \{0,1,...,L\}$, where $l(v)$ is the length of the longest path that starts from node $v$. For simplicity, we say node $v$ in layer $l(v)$. This sorting and labelling procedure can be done in $O(V+E)$ time complexity. Given this labelling $l(v)$ for $v\in V$, we partition edges $E$ into $L$ disjoint edge sets so that $E = \bigcup_{s \in [L]}E_s$, where $E_s$ contains edges that point from node in layer $s$ to node in lower layers, i.e., $E_s=\{(u,v): l(u)=s, l(v)< s\}$. One critical property we will use later is that for any two edges $e, e' \in E_s$ in the same layer $s$, there are no directed path $p$ from any node $v \in V$ so that both $e, e'$ are in the path $p$, or otherwise $l(u(e))\neq l(u(e'))$ which contradicts the assumption that starting nodes $u(e)$ and $u(e')$ belong to the same layer $s$.

Let $\bmu= (\bmu_1, ..., \bmu_L)$ be the true mean vector for the partitions $(E_1, ..., E_L)$ mentioned above, where each partition contains $n_{s}=|E_{s}|$ base arms (with $\bmu_{s} \in (0,1)^{n_{s}}$), for $s \in [L]$. Similarly, we denote $\bar{\bmu}=(\bar{\bmu}_1, ..., \bar{\bmu}_L), \boldzeta= (\boldzeta_1, ..., \boldzeta_L), \boldeta= (\boldeta_1, ..., \boldeta_L)$ such that $\bar{\bmu}=\bmu + \boldzeta + \boldeta$ and $\bar{\bmu}_s \in (0,1)^{n_s}, \boldzeta, \boldeta \in [0,1]^{n_s}$.
Now for the reward, we focus on each node $t$ as target, so that $r(S;\bmu)\triangleq \sum_{t \in V} r_t(S;\bmu)$ and $r(S;\bar{\bmu})\triangleq \sum_{t \in V} r_t(S;\bar{\bmu})$. Now fix any target node $t \in V$, we can telescope the reward by gradually changing one partition from $\bar{\bmu}_s$ to $\bmu_s$ as follows:
\begin{align}
    &\abs{r_t(S;\bar{\bmu})-r_t(S;\bmu)}\notag\\
    &=r_t(S;\bar{\bmu}_1, ..., \bar{\bmu}_L)-r_t(S;\bmu_1, \bar{\bmu}_2, ..., \bar{\bmu}_L)\notag\\
        &=r_t(S;\bmu_1, \bar{\bmu}_2, ..., \bar{\bmu}_L)-r_t(S;\bmu_1, \bmu_2, \bar{\bmu}_3, ..., \bar{\bmu}_L)\notag\\
    &+ ... \notag\\
    &+\underbrace{r_t(S;\bmu_1,... \bmu_{s-1}, \bar{\bmu}_{s}, ..., \bar{\bmu}_{L})-r_t(S;\bmu_1,..., \bmu_{s}, \bar{\bmu}_{s+1}, ..., \bar{\bmu}_{L})}_{\text{s-th partition}}\notag\\
    &+ ... \notag\\
    &+r_t(S;\bmu_1, ..., \bmu_{L-1}, \bar{\bmu}_L)-r_t(S;\bmu_1, ..., \bmu_L)
\end{align}

For the s-th partition, since the first term and the second term differs only in $\bmu_s$ and $\bar{\bmu}_s$, we denote the first $s-1$ partitions as $\cup_{\ell=1}^{s-1}\bmu_{\ell}$ and the last $L-s$ partitions as $\cup_{\ell=s+1}^L \bar{\bmu}_{\ell}$. We can further telescope the reward difference by gradually changing one parameter from $\bar{\mu}_{s,i}$ to $\mu_{s,i}$ as follows:

\begin{align}
    &\text{s-th partition}\notag\\     
    &=r_t(S;\cup_{\ell=1}^{s-1}\bmu_{\ell}, \bar{\mu}_{s,1}, ..., \bar{\mu}_{s,n_s}, \cup_{\ell=s+1}^L\bar{\bmu}_{\ell})-r_t(S;\cup_{\ell=1}^{s-1}\bmu_{\ell}, \mu_{s,1}, \bar{\mu}_{s,2}, ..., \bar{\mu}_{s,n_s}, \cup_{\ell=s+1}^L\bar{\bmu}_{\ell})\notag\\
        &=r_t(S;\cup_{\ell=1}^{s-1}\bmu_{\ell}, \mu_{s,1}, \bar{\mu}_{s,2}, ..., \bar{\mu}_{s,n_s}, \cup_{\ell=s+1}^L\bar{\bmu}_{\ell})-r_t(S;\cup_{\ell=1}^{s-1}\bmu_{\ell}, \mu_{s,1}, \mu_{s,2}, \bar{\mu}_{s,3}, ..., \bar{\mu}_{s,n_s}, \cup_{\ell=s+1}^L\bar{\bmu}_{\ell})\notag\\
    &+ ... \notag\\
    &+\underbrace{r_t(S;\cup_{\ell=1}^{s-1}\bmu_{\ell}, \cup_{\ell=1}^{i-1}\mu_{s,\ell}, \cup_{\ell=i}^{n_s}\bar{\mu}_{s,\ell}, \cup_{\ell=s+1}^L\bar{\bmu}_{\ell})-r_t(S;\cup_{\ell=1}^{s-1}\bmu_{\ell}, \cup_{\ell=1}^{i}\mu_{s,\ell}, \cup_{\ell=i+1}^{n_s}\bar{\mu}_{s,\ell}, \cup_{\ell=s+1}^L\bar{\bmu}_{\ell}))}_{\text{i-th change in s-th partition}}\notag\\
    &+ ... \notag\\
    &+r_t(S;\cup_{\ell=1}^{s-1}\bmu_{\ell}, \mu_{s,1}, ..., \mu_{s,n_s-1}, \bar{\mu}_{s,n_s}, \cup_{\ell=s+1}^L\bar{\bmu}_{\ell})-r_t(S;\cup_{\ell=1}^{s-1}\bmu_{\ell}, \mu_{s,1}, ..., \mu_{s,n_s}, \cup_{\ell=s+1}^L\bar{\bmu}_{\ell})
\end{align}

For the i-th change in s-th partition, we use the live edge graph~\cite{chen2013information} technique (which links the probability of a node is activated to the probability this node is reachable in the random live-edge graph where each edge $i$ is independently selected as live-edge with probability $\mu_i$), in order to transform the reward difference into the probability of some events happens. Note that by using ``bypass of edge $e$", we mean the path that connects seed nodes $S$ and target node $t$ but does not contains edge $e$ in this path.
Let $e_{s,i}=(u(e_{s,i}), v(e_{s,i}))$ denote the $i$-th edge with staring node $u(e_{s,i})$ and ending node $v(e_{s,i})$ in partition $E_s$, for $i\in [n_s]$. We use $\Pr[A| \bmu]$ to denote the probability of all live-edge graphs under parameter $\bmu$ so that event $A$ happens. Consider the live-edge graphs without $e_{s,i}$,

\begin{align}
    &\text{i-th change in s-th partition}\notag\\
    &=(\bar{\mu}_{s,i}-\mu_{s,i}) \Pr[S\rightarrow u(e_{s,i}), v(e_{s,i})\rightarrow t, \text{ no bypass of } e_{s,i} \mid \cup_{\ell=1}^{s-1}\bmu_{\ell},\notag\\
    &\mu_{s,1},... \mu_{s,i-1}, \bar{\mu}_{s,i+1}, ..., \bar{\mu}_{s,n_s}, \cup_{\ell=s+1}^L\bar{\bmu}_{\ell}]\label{apdx_eq:oim_dag_live_edge}\\
    &\le (\bar{\mu}_{s,i}-\mu_{s,i}) \Pr[S\rightarrow u(e_{s,i}), v(e_{s,i})\rightarrow t, \text{ no bypass of } e_{s,i} \mid \cup_{\ell=1}^{s-1}\bmu_{\ell}, \bmu_{s} \backslash \mu_{s,i}, \cup_{\ell=s+1}^L\bar{\bmu}_{\ell}]\label{apdx_eq:oim_dag_live_edge_change}\\
    &\le \zeta_{s,i} \Pr[S\rightarrow u(e_{s,i}), v(e_{s,i})\rightarrow t, \text{ no bypass of } e_{s,i} \mid \cup_{\ell=1}^{s-1}\bmu_{\ell}, \bmu_{s} \backslash \mu_{s,i}, \cup_{\ell=s+1}^L\bar{\bmu}_{\ell}] \notag\\
    &+\eta_{s,i}\Pr[S\rightarrow u(e_{s,i}) \mid \cup_{\ell=1}^{s-1}\bmu_{\ell}]\label{apdx_eq:oim_dag_second_obs}\\
    &=\sqrt{\frac{\zeta^2_{s,i}}{\mu_{s,i}}} \sqrt{\mu_{s,i}(\Pr[S\rightarrow u(e_{s,i}), v(e_{s,i})\rightarrow t, \text{ no bypass of } e_{s,i} \mid \cup_{\ell=1}^{s-1}\bmu_{\ell}, \bmu_{s}\backslash \mu_{s,i}, \cup_{\ell=s+1}^L\bar{\bmu}_{\ell}])^2}\notag\\
    &+ \eta_{s,i}\Pr[S\rightarrow u(e_{s,i}) \mid \cup_{\ell=1}^{s-1}\bmu_{\ell}]\label{apdx_eq:oim_dag_multiply_divide}\\
    &\le\sqrt{ \frac{\zeta^2_{s,i}\Pr[S\rightarrow u(e_{s,i})\mid \cup_{\ell=1}^{s-1}\bmu_{\ell}]}{(1-\mu_{s,i})\mu_{s,i}}} \notag\\
    &\cdot\sqrt{\Pr[S\rightarrow u(e_{s,i}), e_{s,i} \text{ is live},  v(e_{s,i})\rightarrow t, \text{ no bypass of } e_{s,i} \mid \cup_{\ell=1}^{s}\bmu_{\ell}, \cup_{\ell=s+1}^L\bar{\bmu}_{\ell}]}\notag\\
    &+ \eta_{s,i}\Pr[S\rightarrow u(e_{s,i}) \mid \cup_{\ell=1}^{s-1}\bmu_{\ell}]\label{apdx_eq:oim_dag_divide},
\end{align}
where \cref{apdx_eq:oim_dag_live_edge} is due to the coupling technique~\cite{li2020online} and the contribution of live-edge graph $G'$ to the reward difference is $\bar{\mu}_{s,i} - \mu_{s,i}$ if $u(e_{s,i})$ is reachable from $S$, the target node $t$ is reachable from $v(e_{s,i})$ but $t$ is not reachable from any other paths in $G'$ that bypasses edge $e_{s,i}$ or otherwise the contribution of $G'$ is $0$; \cref{apdx_eq:oim_dag_live_edge_change} is due to $e_{s,i}$ and $e_{s,j}$ for $j \in [n_s]\backslash\{i\}$ are in the same s-th partition, so that the parameter change from $\bar{\mu}_{s,j}$ to $\mu_{s,j}$ for $j \in [n_s]\backslash\{i\}$ neither affects the probability of $u(e_{s,i})$ is reachable from $S$ nor affects the probability of the target node $t$ is reachable from $v(e_{s,i})$, but only increases the probability of there is no bypass of $e_{s,i}$ (since there is less possibility of a path connects $S$ and $t$ by reducing $\bar{\mu}_{s,j}$ to $\mu_{s,j}$); \cref{apdx_eq:oim_dag_second_obs} is because for the second term, we only require $u(e_{s,i})$ is reachable from $S$; \cref{apdx_eq:oim_dag_multiply_divide} is because we multiply and divide the first term by $\sqrt{\mu_{s,i}}$ at the same time; \cref{apdx_eq:oim_dag_divide} is because we divide the first term by $\sqrt{1-\mu_{s,i}}$ which is within $(0,1)$.

Now we can summation over all $(s,i)_{s \in [L], i \in [n_s]}$, we have
\begin{align}
    &\abs{r_t(S;\bar{\bmu})-r_t(S;\bmu)}\notag\\
    &\le\sum_{s,i}\sqrt{ \frac{\zeta^2_{s,i}\Pr[S\rightarrow u(e_{s,i})\mid \cup_{\ell=1}^{s-1}\bmu_{\ell}]}{(1-\mu_{s,i})\mu_{s,i}}} \notag\\
    &\cdot\sqrt{\Pr[S\rightarrow u(e_{s,i}),e_{s,i} \text{ is live},  v(e_{s,i})\rightarrow t, \text{ no bypass of } e_{s,i} \mid \cup_{\ell=1}^{s}\bmu_{\ell}, \cup_{\ell=s+1}^L\bar{\bmu}_{\ell}]}\notag\\
    &+ \sum_{s \in [L]}\sum_{i \in [n_s]}\eta_{s,i}\Pr[S\rightarrow u(e_{s,i}) \mid \cup_{\ell=1}^{s-1}\bmu_{\ell}]\notag\\
    &\le \sqrt{\sum_{i \in [m]}\frac{\zeta_{i}^2 p_i^{D,S}}{\mu_i(1-\mu_i)}\left(\sum_{s \in [L]}\sum_{i \in [n_s]}\Pr[S\rightarrow u(e_{s,i}),e_{s,i} \text{ is live},  v(e_{s,i})\rightarrow t, \text{ no bypass of } e_{s,i} \mid \cup_{\ell=1}^{s}\bmu_{\ell}, \cup_{\ell=s+1}^L\bar{\bmu}_{\ell}]\right)}\notag\\
    &+ \sum_{i \in [m]}\eta_{i}p_{i}^{D,S}\label{apdx_eq:oim_dag_cauchy}\\
    &\le \sqrt{\sum_{i \in [m]}\frac{\zeta_{i}^2 p_i^{D,S}}{\mu_i(1-\mu_i)}}\sqrt{L}+ \sum_{i \in [m]}\eta_{i}p_{i}^{D,S},\label{apdx_eq:oim_dag_no_overlap}
\end{align}
where \cref{apdx_eq:oim_dag_cauchy} uses the Cauchy-Schwatz Inequality; \cref{apdx_eq:oim_dag_no_overlap} uses the critical property that for $e_{s,i}$ and $e_{s,j}$ with $i\neq j$, there does not exist any path that can contain both $e_{s,i}$ and $e_{s,j}$ as mentioned earlier, which means under the same parameter $(\cup_{\ell=1}^{s}\bmu_{\ell}, \cup_{\ell=s+1}^L\bar{\bmu}_{\ell})$, the live edge graphs that satisfies $\{S\rightarrow u(e_{s,i}),e_{s,i} \text{ is live},  v(e_{s,i})\rightarrow t, \text{ no bypass of } e_{s,i}\}$ and the live edge graphs that satisfies $\{S\rightarrow u(e_{s,j}),e_{s,j} \text{ is live},  v(e_{s,j})\rightarrow t, \text{ no bypass of } e_{s,j}\}$ are disjoint (or otherwise contradicts the fact that there is no bypass of $e_{s,i}$ or $e_{s,j}$), so it holds that $\sum_{i \in [n_s]}\Pr[S\rightarrow u(e_{s,i}),e_{s,i} \text{ is live},  v(e_{s,i})\rightarrow t, \text{ no bypass of } e_{s,i} \mid \cup_{\ell=1}^{s}\bmu_{\ell}, \cup_{\ell=s+1}^L\bar{\bmu}_{\ell}]\le 1$.

Considering all the target nodes $t \in [V]$, we have $(B_v, B_1, \lambda)=(|V|\sqrt{L}, |V|, 1)$.
\end{proof}

\subsection{Probabilistic Maximum Coverage Bandit~\cite{chen2016combinatorial,merlis2019batch}}
In this section, we consider the probabilistic maximum coverage (PMC) problem. PMC is modeled by a weighted bipartite graph $G=(L,V,E)$, where $L$ are the source nodes, $V$ are the target nodes and each edge $(u,v) \in E$ is associated with a probability $p(u,v)$. The task of PMC is to select a set $S \subseteq L$ of size $k$ so as to maximize the expected number of nodes activated in $V$, where a node $v \in V$ can be activated by a node $u \in S$ with an independent probability $p(u,v)$. PMC can naturally model the advertisement placement application, where $L$ are candidate web-pages, $V$ are the set of users, and $p(u,v)$ is the probability that a user $v$ will click on web-page $u$. 

PMC fits into the non-triggering CMAB framework: each edge $(u,v) \in E$ corresponds to a base arm, the action is the set of edges that are incident to the set $S \subseteq L$, the unknown mean vectors $\bmu \in (0,1)^E$ with $\mu_{u,v}=p(u,v)$ and we assume they are independent across all base arms. In this context, the reward function $r(S;\bmu)=\sum_{v\in V}(1-\prod_{u \in S}(1-\mu_{u,v}))$.

\begin{lemma}\label{lem:app_pmc}
PMC problem satisfies VM bounded smoothness condition (Condition~\ref{cond:VM}) with coefficient $(B_v,B_1)=(3\sqrt{2|V|},1)$.
\end{lemma}

\begin{proof}
We prove PMC satisfies VM condition by the definition of Gini-smoothness condition (Condition~\ref{cond:gini-smooth}) and \cref{lem:lemma6_no_trigger}.
First, we know $\frac{\partial r(S;\bmu)}{\partial \mu_{u,v}}=\prod_{i \in S, i \neq u} (1-\mu_{i,j}) \le 1$, thus $\gamma_{\infty}=1$. Also $\sqrt{\sum_{u \in S, v \in V}\mu_{u,v}(1-\mu_{u,v} \frac{\partial (r(S;\bmu)}{\partial \mu_{u,v}})^2}=\sqrt{\sum_{u \in S, v \in V} \mu_{u,v}(1-\mu_{u,v}) \prod_{i \in S, i \neq u} (1-\mu_{i,v})^2}= \sqrt{\sum_{v \in V} \sum_{u \in S}\left( \mu_{u,v} \prod_{i \in S,  i \neq u}(1-\mu_{i,v})\right)\left(\prod_{i \in S}(1-\mu_{i,v})\right)}\le \sqrt{\sum_{v\in V} 1/4}=\sqrt{|V|/4}$, where the second last inequality uses the fact consider $S$ coins and the $i$-th coin is up with prob. $\mu_{i,v}$, the first term $\left( \mu_{u,v} \prod_{i \in S,  i \neq u}(1-\mu_{i,v})\right)$ corresponds to probability $P_1$ that only one coin is up and the second term $\left(\prod_{i \in S}(1-\mu_{i,v})\right)$ corresponds to the probability $P_2$ that all coins are down, thus $P_1*P_2 \le P_1 (1-P_1)\le 1/4$. Hence $\gamma_g=\sqrt{V/4}$. By \cref{lem:lemma6_no_trigger}, we have $(B_v,B_1)=(3\sqrt{2}\gamma_g, \gamma_{\infty})=(3\sqrt{|V|/2},1)$. 
\end{proof}

\section{Experiments}\label{apdx_sec:exp}
In this section, we conduct experiments to validate our proposed algorithms. All experiments were performed on a desktop with i7-9700K CPU and 32 GB RAM.
\begin{figure}[ht]
    \centering
    \includegraphics[width=0.45\textwidth]{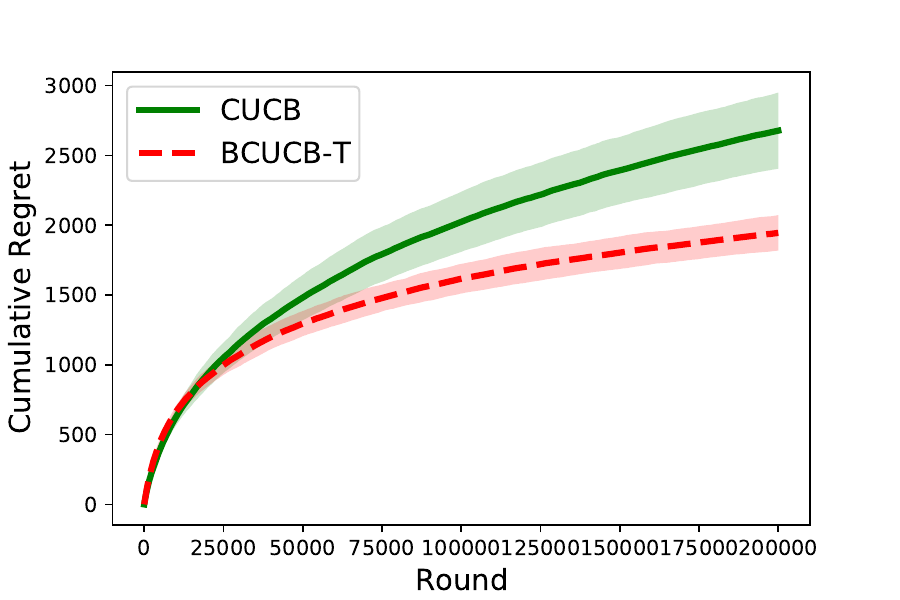}
    \caption{Cumulative regrets of CUCB and BCUCB-T for disjunctive cascading bandit.}
    \label{fig:cascading}
\end{figure}
\subsection{Cascading Bandits}
We consider cascading bandits as the application for CMAB with probabilistically triggered arms. More specifically, we choose the disjunctive cascading bandit problem to compare the performance of CUCB and BCUCB-T, where the reward function is $r(S;\bm{\mu}) = 1 - \prod_{i\in S} (1-\mu_i)$. Similar to the experimental setup in~\cite{kveton2015cascading}, we set batch-size $K=10$ and generate 30 base arms with means randomly sampled from the uniform distribution $U(0,0.1)$. Figure~\ref{fig:cascading} shows the cumulative regrets of CUCB and BCUCB-T for $200,000$ rounds. We repeat each experiment 20 times and show the average regret with shaded standard deviation. The average running time of CUCB and BCUCB-T for $200,000$ rounds are $13$s and $19$s, respectively.
As shown in the figure, BCUCB-T achieves around $20\%$ less regret than CUCB, owing to its variance-aware confidence radius control.

\begin{figure}[ht]
    \centering
    \includegraphics[width=0.45\textwidth]{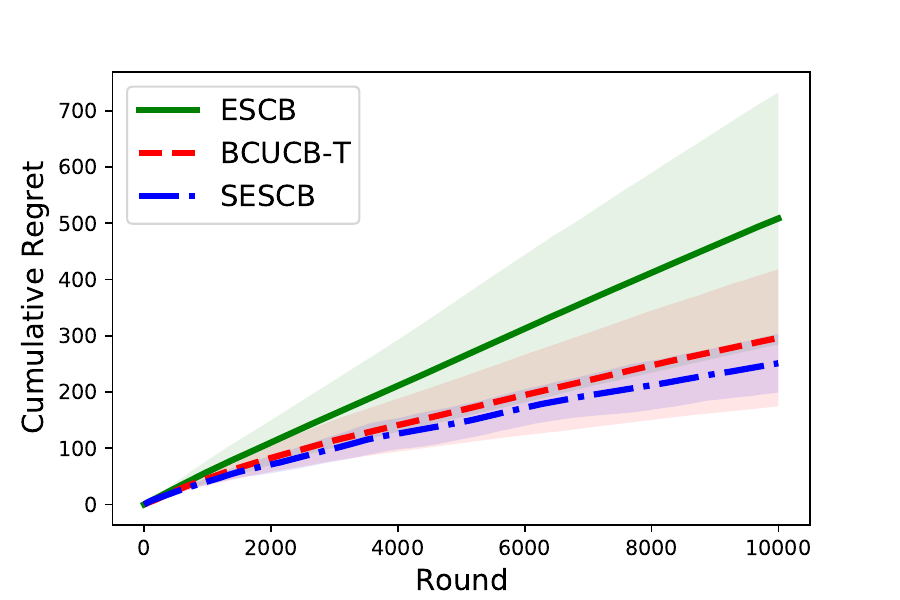}
    \caption{Cumulative regrets of ESCB, BCUCB-T and SESCB for the PMC problem.}
    \label{fig:pmc}
\end{figure}

\subsection{Probabilistic Maximum Coverage}
We consider Probabilistic Maximum Coverage (PMC) as the application for non-triggering CMAB. We generate a complete bipartite graph with 10 source nodes on the left and 20 target nodes on the right. 
The goal is to select 5 seed nodes from source nodes to influence as many as target nodes. The edge probabilities are randomly sampled from the uniform distribution $U(0.05,0.06)$. We run ESCB~\cite{combes2015combinatorial}, BCUCB-T and SESCB on this graph. 
For SESCB, we set sub-Gaussian parameter $C_1 = 3$ (according to Remark 7) and VM smoothness coefficient $B_v = 3\sqrt{2\cdot20}/2$. Since the number of base arms to be learned in PMC is large, we shrink the confidence intervals of all algorithms by $\alpha_{\rho}=0.01$ to speed up the learning, e.g., for SESCB, $\bar{r}_t(S) = r(S;\hat{\bm{\mu}}_{t-1}) + \alpha_{\rho}\cdot\rho_t(S)$. We repeat each experiment 10 times and show the average regret. The average running time of ESCB, BCUCB-T and SESCB for $10,000$ rounds are $80$s, $7$s and $115$s, respectively.
Figure~\ref{fig:pmc} shows the cumulative regrets with shaded standard deviations. 
SESCB achieves around $15\%$ less regret than BCUCB-T, since it utilizes the independence of base arms while BCUCB-T; it also outperforms ESCB as ESCB is mainly designed for the linear reward case.


\end{document}